\documentclass[twoside]{article}
\usepackage[accepted]{aistats2018}

\usepackage{hyperref}
\usepackage{natbib}
\usepackage{amsmath, amssymb, amsthm, booktabs, caption}
\numberwithin{equation}{section}
\usepackage{color, graphicx, epstopdf, makeidx, mathrsfs, subcaption}
\usepackage{algorithm, algorithmic, cleveref}
\crefname{equation}{}{}

\graphicspath{{./exp/}{./}}

\newtheorem{theorem}{Theorem}
\newtheorem{lemma}{Lemma}

\newtheorem{proposition}{Proposition}
\theoremstyle{definition}
\theoremstyle{definition}\newtheorem{example}{Example}
\theoremstyle{definition}
\theoremstyle{definition}\newtheorem{assumption}{Assumption}
\theoremstyle{remark}\newtheorem{remark}{Remark}
\theoremstyle{remark}

\newcommand{\z}{\mathbf{z}}

\newcommand{\supp}{\mathrm{supp}}
\newcommand{\prox}{\mathrm{prox}}

\allowdisplaybreaks[3]

\begin{document}


\twocolumn[
\aistatstitle{A Unified Dynamic Approach to Sparse Model Selection}
\aistatsauthor{Chendi Huang \And Yuan Yao}
\aistatsaddress{
    School of Mathematical Sciences,\\
    Peking University,\\
    Beijing, China\\
    \texttt{cdhuang@pku.edu.cn}
    \And
    Department of Mathematics,\\
    Hong Kong University of Science and Technology,\\
    HKSAR, China\\
    \texttt{yuany@ust.hk}
}]

\begin{abstract}
    Sparse model selection is ubiquitous from linear regression to graphical models where regularization paths, as a family of estimators upon the regularization parameter varying, are computed when the regularization parameter is unknown or decided data-adaptively. Traditional computational methods rely on solving a set of optimization problems where the regularization parameters are fixed on a grid that might be inefficient. In this paper, we introduce a simple iterative regularization path, which follows the dynamics of a sparse Mirror Descent algorithm or a generalization of Linearized Bregman Iterations with nonlinear loss. Its performance is competitive to \texttt{glmnet} with a further bias reduction. A path consistency theory is presented that under the Restricted Strong Convexity (RSC) and the Irrepresentable Condition (IRR), the path will first evolve in a subspace with no false positives and reach an estimator that is sign-consistent or of minimax optimal $\ell_2$ error rate. Early stopping regularization is required to prevent overfitting. Application examples are given in sparse logistic regression and Ising models for NIPS coauthorship.
\end{abstract}

\section{Introduction}
\label{sec:intro}


In high dimensional statistics and machine learning, the data $\z$ is often assumed to be generated from a statistical model $\mathcal{P}(\alpha^\star, \beta^\star)$ with a sparse parameter $\beta^\star$, and the purpose is to estimate $\beta^\star$ typically via the following optimization approach,
\begin{equation}
    \label{eq:pen-opt}
    \min_{\alpha, \beta} \left( \ell\left( \alpha, \beta; \z \right) + \lambda P\left( \beta \right) \right),
\end{equation}
where $\ell(\alpha, \beta; \z)$ is a loss function depending on data $\z$ and parameter $(\alpha, \beta)$, usually based on likelihood, and $P(\beta)$ is a penalty function. For simplicity, we shall omit the dependence on $\z$ for the loss when it is clear from the context.

\begin{example}[Sparse linear regression model]
    \label{thm:model-linear}
    Let $X = (x^{(1)}, \ldots, x^{(n)})^T \in \mathbb{R}^{n\times p}$ be a fixed design matrix, and $y = (y^{(1)}, \ldots, y^{(n)})^T \in \mathbb{R}^n$,
    \begin{equation*}
        y^{(i)} = \alpha^{\star} + \beta^{\star T} x^{(i)} + \epsilon^{(i)}\ (1\le i \le n),
    \end{equation*}
    with $\epsilon^{(i)}$'s i.i.d. drawn from $N(0,\sigma^2)$, and $\beta^{\star}$ sparse. Let
    \begin{equation*}
        \ell(\alpha, \beta; \z) = \| y - \alpha - X \beta \|_2^2 / (2n)
    \end{equation*}
    be the loss function for data $\z = (x^{(i)}, y^{(i)})_{i=1}^n$ and parameter $(\alpha, \beta)$ (intercept $\alpha$ and linear parameter $\beta$), as well as the Lasso penalty $P(\beta) = \|\beta\|_1$. For model selection consistency, \citet{Zhao_model_2006,wainwright_sharp_2009} showed it under Restricted Strong Convexity (RSC) and Irrepresentable Condition (IRR); under a weaker restricted eigenvalue condition, \citet{BicRitTsy09} established the $\ell_2$-error at minimax optimal rates. 
\end{example}

\begin{example}[Sparse logistic regression model]
    \label{thm:model-logistic}
    Let $x^{(i)}\in \mathbb{R}^p\ (1\le i\le n)$, and $y^{(i)}\in\left\{ 1,-1 \right\}$,
    \begin{equation*}
        \mathbb{P}\left( y^{(i)} = 1 | x^{(i)} \right) = 1 / \left( 1 + \exp\left( - \left( \alpha^{\star} + \beta^{\star T} x^{(i)} \right) \right) \right),
    \end{equation*}
    with $\beta^{\star}$ sparse. \citet{ravikumar_high-dimensional_2010} considered \cref{eq:pen-opt} with the loss function for data $\z = (y^{(i)}, x^{(i)})_{i=1}^n$ and parameter $(\alpha, \beta)$
    \begin{equation}
        \label{eq:loss-logistic}
        \begin{split}
            \ell\left( \alpha, \beta; \z \right) &= - \frac{1}{n} \sum_{i=1}^n \log \mathbb{P}_{\alpha,\beta} \left( y = y^{(i)} | x^{(i)} \right)\\
            &= \frac{1}{n} \sum_{i=1}^n \log\left( 1 + \exp\left( - \left( \alpha + \beta^T x^{(i)} \right) y^{(i)} \right) \right),
        \end{split}
    \end{equation}
    as well as $P(\beta) = \|\beta\|_1$. They also showed its selection/estimation consistency.
\end{example}

\begin{example}[Sparse Ising model]
    \label{thm:model-ising}
    $x^{(i)}\ (1\le i\le n)$ are drawn from $x\in \{1, -1\}^p$ whose population satisfies
    \begin{multline}
        \label{eq:model-ising}
        \mathbb{P}\left( x = \left( x_1, \ldots, x_p \right)^T \right) \\
        \propto \exp\left( \frac{1}{2} \sum_{j=1}^p \alpha_j^{\star} x_j + \frac{1}{2} \sum_{j<j'} \beta_{j,j'}^{\star} x_j x_{j'} \right),
    \end{multline}
    where $\alpha^{\star}\in \mathbb{R}^p,\ \beta^{\star}\in \mathbb{R}^{p\times p}$,\protect\footnote{We assume $\mathrm{diag}(\beta^{\star}) = 0$ and $\beta^{\star}$ is symmetric.} and $\beta^{\star}$ is sparse. \citet{ravikumar_high-dimensional_2010} studied sparse Ising model \cref{eq:model-ising} by the so-called \textit{neighborhood-based logistic regression}, based on the discussion on sparse logistic models in their paper. Specifically, despite the difficulty to deal with the whole $(\alpha^{\star}, \beta^{\star})$ by using likelihood-based loss functions of Ising model, they noticed that
    \begin{equation*}
        \mathbb{P} \left( x_j | x_{-j} \right) = 1 / \left( 1 + \exp\left( - \left( \alpha_j^{\star} + \beta_{-j, j}^{\star T} x_{-j} \right) x_j \right) \right).
    \end{equation*}
    Thus each $j$ corresponds to a sparse logistic regression problem, i.e. \Cref{thm:model-logistic}, with $y, x, \alpha, \beta$ replaced by $x_j, x_{-j}, \alpha_j, \beta_{-j,j}$. Thus they learned $(\alpha_j^{\star}, \beta_{-j,j}^{\star})$ (by $\ell_1$ regularized logistic regression) for each $j$, instead of dealing with $(\alpha^{\star}, \beta^{\star})$ directly. \citet{xue_nonconcave_2012} considered \cref{eq:pen-opt} with the loss $\ell(\alpha,\beta)$ being the \textit{negative composite conditional log-likelihood}
    \begin{equation}
        \label{eq:loss-ising-composite}
        \frac{1}{n} \sum_{i=1}^n \sum_{j=1}^p \log\left( 1 + \exp\left( - \left( \alpha_j + \beta_{-j,j}^T x_{-j}^{(i)} \right) x_j^{(i)} \right) \right).
    \end{equation}
    $P(\cdot)$ can be $\ell_1$ penalty, SCAD penalty or other positive penalty function defined on $[0, +\infty)$. Alternatively \citet{sohl-dickstein_minimum_2011} proposed an approach of Minimum Probability Flow (MPF) which in the case of Ising model uses the following loss
    \begin{equation}
        \label{eq:loss-ising-mpf}
        \frac{1}{n} \sum_{i=1}^n \sum_{j=1}^p \exp\left( - \frac{1}{2} \left( \alpha_j + \beta_{-j,j}^T x_{-j}^{(i)} \right) x_j^{(i)} \right).
    \end{equation}
    The minimizer of this function is a reasonable estimator of $(\alpha^{\star}, \beta^{\star})$. However their work did not treat sparse models in high-dimensional setting. When facing sparse Ising model, one may consider \cref{eq:pen-opt} with the loss $\ell(\alpha,\beta)$ being the expression in \cref{eq:loss-ising-mpf} and $P(\beta) = \|\beta\|_1$, which is not seen in literature to the best of our knowledge.
\end{example}

\begin{example}[Sparse Gaussian graphical model]
    \label{thm:model-gau}
    $x^{(i)}\ (1\le i\le n)$ are drawn from a multivariate Gaussian distribution with covariance $\Sigma^{\star}\in \mathbb{R}^{p\times p}$, and the \textit{precision matrix} $\Omega^{\star} = \Sigma^{\star - 1}$ is assumed to be sparse. \citet{yuan_model_2007,ravikumar_model_2008} studied \cref{eq:pen-opt}, with the loss function being the \emph{negative scaled log-likelihood}, and the penalty being the sum of the absolute values of the off-diagonal entries of the precision matrix.
\end{example}

In general, \citet{negahban_unified_2009} provided a unified framework for analyzing the statistical consistency of the estimators derived by solving \cref{eq:pen-opt} with a proper choice of $\lambda$. However in practice, since $\lambda$ is unknown, one typically needs to compute the regularization path $\beta_{\lambda}$ as regularization parameter $\lambda$ varies on a grid, e.g. the \texttt{lars} \citep{lars} or the coordinate descent in \texttt{glmnet}. Such regularization path algorithms can be inefficient in solving many optimization problems. 

In this paper, we look at the following three-line iterative algorithm which, despite its simplicity, leads to a novel unified scheme of regularization paths for all cases above,
\begin{subequations}
    \label{eq:glbi-show}
    \begin{align}
        \label{eq:glbi-show-a}
        \alpha_{k+1} &= \alpha_k - \kappa \delta_k \nabla_{\alpha} \ell\left( \alpha_k, \beta_k \right),\\
        \label{eq:glbi-show-b}
        z_{k+1} &= z_k - \delta_k \nabla_{\beta} \ell \left( \alpha_k, \beta_k \right),\\
        \label{eq:glbi-show-c}
        \beta_{k+1} &= \kappa \mathcal{S}\left( z_{k+1}, 1 \right),
    \end{align}
\end{subequations}
where $z_0 = \beta_0 = 0$, $\alpha_0$ can be arbitrary and is naturally set $\arg\min_{\alpha} \ell\left( \alpha, \beta_0 \right)$, step size $\delta_k=\delta$ and $\kappa$ are parameters whose selection to be discussed later, and the shrinkage operator $\mathcal{S}(\cdot, 1)$ is defined element-wise as
$\mathcal{S}\left( z, 1 \right) = \mathrm{sign}(z)\cdot \max\left( |z| - 1, 0 \right)$. Such an algorithm is easy for parallel implementation, with linear speed-ups demonstrated in experiment \Cref{sec:exp} below. 

To see the regularization paths returned by the iteration, \Cref{fig:show-path} compared it against the \texttt{glmnet}. Such simple iterative regularization paths exhibit competitive or even better performance than the Lasso regularization paths by \texttt{glmnet} in reducing the bias and improving the accuracy (\Cref{sec:simu-logistic} for more details).

\textbf{\emph{How does this simple iteration algorithm work?}}

There are two equivalent views on algorithm \Cref{eq:glbi-show}. First of all, it can be regarded as a mirror descent algorithm (MDA) \citep{NemYud83,BecTeb03,Nemirovski12}
\begin{align*}
    & (\alpha_{k+1},\beta_{k+1}) \\
    = {} & \arg \min_z  \left\{ \left<z, \delta \nabla_{\alpha,\beta} \ell(\alpha_k,\beta_k)\right> + B_\Phi(z,(\alpha_k,\beta_k)) \right\} \\
    := {} & \prox_{\Phi}(\delta \nabla_{\alpha,\beta} \ell(\alpha_k,\beta_k))
\end{align*}
where $B_\Phi$ is the bregman divergence associated with $\Phi$, i.e.  defined by 
\begin{equation}
B_\Phi(u,v) = \Phi(u) - \Phi(v) - \left \langle \partial \Phi(v), u - v \right \rangle.
\end{equation}
Now set $\Phi(\alpha,\beta) = \|\alpha\|_2^2 / (2\kappa) + \|\beta\|_1 + \|\beta\|_2^2 / (2\kappa)$ involving a Ridge ($\ell_2$) penalty on $\alpha$ and an elastic net type ($\ell_1$ and $\ell_2$) penalty on $\beta$. Hence $\partial_\alpha \Phi(\alpha,\beta) = \alpha/\kappa$ and $\partial_\beta \Phi(\alpha,\beta)= \rho + \beta/\kappa$ where $\rho \in \partial \|\beta\|_1$. With this, the optimization in MDA leads to \Cref{eq:glbi-show-a} and
\begin{equation}
    \label{eq:md-b}
    \rho_{k+1} + \frac{1}{\kappa} \beta_{k+1} = \rho_k + \frac{1}{\kappa} \beta_k - \delta \nabla_{\beta} \ell(\alpha_k,\beta_k),\ \rho_k \in \partial \|\beta_k\|_1, 
\end{equation}
which is equivalent to \Cref{eq:glbi-show-b}. There has been extensive studies on the convergence $\ell(\alpha_k,\beta_k)- \min_{\alpha,\beta}\ell(\alpha,\beta)\leq O(k^{-r})$ ($r>0$), which are however not suitable for statistical estimate above as such convergent solutions lead to overfitting estimators. 

\begin{figure}[!htpb]
    \centering
    \includegraphics[width = 0.7\linewidth]{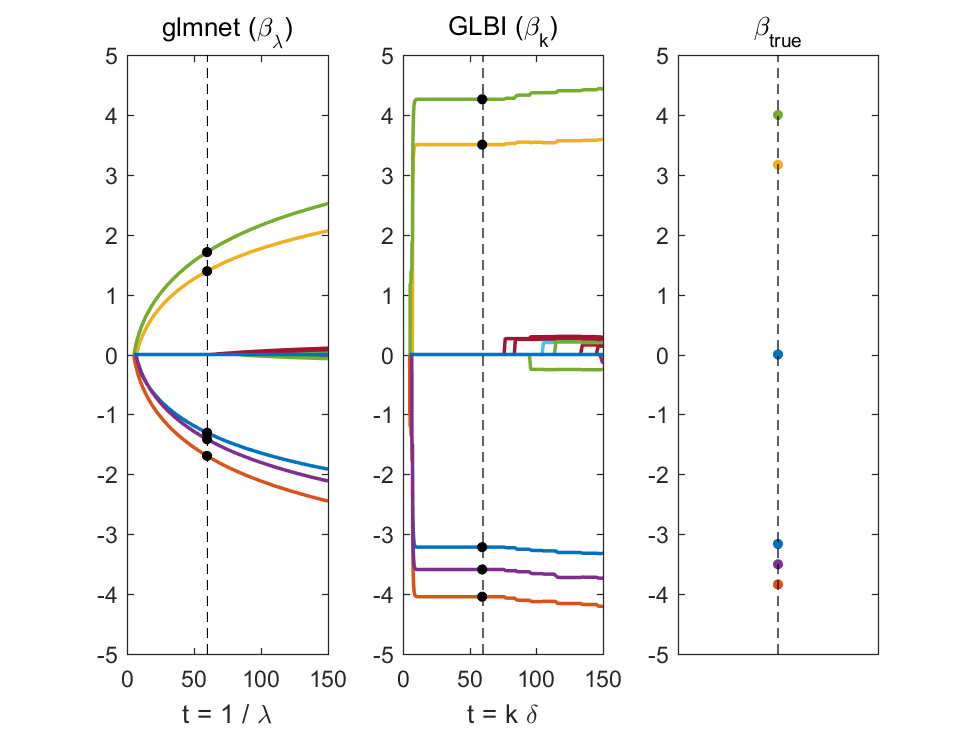} 
    \includegraphics[width = 0.9\linewidth]{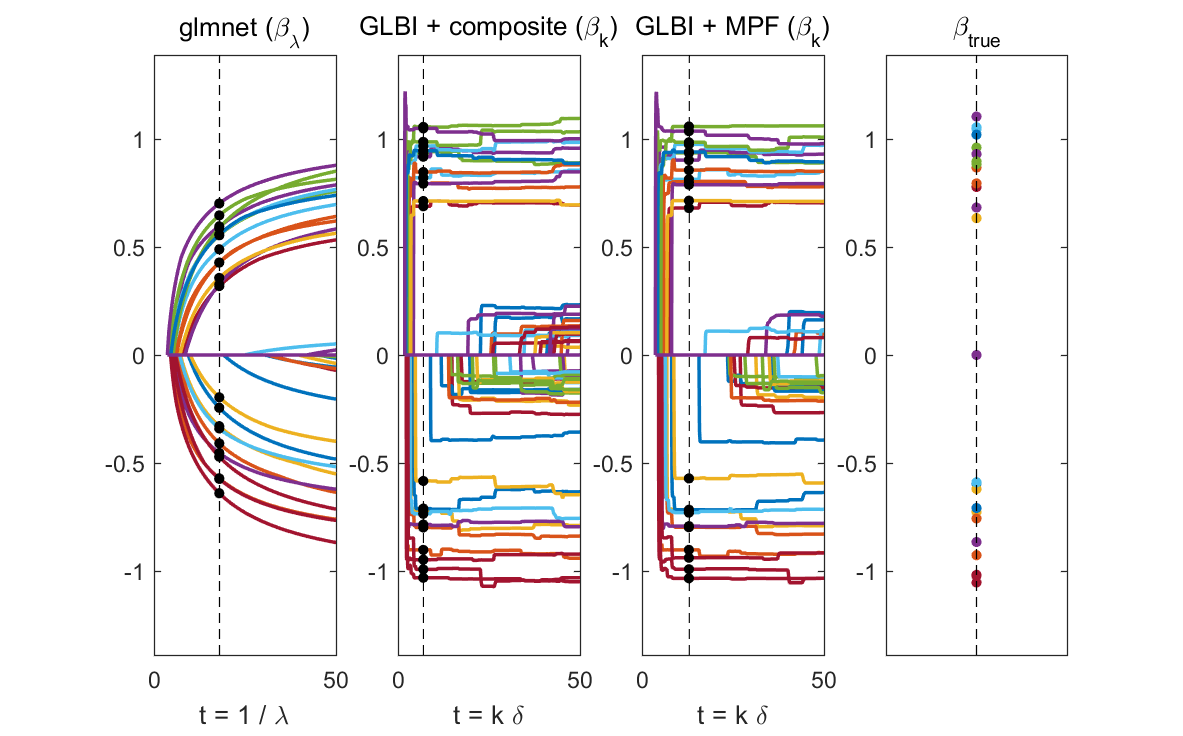}
    \caption{Top: path comparison between $\{\beta_\lambda\}\ (t=1/\lambda)$ by \texttt{glmnet} (left) and $\{\beta_k\}\ (t = k \delta)$ by GLBI (middle), for logistic models with true parameters (right). Bottom: path comparison of $\{\beta_\lambda\}\ (t=1/\lambda)$ by \texttt{glmnet} (left), $\{\beta_k\}\ (t = k \delta)$ by GLBI + composite loss (middle left), and $\{\beta_k\}\ (t = k \delta)$ by GLBI + MPF loss (middle right), for Ising models with true parameters (right). In both cases, \texttt{glmnet} selects biased estimates while GLBI finds more accurate ones.}
    \label{fig:show-path}
\end{figure}

An alternative dynamic view may lead to a deeper understanding of the regularization path. In fact for \Cref{thm:model-linear}, \cref{eq:glbi-show} reduces to the \textit{Linearized Bregman Iteration (LBI)} proposed by \citet{lbi} and analyzed by \citet{osher_sparse_2016} via its limit differential inclusions. It shows that equipped with the standard conditions as Lasso, an early stopping rule can find a point on the regularization path of \cref{eq:glbi-show} with the same sign pattern as true parameter (sign-consistency) and gives the unbiased oracle estimate, hence better than Lasso or any convex regularized estimates which are always biased. This can be generalized to our setting where \Cref{eq:md-b} is a discretization of the following dynamics
\begin{subequations}
    \label{eq:giss-show}
    \begin{align}
        \label{eq:giss-show-a}
        \dot{\alpha}(t) / \kappa &= - \nabla_{\alpha} \ell\left( \alpha(t), \beta(t) \right),\\
        \label{eq:giss-show-b}
        \dot{\rho}(t) + \dot{\beta}(t) / \kappa &= - \nabla_{\beta} \ell \left( \alpha(t), \beta(t) \right), \\
        \rho(t) & \in \partial \|\beta(t)\|_1.
    \end{align}
\end{subequations}
It is a restricted gradient flow (differential inclusion) where $\beta(t)$ has its sparse support controlled by $\rho(t)$. As $\kappa\to \infty$, it gives a sequence of estimates by minimizing $\ell$ with the sign pattern of $\beta(t)$ restricted on $\rho(t)$. Thus if an estimator $\beta(t)$ has the same sign pattern as $\beta^\star$, it must returns the \emph{unbiased oracle estimator} which is optimal. So it is natural to ask if there is a point on the path $\beta(t)$ (or $\beta_k$) which meets the sparsity pattern of true parameter $\beta^\star$. This is the \textbf{path consistency} problem to be addressed in this paper. In \Cref{sec:theory}, we shall present a theoretical framework as an answer, and \Cref{sec:exp} gives more applications, including Ising model learning for NIPS coauthorship.

Note that for \Cref{thm:model-logistic}, \cref{eq:glbi-show} reduces to the linearized Bregman iterations for logistic regression proposed by \citet{shi_new_2013} without a study of statistical consistency. A variable splitting scheme in comparison to generalized Lasso is studied in \citet{SplitLBI} which shows improved model selection consistency in some scenarios. Hence in this paper, we shall call the general form \cref{eq:glbi-show} as \emph{Generalized Linear Bregman Iterations (GLBI)}, in addition to (sparse) Mirror Descent flows.

\section{Path Consistency of GLBI}
\label{sec:theory}

Let $\theta^\star = (\alpha^\star, \beta^{\star T})^T$ denotes the true parameter, with sparse $\beta^\star$. Define $S := \supp (\beta^{\star})$ ($s := |S| \ll p$) as the index set corresponding to nonzero entries of $\beta$, and $S^c$ be its complement. Let $S_\alpha = (\alpha, S)$, and $S_\alpha = S$ when $\alpha$ drops. Let the \textit{oracle estimator} be
\begin{equation}
    \label{eq:orc-def-nonlinear}
    \theta^o = (\alpha^o, \beta^{o T})^T \in \arg\min_{\substack{\alpha, \beta\\ \beta_{S^c} = 0}} \ell\left( \alpha, \beta \right),
\end{equation}
which is an optimal estimate of $\theta^{\star}$. GLBI starts within the oracle subspace ($\{\theta = (\alpha, \beta^T)^T:\ \beta_{S^c} = 0\}$), and we are going to prove that under an Irrepresentable Condition (IRR) the dynamics will evolve in the oracle subspace with high probability before the stopping time $\bar{k}$, approaching the oracle estimator exponentially fast due to the Restricted Strong Convexity (RSC). Thus if all the true parameters are large enough, then we can identify their sign pattern correctly; otherwise, such a stopping time still finds an estimator (possibly with false positives) at minimax optimal $\ell_2$ error rate. Furthermore, if the algorithm continues beyond the stopping time, it might escape the oracle subspace and eventually reach overfitted estimates. Such a picture is illustrated in \Cref{fig:saddle}.

\begin{figure}[h]
    \centering
    \includegraphics[width = 0.75\linewidth]{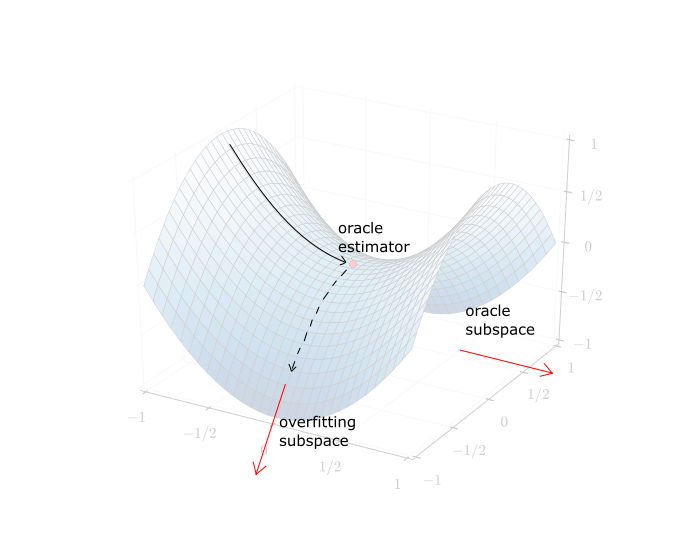}
    \caption{An illustration of global dynamics of the algorithm in this paper.}
    \label{fig:saddle}
\end{figure}

Hence, it is helpful to define the following \textbf{\emph{oracle dynamics}}: 
\begin{subequations}
    \label{eq:glbi-orc-show}
    \begin{align}
        \label{eq:glbi-orc-show-a}
        \alpha_{k+1}' &= \alpha_k' - \kappa \delta \nabla_{\alpha} \ell\left( \alpha_k', \beta_k' \right),\\
        \label{eq:glbi-orc-show-b}
        z_{k+1,S}' &= z_{k,S}' - \delta \nabla_S \ell \left( \alpha_k', \beta_k' \right),\\
        \label{eq:glbi-orc-show-c}
        \beta_{k+1,S}' &= \kappa \mathcal{S}\left( z_{k+1,S}', 1 \right),
    \end{align}
\end{subequations}
with $z_{k,S^c}' = \beta_{k,S^c}' \equiv 0_{p-s}$. Let $\theta_k' := (\alpha_k', \beta_k'^T)^T$.

\subsection{Basic Assumptions}

Now we are ready to state the general assumptions that can be reduced to existing ones. We write
\begin{gather*}
    \ell(\theta) := \ell\left( \alpha, \beta \right),\\
    \bar{H}(\theta) := \bar{H}(\alpha, \beta) := \int_0^1 \nabla^2 \ell\left( \theta^{\star} + \mu \left( \theta - \theta^{\star} \right) \right) \mathrm{d}\mu,\\
    \bar{H}^o(\theta) := \bar{H}^o(\alpha, \beta) := \int_0^1 \nabla^2 \ell\left( \theta^o + \mu \left( \theta - \theta^o \right) \right) \mathrm{d}\mu.
\end{gather*}

\begin{assumption}[Restricted Strong Convexity (RSC)]
    \label{thm:rsc-glbi}
    There exist $\lambda, \Lambda > 0$, such that for any $k\ge 0$, and for any $\theta$ on the line segment between $\theta_k'$ and $\theta^o$, or on the line segment between $\theta^{\star}$ and $\theta^o$,
    \begin{equation*}
        \lambda I \preceq \nabla_{S_\alpha, S_\alpha}^2 \ell(\theta) \preceq \Lambda I,
    \end{equation*}
\end{assumption}

\begin{assumption}[Irrepresentable Condition (IRR)]
    \label{thm:irr-glbi}
    There exist $\eta \in (0,1]$ and $C>0$ such that
    \begin{gather*}
        \sup_{K\ge 1} \left\| \sum_{k=0}^{K-1} \overline{\mathrm{irr}}_k \left( \begin{pmatrix} \alpha_{k+1}' / \kappa \\ z_{k+1,S}' \end{pmatrix} - \begin{pmatrix} \alpha_k' / \kappa \\ z_{k,S}' \end{pmatrix} \right) \right\|_{\infty} < 1 - \frac{\eta}{2},\\
        \sup_{k\ge 0} \left\| \overline{\mathrm{irr}}_k \right\|_{\infty} \le C,
    \end{gather*}
    where
    \begin{equation*}
        \overline{\mathrm{irr}}_k := \bar{H}_{S^c, S_{\alpha}}\left( \theta_k' \right) \cdot \bar{H}_{S_{\alpha}, S_{\alpha}} \left( \theta_k' \right)^{-1}.
    \end{equation*}
\end{assumption}

\begin{remark}
    For sparse linear regression problem (\Cref{thm:model-linear}) with no intercept ($\alpha$ drops), \Cref{thm:rsc-glbi} reduces to $\lambda I \preceq X_S^{*} X_S \preceq \Lambda I$. The lower bound is exactly the RSC proposed in linear problems. Although the upper bound is not needed in linear problems, it arises in the analysis for logistic problem by \citet{ravikumar_high-dimensional_2010} (see (A1) in Section 3.1 in their paper). Besides, $\overline{\mathrm{irr}}_k$ is constant and \Cref{thm:irr-glbi} reduces to
    \begin{gather*}
        \sup_{K\ge 1}\left\| X_{S^c}^{*} X_S \left( X_S^{*} X_S \right)^{-1} z_{K,S}' \right\|_{\infty} < 1 - \frac{\eta}{2},\\
        \left\| X_{S^c}^{*} X_S \left( X_S^{*} X_S \right)^{-1} \right\|_{\infty} \le C,
    \end{gather*}
    which is true with high probability, as long as the classical Irrepresentable Condition \citep{Zhao_model_2006} $\| X_{S^c}^{*} X_S (X_S^{*} X_S)^{-1} \|_{\infty} \le 1 - \eta$ holds along with $C \ge 1$ and $\kappa$ is large, since by \cref{eq:glbi-orc-cstc-l2},
    \begin{align*}
        \left\| z_{K,S}' \right\|_{\infty} &\le \left\| z_{K,S}' - \mathcal{S}\left( z_{K,S}', 1 \right) \right\|_{\infty} + \left\| \mathcal{S}\left( z_{K,S}', 1 \right) \right\|_{\infty}\\
        & \le 1 + \left\| \beta_{K,S}' \right\|_{\infty} / \kappa\\
        & \le 1 + \left( \left\| \beta_{K,S}' - \beta_S^o \right\|_2 + \left\| \beta_S^o \right\|_2 \right) / \kappa\\
        & \le 1 + \left( \sqrt{\Lambda / \lambda} + 1 \right) \left\| \beta_S^o \right\|_2 / \kappa\\
        &< (1 - \eta / 2) / (1 - \eta).
    \end{align*}
\end{remark}

\begin{remark}
    For sparse logistic regression problem (\Cref{thm:model-logistic}), we have the following proposition stating that \Cref{thm:rsc-glbi} and \ref{thm:irr-glbi} hold with high probability under some natural setting, along with condition \cref{eq:rsc-irr-logistic-cond}. See its proof in \Cref{sec:rsc-irr-logistic}. A slightly weaker condition compared to \cref{eq:rsc-irr-logistic-cond-b}, and a same version of \cref{eq:rsc-irr-logistic-cond-c}, can be found in \citet{ravikumar_high-dimensional_2010}, where $x^{(i)}$ are discrete.
\end{remark}

\begin{proposition}
    \label{thm:rsc-irr-logistic}
    In \Cref{thm:model-logistic}, we suppose $x^{(i)}$'s are i.i.d. drawn from some $X \sim N(0, \Sigma)$, where $\Sigma_{j,j}\le 1\ (1\le j\le p)$. Then there exist constants $C_0, C_1, C_2 > 0$, such that \Cref{thm:rsc-glbi} and \ref{thm:irr-glbi} hold with probability not less than $1 - C_0/p$, as long as $\kappa$ is sufficiently large and
    \begin{subequations}
        \label{eq:rsc-irr-logistic-cond}
        \begin{gather}
            \label{eq:rsc-irr-logistic-cond-a}
            \left\| \Sigma_{S^c, S} \Sigma_{S,S}^{-1} \right\|_{\infty} \le 1 - \eta,\\
            \label{eq:rsc-irr-logistic-cond-b}
            n / (\log n)^2 \ge C_1 s^4 \log p,\\
            \label{eq:rsc-irr-logistic-cond-c}
            \beta_{\min}^\star := \min_{j\in S} \left| \beta_j^\star \right| \ge C_2 \sqrt{(s\log p)/n}.
        \end{gather}
    \end{subequations}
\end{proposition}

\subsection{Path Consistency Theorem}

\begin{theorem}[Consistency of GLBI]
    \label{thm:glbi-cstc}
    Under \Cref{thm:rsc-glbi} and \ref{thm:irr-glbi}, suppose $\kappa\ge 2 \|\theta_{S_\alpha}^o\|_2$, and $\bar{k}\in \mathbb{N}$ such that
    \begin{equation}
        \label{eq:glbi-kdelta-upper}
        \left( \bar{k} - 1 \right) \delta < \frac{\eta}{2(C+1)} \cdot \frac{1}{\left\| \nabla \ell\left( \alpha^{\star}, \beta^{\star} \right) \right\|_{\infty}} \le \bar{k} \delta.
    \end{equation}
    Define $\lambda'$ as in \cref{eq:lambda-prime}. We have the following properties.

    \textit{No-false-positive}: For all $0\le k \le \bar{k}$, the solution path of GLBI has no false-positive, i.e. $\beta_{k, S^c} = 0$.

    \textit{Sign consistency}: If $\bar{k} \ge 5$ and
    \begin{multline}
        \label{eq:beta-min-lower}
        \beta_{\min}^{\star} := \min_{j\in S} \left| \beta_j^{\star} \right| \ge \max \Bigg( 2 \left\| \beta_S^o - \beta_S^{\star} \right\|_{\infty}, \Bigg.\\
        \Bigg. \frac{(8\log s + 18) (C+1)}{\lambda' \eta \left( 1 - 4 / \bar{k} \right)} \left\| \nabla \ell\left( \alpha^{\star}, \beta^{\star} \right) \right\|_{\infty} \Bigg),
    \end{multline}
    then $\mathrm{sign}(\beta_{\bar{k}}) = \mathrm{sign}(\beta^{\star})$.

    \textit{$\ell_2$ consistency}: The $\ell_2$ error
    \begin{equation*}
        \left\| \begin{pmatrix} \alpha_{\bar{k}} - \alpha^{\star}\\ \beta_{\bar{k}} - \beta^{\star} \end{pmatrix} \right\|_2 \le \frac{20 (C+1) \sqrt{s}}{\lambda' \eta} \left\| \nabla \ell\left( \alpha^{\star}, \beta^{\star} \right) \right\|_{\infty}.
    \end{equation*}
\end{theorem}

The proof of \Cref{thm:glbi-cstc} is collected in \Cref{sec:proof-glbi-cstc}, which largely follows the analysis of differential inclusion \Cref{eq:glbi-show}, given in \Cref{sec:glbiss-cstc}, as its discretization. 

\begin{remark}
    For sparse linear regression problem (\Cref{thm:model-linear}), with high probability we have
    \begin{equation*}
        \left\| \nabla \ell\left( \alpha^{\star}, \beta^{\star} \right) \right\|_{\infty} \lesssim \sqrt{\frac{\log p}{n}},\ \left\| \beta_S^o - \beta_S^{\star} \right\|_{\infty} \lesssim \sqrt{\frac{\log s}{n}}.
    \end{equation*}
    Hence pick $\bar{k} \delta \sim \sqrt{n / \log p}$ satisfying \cref{eq:glbi-kdelta-upper}, by \Cref{thm:glbi-cstc} the sign consistency is guaranteed at $\bar{k}$ if
    \begin{equation*}
        \beta_{\min}^{\star} \gtrsim (\log s) \sqrt{(\log p)/n}.
    \end{equation*}
    The $\ell_2$ error bound reaches the minimax optimal rate:
    \begin{equation*}
        \left\| \begin{pmatrix} \alpha_{\bar{k}} - \alpha^o\\ \beta_{\bar{k}} - \beta^o \end{pmatrix} \right\|_2 \lesssim \sqrt{\frac{s\log p}{n}}.
    \end{equation*}
\end{remark}

\begin{remark}
    For sparse logistic regression problem (\Cref{thm:model-logistic}), with high probability we have
    \begin{align*}
        \left\| \nabla \ell\left( \alpha^{\star}, \beta^{\star} \right) \right\|_{\infty} & \lesssim \sqrt{(\log p)/n},\\
        \left\| \beta_S^o - \beta_S^{\star} \right\|_{\infty} & \lesssim \left\| \beta_S^o - \beta_S^{\star} \right\|_2 \lesssim \sqrt{s} \left\| \nabla \ell\left( \alpha^{\star}, \beta^{\star} \right) \right\|_{\infty}\\
        & \lesssim \sqrt{(s \log p)/n}.
    \end{align*}
    Hence the sign consistency is guaranteed at some $\bar{k} \sim \sqrt{n / \log p}$ if
    \begin{equation*}
        \beta_{\min}^{\star} \gtrsim \sqrt{(s \log p)/n}
    \end{equation*}
    (meeting Condition (19) in \citet{ravikumar_high-dimensional_2010}). The $\ell_2$ error rate $\lesssim \sqrt{(s\log p) / n}$ is minimax optimal.
\end{remark}

\section{Experiments}
\label{sec:exp}


As for the setting of algorithm parameters: $\kappa$ should be large, and then $\delta \sim 1 / (\kappa \Lambda)$ is automatically calculated based on $\kappa$ (as long as $\kappa\delta \Lambda < 2$, such that $\lambda'$ is positive in \cref{eq:lambda-prime}). In practice, a small $\delta$ can prevent the iterations from oscillations.

\subsection{Efficiency of Parallel Computing}
\label{sec:parallel}

\citet{osher_sparse_2016} has elaborated that LBI can easily be implemented in parallel and distributed manners, and applied on very large-scale datasets. Likewise, GLBI can be parallelized in many usual applications. We now take the logistic model \Cref{thm:model-logistic} as an example to explain the details. The iteration \cref{eq:glbi-show} (generally taking $\delta_k = \delta$) can be written as
\begin{subequations}
    \label{eq:glbi-show-logistic}
    \begin{align}
        \alpha_{k+1} &= \alpha_k - \kappa \delta f\left( \alpha_k, X \beta_k \right),\\
        z_{k+1} &= z_k - \delta X^T g\left( \alpha_k, X \beta_k \right),\\
        \beta_{k+1} &= \kappa \mathcal{S}\left( z_{k+1}, 1 \right),
    \end{align}
\end{subequations}
where $f: \mathbb{R}^{n+1} \rightarrow \mathbb{R},\ g: \mathbb{R}^{n+1} \rightarrow \mathbb{R}^n$ such that
\begin{align*}
    g\left( \alpha, w \right)_i :={}& - \frac{1}{n} \cdot \frac{1}{1 + \exp\left( - \left( \alpha + w_i \right) y^{(i)} \right)} y^{(i)} \in \mathbb{R}\\
    &\ \left( 1\le i \le n;\ w \in \mathbb{R}^n \right)\\
    f\left( \alpha, w \right) :={}& 1_n^T \cdot g (\alpha, w)\\
    ={}& - \frac{1}{n} \sum_{i=1}^n \frac{1}{1 + \exp\left( - \left( \alpha + w_i \right) y^{(i)} \right)} y^{(i)}.
\end{align*}
Suppose
\begin{equation*}
    X = \left[ X_1, X_2, \ldots, X_L \right] \in \mathbb{R}^{n\times p},
\end{equation*}
where $X_l$'s are submatrices stored in a distributed manner on a set of networked workstations. The sizes of $X_l$'s are flexible and can be chosen for good load balancing. Let each workstation $l$ hold data $y$ and $X_l$, and variables $z_{k,l}$ and $X_l \beta_{k,l}$ which are parts of $z_k$ and summands of $w_k := X \beta_k$, respectively. The iteration \cref{eq:glbi-show-logistic} is carried out as
\begin{gather*}
    \begin{cases}
        \alpha_{k+1} = \alpha_k - \kappa \delta f(\alpha_k,w_k),\\
        z_{k+1,l} = z_{k,l} - \delta X_l^T g(\alpha_k,w_k),\\
        w_{k+1,l} = \kappa X_l \mathcal{S}(z_{k+1,l}, 1)
    \end{cases}
    \text{in parallel for $l$}\\
    w_{k+1} = \sum_{l=1}^L w_{k+1,l}\ \ \ \ \text{(all-reduce summation)},
\end{gather*}
where the all-reduce summation step collects inputs from and then returns the sum to all the $L$ workstations. It is the sum of $L$ $n$-dimensional vectors. Therefore, the communication cost is independent of $p$ no matter how the all-reduce step is implemented. It is important to note that the algorithm is not changed at all. particularly, increasing $L$, does not increase the number of iterations. So the parallel implementation is truly scalable.

If $t_L$ denotes the time cost of a single GLBI run with $L$ workstations under the same dataset and the same algorithmic settings, it is expected that $t_L \sim 1 / L$. Here we show this by an example. Construct a logistic model in \Cref{sec:simu-logistic} with $M = 1,\ r = 0.25$, and three settings for $(p,s,n)$: (I) $(2000,\ 200,\ 6000)$, (II) $(5000,\ 500,\ 15000)$, (III) $(10000,\ 1000,\ 30000)$. For each setting, we run our parallelized version of GLBI algorithm written in C++, with $\kappa = 10,\ \delta = 0.1,\ k_{\max} = 1000 k_0$, where $k_0$ is the maximal $k$ such that $\beta_1 = \cdots = \beta_k = 0$, and the path is early stopped at the $k_{\max}$-th iteration. The recorded $t_L$'s are shown in \Cref{fig:parallel-time-cost}. The left panel shows $t_L$ (in seconds) while the right panel shows $t_1 / t_L$, for $L = 1, \ldots, 8$. We see truly $t_L \sim 1/L$, which is expected in our parallel and distributed treatment. When $L$ is large, our package can deal with very large scale problems.

\begin{figure}[h]
    \centering
    \includegraphics[width = 0.8 \linewidth]{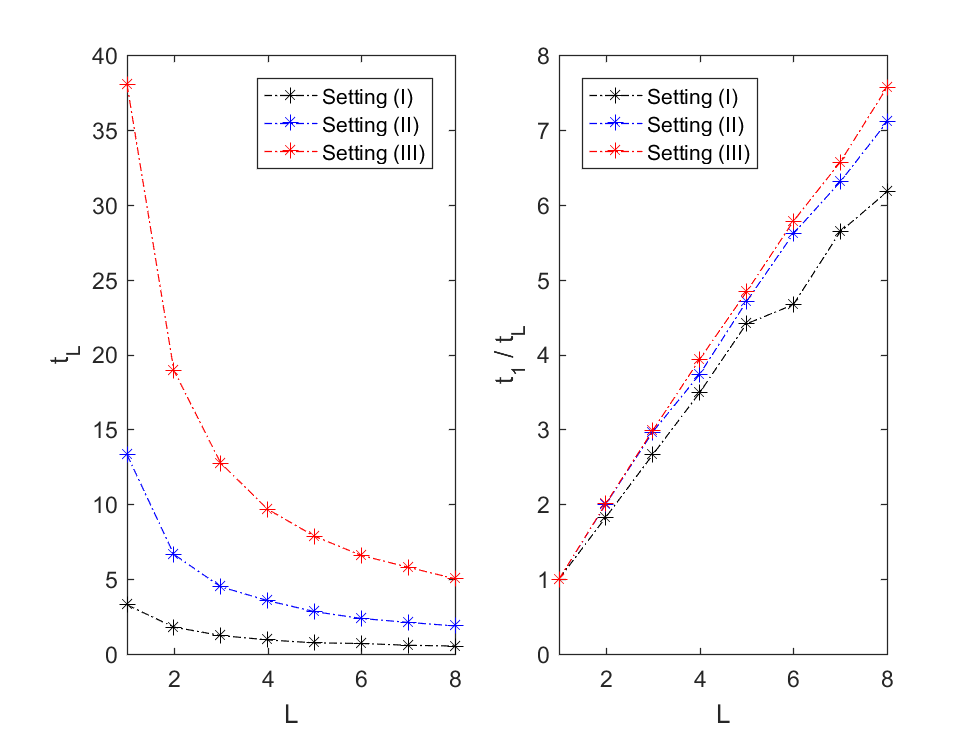}
    \caption{Time cost illustration for logistic model with three settings (black for Setting (I), blue for Setting (II) and red for Setting (III)). In each setting, the left panel shows $t_L$ while the right panel shows $t_1 / t_L$, for $L = 1, \ldots, 8$.}
    \label{fig:parallel-time-cost}
\end{figure}

\subsection{Application: Logistic Model}
\label{sec:simu-logistic}

We do $\mathrm{rep} = 20$ independent experiments, in each of which we construct a logistic model (\Cref{thm:model-logistic}), and then compare GLBI with other methods. Specifically, suppose that $\beta^{\star}$ has a support set $S = \{1,\ldots,s\}$ without loss of generality. $\alpha^{\star}, \beta_j^{\star}\ (j\in S)$ are independent, each has a uniform distribution on $[-2M,-M]\cup[M,2M]$. Each row of $X \in \mathbb{R}^{n\times p}$ is i.i.d. sampled from $N(0,\Sigma)$, where $\Sigma$ is a Toeplitz matrix satisfying $\Sigma_{j,k} = r^{|j-k|}$. When $X$ and $(\alpha^{\star}, \beta^{\star})$ are determined, we generate $y\in \mathbb{R}^n$ as in \Cref{thm:model-logistic}.

After getting the sample $(X,y)$, consider GLBI \cref{eq:glbi-show} and $\ell_1$ optimization \cref{eq:pen-opt}, both with logistic loss \cref{eq:loss-logistic}. For GLBI, set $\kappa = 10$. For \cref{eq:pen-opt}, apply a grid search for differently penalized problems, for which we use \texttt{glmnet} -- a popular package available in Matlab/R that can be applied on $\ell_1$ regularization for sparse logistic regression models.

For each algorithm, we use $K$-fold ($K = 5$) cross validation (CV) to pick up an estimator from the path calculate based on the smallest CV estimate of \emph{prediction error}. Specifically, we split the data into $K$ roughly equal-sized parts. For a certain position on paths ($t$ for GLBI, or $\lambda$ for \texttt{glmnet}) and $k\in \{1,\ldots,K\}$, we obtain a corresponding estimator based on the data with the $k$-th part removed, use the estimator to build a classifier, and get the mis-classification error on the $k$-th part. Averaging the value for $k \in \{1,\ldots,K\}$, we obtain the CV estimate of prediction error, for the obtained estimator corresponding to a certain position on paths. Among all positions, we pick up the estimator producing the smallest CV estimate of prediction error. Besides, we calculate \emph{AUC (Area Under Curve)}, for evaluating the path performance of learning sparsity patterns without choosing a best estimator. 

Results for $p = 80,\ s = 20,\ M = 1$ are summarized in \Cref{tab:simu-logistic-p80-s20-m1}. We see that in terms of CV estimate of prediction error, GLBI is generally better than \texttt{glmnet}. Besides, GLBI is competitive with $\ell_1$ regularization method in variable selection, in terms of AUC. Similar observations for more settings are listed in \Cref{tab:simu-logistic-p80-s20-m2}, \ref{tab:simu-logistic-p200-s50-m1} and \ref{tab:simu-logistic-p200-s50-m2} in \Cref{sec:exp-ext}. Apart from these tables, we can also see the outperformance of CV estimate of prediction error \Cref{fig:simu-logistic-cv} in \Cref{sec:exp-ext}, while in that figure we can see that GLBI further reduces bias, as well as provides us a relatively good estimator with small prediction error if a proper early stopping is equipped.

\begin{table}[h]
    \caption{Comparisons between GLBI and \texttt{glmnet}, for logistic models with $p = 80,\ s = 20,\ M = 1$. For each algorithm, we run $\mathrm{rep} = 20$ independent experiments.}
    \label{tab:simu-logistic-p80-s20-m1}
    \centering
    \begin{tabular}{ccccccc}
        \toprule
        & & \multicolumn{2}{c}{AUC} & \multicolumn{2}{c}{prediction error}\\
        \cmidrule(lr){3-4} \cmidrule(lr){5-6}
        $r$ & $n$ & GLBI & \texttt{glmnet} & GLBI & \texttt{glmnet}\\
        \cmidrule(lr){1-6}
        $0.25$ & $400$ & $.9902$ & $\boldsymbol{.9906}$ & $\boldsymbol{.1221}$ & $.1355$\\
               &       & $(.0065)$ & $(.0062)$ & $(.0218)$ & $(.0223)$\\
               & $800$ & $\boldsymbol{.9991}$ & $.9990$ & $\boldsymbol{.1082}$ & $.1132$\\
               &       & $(.0020)$ & $(.0022)$ & $(.0125)$ & $(.0104)$\\
        \cmidrule(lr){1-6}
        $0.5$  & $400$ & $\boldsymbol{.9690}$ & $.9681$ & $\boldsymbol{.1321}$ & $.1379$\\
               &       & $(.0180)$ & $(.0165)$ & $(.0268)$ & $(.0289)$\\
               & $800$ & $\boldsymbol{.9925}$ & $.9921$ & $\boldsymbol{.1139}$ & $.1197$\\
               &       & $(.0069)$ & $(.0076)$ & $(.0138)$ & $(.0134)$\\
        \bottomrule
    \end{tabular}
\end{table}

\subsection{Application: Ising Model with 4-Nearest-Neighbor Grid}
\label{sec:simu-ising}

\begin{table}[h]
    \caption{Comparisons of GLBI1 (GLBI + composite), GLBI2 (GLBI + MPF), and \texttt{glmnet}, for Ising models with $p = 36$. For each algorithm, we run $\mathrm{rep} = 20$ independent experiments.}
    \label{tab:simu-ising-p36}
    \centering
    \begin{tabular}{ccccc}
        \toprule
        & & \multicolumn{3}{c}{AUC}\\
        \cmidrule(lr){3-5}
        $T$ & $n$ & GLBI1 & GLBI2 & \texttt{glmnet}\\
        \cmidrule(lr){1-5}
        $1.25$ & $500$ & $.9754$ & $\boldsymbol{.9867}$ & $.9774$\\
               &       & $(.0277)$ & $(.0128)$ & $(.0265)$\\
        \cmidrule(lr){2-5}
               & $750$ & $.9868$ & $\boldsymbol{.9919}$ & $.9891$\\
               &       & $(.0137)$ & $(.0082)$ & $(.0134)$\\
        \cmidrule(lr){1-5}
        $1.5$  & $500$ & $.9915$ & $\boldsymbol{.9963}$ & $.9929$\\
               &       & $(.0110)$ & $(.0033)$ & $(.0104)$\\
        \cmidrule(lr){2-5}
               & $750$ & $.9963$ & $\boldsymbol{.9980}$ & $.9975$\\
               &       & $(.0041)$ & $(.0029)$ & $(.0042)$\\
        \bottomrule
    \end{tabular}
    \begin{tabular}{ccccc}
        \toprule
        & & \multicolumn{3}{c}{2nd order MDC}\\
        \cmidrule(lr){3-5}
        $T$ & $n$ & GLBI1 & GLBI2 & \texttt{glmnet}\\
        \cmidrule(lr){1-5}
        $1.25$ & $500$ & $\boldsymbol{.9762}$ & $.9758$ & $.9744$\\
               &       & $(.0079)$ & $(.0079)$ & $(.0086)$\\
        \cmidrule(lr){2-5}
               & $750$ & $\boldsymbol{.9840}$ & $.9830$ & $.9827$\\
               &       & $(.0053)$ & $(.0066)$ & $(.0061)$\\
        \cmidrule(lr){1-5}
        $1.5$  & $500$ & $\boldsymbol{.9655}$ & $.9646$ & $.9630$\\
               &       & $(.0087)$ & $(.0099)$ & $(.0094)$\\
        \cmidrule(lr){2-5}
               & $750$ & $\boldsymbol{.9774}$ & $.9766$ & $.9756$\\
               &       & $(.0060)$ & $(.0066)$ & $(.0070)$\\
        \bottomrule
    \end{tabular}
\end{table}

We do $\mathrm{rep} = 20$ independent experiments, in each of which we construct an ising model (\Cref{thm:model-ising}), and then compare GLBI with other methods. Specifically, construct an $N \times N$ 4-nearest neighbor grid (with aperiodic boundary conditions) to be graph $G$, with node set $V$ and edge set $E$. The distribution of a random vector $x$ is given by \cref{eq:model-ising} ($p = N^2$), where $\alpha_j^{\star}$'s and $\beta_{j,j'}^{\star}$'s $( (j,j')\in E )$ are i.i.d. and each has a uniform distribution on $[-2/T, -1/T] \cup [1/T, 2/T]$. Let $X\in \mathbb{R}^{n\times p}$ represents $n$ samples drawn from the distribution of $x$ via Gibbs sampling.

After getting the sample $X$, consider GLBI1 (GLBI with composite loss \cref{eq:loss-ising-composite}), GLBI2 (GLBI with MPF loss \cref{eq:loss-ising-mpf}) and $\ell_1$ optimization \cref{eq:pen-opt} with logistic loss (see \Cref{thm:model-ising} for \emph{neighborhood-based logistic regression} applied on Ising models, or see \citet{ravikumar_high-dimensional_2010}). For GLBI1 and GLBI2, set $\kappa = 10$. For \cref{eq:pen-opt}, apply a grid search for differently penalized problems; we still use \texttt{glmnet}.

For each algorithm, we calculate the \emph{AUC (Area Under Curve)}, popular for evaluating the path performance of learning sparsity patterns. Besides, we apply $K$-fold ($K = 5$) cross validation (CV) to pick up an estimator from the path, with the \emph{largest CV estimate of 2nd order marginal distribution correlation (2nd order MDC)} in the same way as the CV process done in \Cref{sec:simu-logistic}, here the 2nd order MDC, defined in the next paragraph, is calculated based on two samples of the same size: the $k$-th part original data, and the newly sampled Ising model data (based on learned parameters) with the same size of the $k$-th part.

For any sample matrix $X'\in \{1,-1\}^{n'\times p}$, we construct $d_2(X')$, the 2nd marginal empirical distribution matrix of $X'$, defined as follows. $d_2(X') = (d_2(X')_{[j_1,j_2]})_{p\times p} \in \mathbb{R}^{2p\times 2p}$, where
\begin{equation}
    \label{eq:mgn-dist}
    \begin{aligned}
        &d_2\left( X' \right)_{[j_1, j_2]}\\
        = {}& \frac{1}{n'} \sum_{i=1}^{n'} \begin{pmatrix} 1_{\left( x_{j_1}^{(i)}, x_{j_2}^{(i)} \right) = (1,1)} & 1_{\left( x_{j_1}^{(i)}, x_{j_2}^{(i)} \right) = (1,-1)} \\ 1_{\left( x_{j_1}^{(i)}, x_{j_2}^{(i)} \right) = (-1,1)} & 1_{\left( x_{j_1}^{(i)}, x_{j_2}^{(i)} \right) = (-1,-1)} \end{pmatrix}.
    \end{aligned}
\end{equation}
For any sample matrices $X_1, X_2$ with the same sample size, we call the correlation between $\mathrm{vec}(d_2(X_1))$ and $\mathrm{vec}(d_2(X_2))$ \emph{the 2nd order marginal distribution correlation (2nd order MDC)}. This value is expected to be large as well as close to $1$ if $X_1, X_2$ come from the same model.

Results for $p = N^2 = 36$ are summarized in \Cref{tab:simu-ising-p36}. GLBI with composite/MPF loss are competitive with or better than \texttt{glmnet}. Similar observations are listed in \Cref{tab:simu-ising-p100} in \Cref{sec:exp-ext}.

\subsection{Application: Coauthorship Network in NIPS}
\label{sec:real-nips}

Consider the information of papers and authors in \emph{Advances in Neural Information Processing Systems} (NIPS) 1987--2016, collected from \url{https://www.kaggle.com/benhamner/nips-papers}. After preprocessing (e.g. author disambiguity), for simplicity, we restrict our analysis on the most productive $p = 30$ authors (\Cref{tab:real-nips}) in the largest connected component of a coauthorship network that two authors are linked if they coauthored at least 2 papers (Coauthorship (2)). The first panel of \Cref{fig:real-nips} shows this coauthorship network with edge width in proportion to the number of coauthored papers. 
There are $n = 1,028$ papers authored by at least one of these persons.

\begin{table}[h]
    \caption{Most productive $p = 30$ authors in the largest connected component of Coauthorship (2).}
    \centering
    \small
    \begin{tabular}{ll}
        \toprule
        01 Michael Jordan & 16 Inderjit Dhillon \\
        \cmidrule(lr){1-2}
        02 Bernhard Sch\"{o}lkopf & 17 Ruslan Salakhutdinov \\
        \cmidrule(lr){1-2}
        03 Geoffrey Hinton & 18 Tong Zhang\\
        \cmidrule(lr){1-2}
        04 Yoshua Bengio & 19 Thomas Griffiths\\
        \cmidrule(lr){1-2}
        05 Zoubin Ghahramani & 20 David Blei\\
        \cmidrule(lr){1-2}
        06 Terrence Sejnowski & 21 R\'{e}mi Munos\\
        \cmidrule(lr){1-2}
        07 Peter Dayan & 22 Joshua Tenenbaum\\
        \cmidrule(lr){1-2}
        08 Alex Smola & 23 Lawrence Carin\\
        \cmidrule(lr){1-2}
        09 Andrew Ng & 24 Eric Xing\\
        \cmidrule(lr){1-2}
        10 Francis Bach & 25 Richard Zemel\\
        \cmidrule(lr){1-2}
        11 Michael Mozer & 26 Martin Wainwright\\
        \cmidrule(lr){1-2}
        12 Pradeep Ravikumar & 27 Yoram Singer\\
        \cmidrule(lr){1-2}
        13 Tommi Jaakkola & 28 Han Liu\\
        \cmidrule(lr){1-2}
        14 Klaus-Robert M\"{u}ller & 29 Satinder Singh\\
        \cmidrule(lr){1-2}
        15 Yee Teh & 30 Christopher Williams\\
        \bottomrule
    \end{tabular}
    \label{tab:real-nips}
\end{table}

Let the $j$-th entry of $x^{(i)} \in \mathbb{R}^p$ be $1$ if the $j$-th person is involved in the authors of the $i$-th paper, and $-1$ otherwise. Now we fit the data $x^{(i)}\ (1\le i\le n)$ by a sparse Ising model \cref{eq:model-ising} with parameter $(\hat{\alpha},\hat{\beta})$. Note that $\hat{\beta}_{j,j'} = 0$ indicates that $j$ and $j'$ are conditional independent on coauthorship, given all the other authors; $\hat{\beta}_{j,j'}>0$ implies that $j$ and $j'$ coauthored more often than their averages, while $\hat{\beta}_{j,j'}<0$ says the opposite. 

\begin{figure}[h]
    \centering
    \includegraphics[width = 0.48\linewidth]{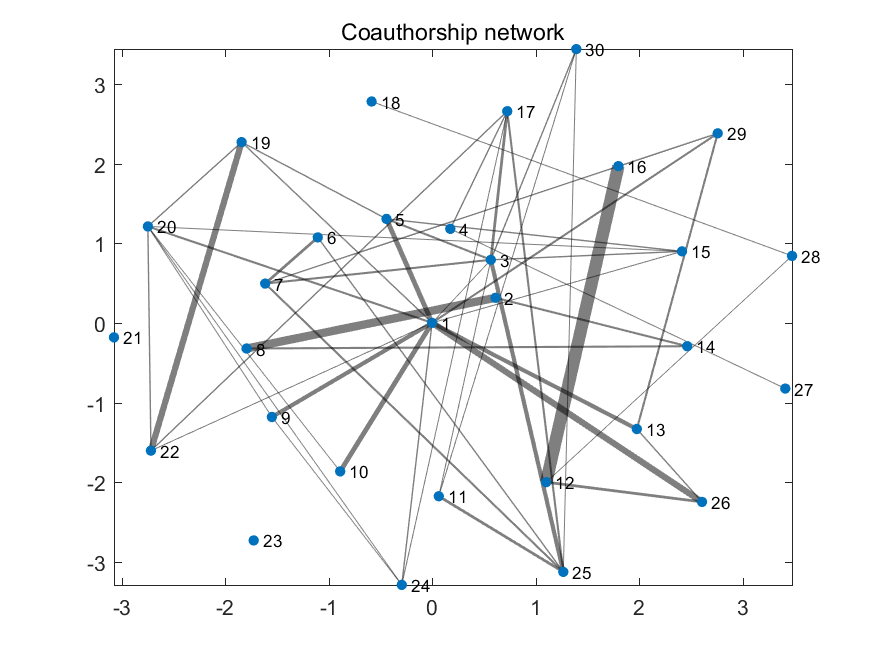}
    \includegraphics[width = 0.48\linewidth]{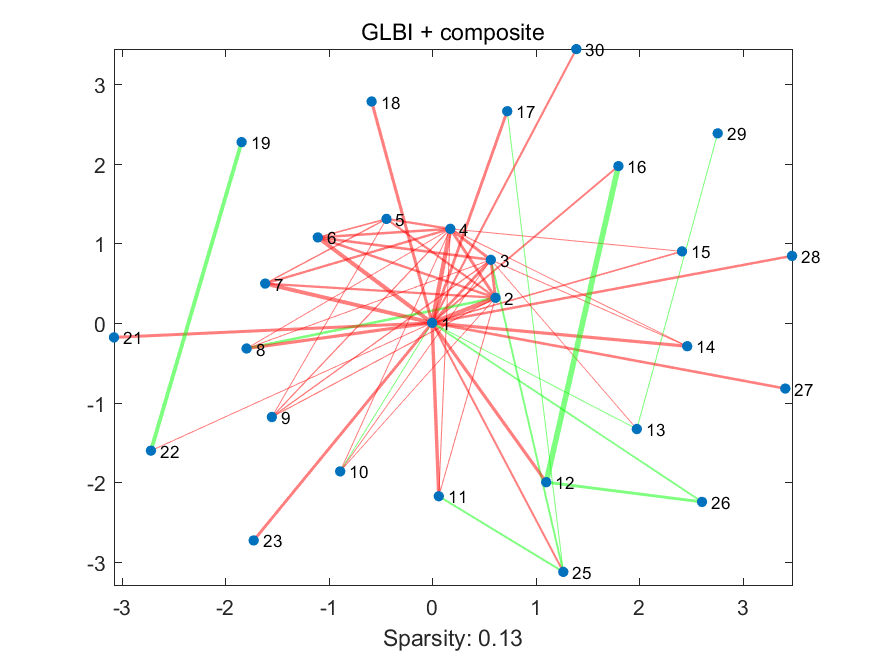}
    \includegraphics[width = 0.48\linewidth]{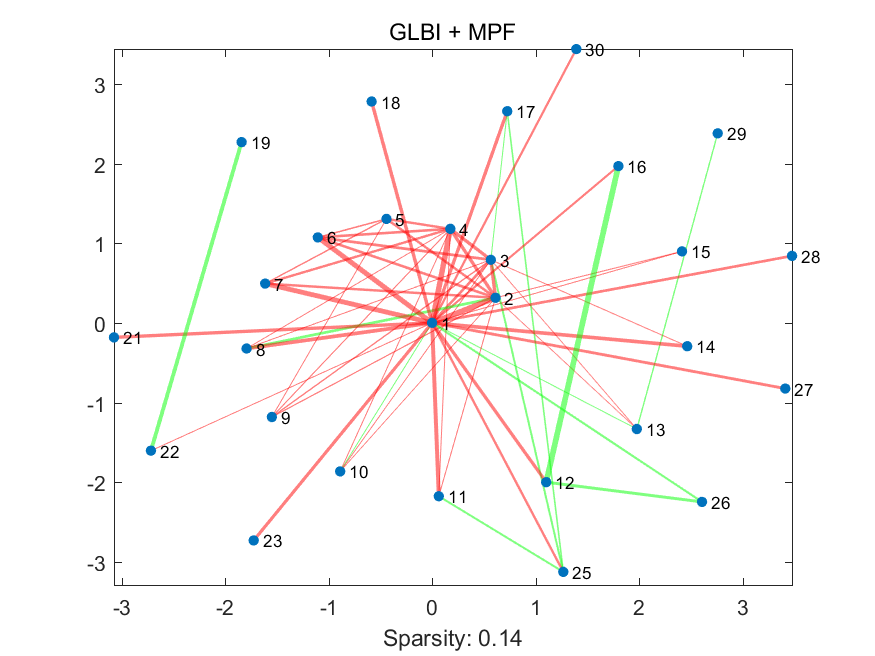}
    \includegraphics[width = 0.48\linewidth]{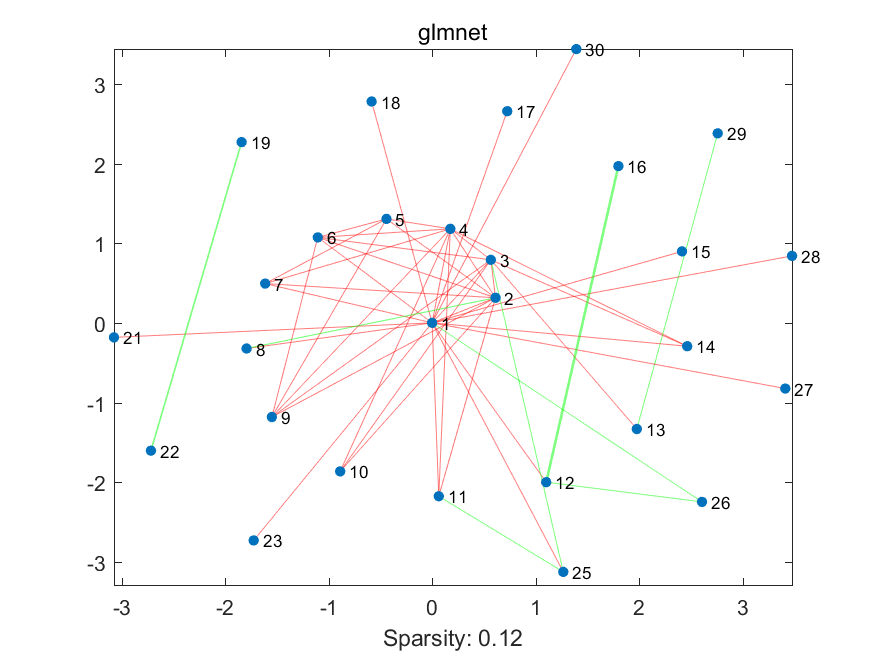}
    \caption{Top left: NIPS coauthorship network, with edge width in proportion to the number of coauthored papers. Top right: a learned graph picked from the path of GLBI1. Bottom left: from GLBI2. Bottom right: from \texttt{glmnet}. Green edges indicate positive conditional dependence of coauthorship -- the probability of coauthoring a paper significantly increases the authors' average behavior, while red edges indicating the negative coauthorship. Edge widths show the strength of such a relationship.} 
    \label{fig:real-nips}
\end{figure}

The right three panels in \Cref{fig:real-nips} compares some sparse Ising models chosen from three regularization paths at a similar sparsity level (the percentage of learned edges over the complete graph, here about $12\% \sim 14\%$): GLBI1 (GLBI with composite loss), GLBI2 (GLBI with MPF loss), and $\ell_1$ regularization (\texttt{glmnet}), respectively. For more learned graphs from these paths, see \Cref{fig:real-nips-ext} in \Cref{sec:exp-ext}. In GLBI1 and GLBI2, set $\kappa = 10$. 

We see that all the learned graphs capture some important coauthorships, such as Pradeep Ravikumar (12) and Inderjit Dhillon (16) in a thick green edge in all the three learned graphs, indicating that they collaborated more often than separately for NIPS. Besides, the most productive author Michael Jordan (01) has coauthored with a lot of other people, but is somewhat unlikely to coauthor with several other productive scholars like Yoshua Bengio (04), Terrence Sejnowski (06), etc., indicating by the red edges between Jordan and those people. Further note the edge widths in the second and third graphs are significantly larger than those in the fourth graph, implying that at a similar sparsity level, GLBI tends to provide an estimator with larger absolute values of entries than that by \texttt{glmnet}. That is because under similar sparsity patterns, GLBI may give \emph{low-biased} estimators.

%

\subsubsection*{Acknowledgements}

This work is supported in part by National Basic Research Program of China (Nos. 2015CB85600, 2012CB825501), NNSF of China (Nos. 61370004, 11421110001), HKRGC grant 16303817, as well as grants from Tencent AI Lab, Si Family Foundation, Baidu BDI, and Microsoft Research-Asia.

\newpage

\

\newpage


\bibliography{ref_local.bib}
\bibliographystyle{natbib}

\newpage

\appendix

\begin{center}
    \Large\textbf{Supplementary Material}
\end{center}

\section{GLBISS and GISS: Limit Dynamics of GLBI}

Consider a differential inclusion called \textit{Generalized Linearized Bregman Inverse Scale Space (GLBISS)}, the limit dynamics of GLBI when the step size $\delta \rightarrow 0$. This will help understanding GLBI, and the proof on sign consistency as well as $\ell_2$ consistency of GLBISS can be moved to the case of GLBI with slight modifications.

Specifically, noting by the following Moreau decomposition
\begin{equation}
    \label{eq:var-subst}
    \begin{cases}
        \rho\in \partial \left\| \beta \right\|_1,\\
        z = \rho + \beta / \kappa
    \end{cases}
    \Longleftrightarrow
    \begin{cases}
        \beta = \kappa \mathcal{S}(z,1),\\
        \rho = z - \mathcal{S}(z,1)
    \end{cases}
\end{equation}
GLBI has an equivalent form
\begin{subequations}
    \label{eq:glbi}
    \begin{align}
        \label{eq:glbi-a}
        \alpha_{k+1}/\kappa &= \alpha_k/\kappa - \delta \nabla_{\alpha} \ell\left( \alpha_k, \beta_k \right),\\
        \label{eq:glbi-b}
        \rho_{k+1} + \beta_{k+1}/\kappa &= \rho_k + \beta_k/\kappa - \delta \nabla_{\beta} \ell \left( \alpha_k, \beta_k \right),\\
        \label{eq:glbi-c}
        \rho_k &\in \partial \left\| \beta_k \right\|_1,
    \end{align}
\end{subequations}
where $\rho_0 = \beta_0 = 0$. Taking $\rho(k\delta) = \rho_k,\ \alpha(k\delta) = \alpha_k,\ \beta(k\delta) = \beta_k$, and $\delta \rightarrow 0$, \cref{eq:glbi} can be viewed as a forward Euler discretization of a differential inclusion called \textit{Generalized Linearized Bregman Inverse Scale Space (GLBISS)}
\begin{subequations}
    \label{eq:glbiss}
    \begin{align}
        \label{eq:glbiss-a}
        \dot{\alpha}(t)/\kappa &= - \nabla_{\alpha} \ell\left( \alpha(t), \beta(t) \right),\\
        \label{eq:glbiss-b}
        \dot{\rho}(t) + \dot{\beta}(t)/\kappa &= - \nabla_{\beta} \ell \left( \alpha(t), \beta(t) \right),\\
        \label{eq:glbiss-c}
        \rho(t) &\in \partial \left\| \beta(t) \right\|_1,
    \end{align}
\end{subequations}
where $\rho(0) = \beta(0) = 0$. Next taking $\kappa \rightarrow +\infty$, we reach the following \textit{Generalized Bregman Inverse Scale Space (GISS)}.
\begin{subequations}
    \label{eq:giss}
    \begin{align}
        \label{eq:giss-a}
        0 &= - \nabla_{\alpha} \ell\left( \alpha(t), \beta(t) \right),\\
        \label{eq:giss-b}
        \dot{\rho}(t) &= - \nabla_{\beta} \ell \left( \alpha(t), \beta(t) \right),\\
        \label{eq:giss-c}
        \rho(t) &\in \partial \left\| \beta(t) \right\|_1,
    \end{align}
\end{subequations}
where $\rho(0) = \beta(0) = 0$. Following the same spirit of \citet{osher_sparse_2016}, it is transparent to obtain the existence and uniqueness of the solution paths of GISS and GLBISS under mild conditions; and for GISS and GLBISS, $\ell(\alpha'(t), \beta'(t))$ is non-increasing for $t$, while for GLBI, $\ell(\alpha_k', \beta_k')$ is non-increasing for $k$ if $\kappa \delta \| \bar{H}^o(\theta_k') \|_2 < 2$.

\section{Path Consistency of GLBISS}
\label{sec:glbiss-cstc}

Now we aim to prove the path consistency of GLBISS, which will shed light on proving the path consistency of GLBI in \Cref{sec:proof-glbi-cstc}. Define the following \textit{Oracle Dynamics} of GLBISS, which is viewed as a version of \cref{eq:glbiss} with $S$ known:
\begin{subequations}
    \label{eq:glbiss-orc}
    \begin{align}
        \label{eq:glbiss-orc-a}
        \dot{\alpha}'(t) / \kappa &= - \nabla_{\alpha} \ell\left( \alpha'(t), \beta'(t) \right),\\
        \label{eq:glbiss-orc-b}
        \dot{\rho}_S'(t) + \beta_S'(t) / \kappa &= - \nabla_S \ell \left( \alpha'(t), \beta'(t) \right),\\
        \label{eq:glbiss-orc-c}
        \rho_S'(t) &\in \partial \left\| \beta_S'(t) \right\|_1,
    \end{align}
\end{subequations}
and $\rho_{S^c}'(t) = \beta_{S^c}'(t) \equiv 0_{p-s}$. Transparently, \cref{eq:glbiss-orc} has an equivalent form
\begin{subequations}
    \label{eq:glbiss-orc-show}
    \begin{align}
        \label{eq:glbiss-orc-show-a}
        \dot{\alpha}'(t) &= - \kappa \nabla_{\alpha} \ell\left( \alpha'(t), \beta'(t) \right),\\
        \label{eq:glbiss-orc-show-b}
        \dot{z}_S'(t) &= - \nabla_S \ell \left( \alpha'(t), \beta'(t) \right),\\
        \label{eq:glbiss-orc-show-c}
        \beta_S'(t) &= \kappa \mathcal{S} \left( z_S'(t), 1 \right),
    \end{align}
\end{subequations}
and $z_{S^c}'(t) = \beta_{S^c}'(t) \equiv 0_{p-s}$, according to \cref{eq:var-subst} with $z'(t) = \rho'(t) + \beta'(t) / \kappa$. Let $\theta'(t) := (\alpha'(t), \beta'(t)^T)^T$.

\subsection{Basic Assumptions}

\begin{assumption}[Restricted Strong Convexity (RSC)]
    \label{thm:rsc-glbiss}
    There exist $\lambda, \Lambda > 0$, such that for any $t\ge 0$, and for any $\theta$ on the line segment between $\theta'(t)$ and $\theta^o$, or on the line segment between $\theta^{\star}$ and $\theta^o$,
    \begin{equation}
        \label{eq:rsc-glbiss}
        \lambda I \preceq \nabla_{S_{\alpha}, S_{\alpha}}^2 \ell \left( \theta \right)\preceq \Lambda I.
    \end{equation}
\end{assumption}

\begin{assumption}[Irrepresentable Condition (IRR)]
    \label{thm:irr-glbiss}
    There exist $\eta \in (0,1]$ and $I > 0$ such that
    \begin{subequations}
        \label{eq:irr-glbiss}
        \begin{gather}
            \label{eq:irr-glbiss-a}
            \sup_{T\ge 0} \left\| \int_0^T \overline{\mathrm{irr}}(t) \begin{pmatrix} \dot{\alpha}'(t) / \kappa \\ \dot{z}_S'(t) \end{pmatrix} \mathrm{d}t \right\|_{\infty} < 1 - \frac{\eta}{2},\\
            \label{eq:irr-glbiss-b}
            \sup_{t\ge 0} \left\| \overline{\mathrm{irr}}(t) \right\|_{\infty} \le C,
        \end{gather}
    \end{subequations}
    where
    \begin{equation*}
        \overline{\mathrm{irr}}(t) := \bar{H}_{S^c, S_{\alpha}}\left( \theta'(t) \right) \cdot \bar{H}_{S_{\alpha}, S_{\alpha}}\left( \theta'(t) \right)^{-1}.
    \end{equation*}
\end{assumption}

\subsection{Properties of the Oracle Dynamics of GLBISS}

Here we state our main idea. GLBISS always start within the \emph{oracle subspace} ($\beta_{S^c}(0) = 0$), and we can prove that under IRR (\Cref{thm:irr-glbiss}) the exit time of the oracle subspace is no earlier than some large $\bar{\tau}$ (i.e. the no-false-positive condition holds before $\bar{\tau}$), with high probability. Before $\bar{\tau}$, the iteration follow the identical path of the \textbf{\emph{oracle dynamics}} \Cref{{eq:glbiss-orc}} of GLBI restricted in the oracle subspace. $\ell$ is dropping along the iterative path. Hence to monitor the distance of an estimator to the oracle estimator, define the \textit{potential function} of the oracle dynamics \cref{eq:glbiss-orc} as
\begin{align*}
    \Psi(t) &:= D^{\rho_S'(t)} \left( \beta_S^o, \beta_S'(t) \right) + d(t)^2 / (2\kappa)\\
    &= \left\| \beta_S^o \right\|_1 - \left\langle \beta_S^o, \rho_S'(t) \right\rangle + d(t)^2 / (2\kappa),
\end{align*}
where
\begin{align*}
    d(t) &:= \left\| \theta_{S_\alpha}'(t) - \theta_{S_\alpha}^o \right\|_2\\
    &= \sqrt{\left\| \alpha'(t) - \alpha^o \right\|_2^2 + \left\| \beta_S'(t) - \beta_S^o \right\|_2^2}
\end{align*}
and the Bregman divergence (distance)
\begin{align*}
    & D^{\rho_S'(t)} \left( \beta_S^o, \beta_S'(t) \right)\\
    :={}& \left\| \beta_S^o \right\|_1 - \left\| \beta_S'(t) \right\|_1 - \left\langle \beta_S^o - \beta_S'(t),\ \rho_S'(t) \right\rangle\\
    ={}& \left\| \beta_S^o \right\|_1 - \left\langle \beta_S^o, \rho_S'(t) \right\rangle.
\end{align*}
Equipped with this potential function, our dynamics can be characterized by the following inequality.

\begin{lemma}[Generalized Bihari's inequality]
    \label{thm:glbiss-orc-gbi}
    Under \Cref{thm:rsc-glbiss}, for all $t\ge 0$ we have
    \begin{equation*}
        \frac{\mathrm{d}\Psi(t)}{\mathrm{d}t} \le - \lambda F^{-1}\left( \Psi(t) \right),
    \end{equation*}
    where $\beta_{\min}^o := \min(|\beta_j^o|:\ \beta_j^o \neq 0)$ and
    \begin{align*}
        F(x) &:= \frac{x}{2\kappa} +
        \begin{cases}
            0,& 0\le x < (\beta_{\mathrm{min}}^o)^2,\\
            2x/\beta_{\mathrm{min}}^o,& (\beta_{\mathrm{min}}^o)^2 \le x < s (\beta_{\mathrm{min}}^o)^2,\\
            2\sqrt{sx},& x \ge s(\beta_{\mathrm{min}}^o)^2,
        \end{cases}\\
        F^{-1}(x) &:= \inf(y:\ F(y)\ge x)\ (y\ge 0).
    \end{align*}
\end{lemma}

Such an inequality leads to an exponential decrease of the potential above enforcing the convergence to the oracle estimator, see \Cref{fig:saddle}. Then we can show that as long as the signal is strong enough with all the magnitudes of entries of $\beta^{\star}$ being large enough, the dynamics stopped at $\bar{\tau}$, exactly selects all nonzero entries of $\beta^o$ (\cref{eq:glbiss-orc-cstc-sign} in \Cref{thm:glbiss-orc-cstc}), hence also of $\beta^{\star}$ with high probability, achieving the sign consistency.

Even without the strong signal condition, with RSC we can also show that the dynamics, at $\bar{\tau}$, returns a good estimator of $\theta^o$ (\cref{eq:glbiss-orc-cstc-l2} in \Cref{thm:glbiss-orc-cstc}), hence also of $\theta^{\star}$, having an $\ell_2$ error (often at a minimax optimal rate) with high probability.

\begin{proof}[Proof of \Cref{thm:glbiss-orc-gbi}]
    Note that $\langle \dot{\rho}_S'(t), \beta_S'(t) \rangle = 0$ and $\nabla_{S_{\alpha}} \ell(\alpha^o, \beta^o) = 0$. Thus
    \begin{equation}
        \label{eq:dPsi-glbiss}
        \begin{split}
            & \frac{\mathrm{d}\Psi(t)}{\mathrm{d}t}\\
            ={}& \left\langle - \dot{\rho}_S'(t),\ \beta_S^o \right\rangle + \frac{1}{\kappa} \left\langle \dot{\beta}_S'(t),\ \beta_S'(t) - \beta_S^o \right\rangle\\
            & + \frac{1}{\kappa} \left\langle \dot{\alpha}'(t),\ \alpha'(t) - \alpha^o \right\rangle\\
            ={}& \left\langle \begin{pmatrix} 0\\ \dot{\rho}_S'(t) \end{pmatrix} + \frac{1}{\kappa} \begin{pmatrix} \dot{\alpha}(t)\\ \dot{\beta}_S(t) \end{pmatrix},\ \begin{pmatrix} \alpha'(t)\\ \beta_S'(t) \end{pmatrix} - \begin{pmatrix} \alpha^o\\ \beta_S^o \end{pmatrix} \right\rangle\\
            ={}& - \left\langle \theta_{S_\alpha}'(t) - \theta_{S_\alpha}^o,\ \nabla_{S_{\alpha}} \ell\left( \theta'(t) \right) - \nabla_{S_{\alpha}} \ell\left( \theta^o \right) \right\rangle\\
            ={}& - \left( \theta_{S_\alpha}'(t) - \theta_{S_\alpha}^o \right)^T \bar{H}^o_{S_{\alpha}, S_{\alpha}} \left( \theta'(t) \right) \left( \theta_{S_\alpha}'(t) - \theta_{S_\alpha}^o \right)\\
            \le{}& - \lambda d(t)^2.
        \end{split}
    \end{equation}
    It suffices to show $F(d(t)^2) \ge \Psi(t)$. Since $\| \beta_S^o \|_1 - \langle \beta_S^o, \rho_S'(t) \rangle = 0$ if $\| \beta_S'(t) - \beta_S^o \|_2^2 < (\beta_{\min}^o)^2$, and
    \begin{align*}
        & \left\| \beta_S^o \right\|_1 - \left\langle \beta_S^o, \rho_S'(t) \right\rangle\\
        \le{}& \sum_{j\in N(t) := \left\{ j:\ \mathrm{sign}\left( \beta_j'(t) \right) \neq \mathrm{sign} \left( \beta_j^o \right) \right\}} 2 \left| \beta_j^o \right|\\
        \le{}&
        \begin{cases}
            \displaystyle \frac{2}{\beta_{\min}^o} \sum_{j\in N(t)} (\beta_j^o)^2 \le \frac{2}{\beta_{\min}^o} \left\| \beta_S'(t) - \beta_S^o \right\|_2^2 \\
            \displaystyle 2 \sqrt{s \sum_{j\in N(t)} (\beta_j^o)^2}\le 2\sqrt{s \left\| \beta_S'(t) - \beta_S^o \right\|_2^2}.
        \end{cases}
    \end{align*}
    Combining with the fact that $F(\cdot + x) \ge F(\cdot) + x / (2\kappa)$, we have
    \begin{multline*}
        F\left( d(t)^2 \right) = F\left( \left\| \alpha'(t) - \alpha^o \right\|_2^2 + \left\| \beta_S'(t) - \beta_S^o \right\|_2^2 \right)\\
        \ge F\left( \left\| \beta_S'(t) - \beta_S^o \right\|_2^2 \right) + \frac{1}{2\kappa} \left\| \alpha'(t) - \alpha^o \right\|_2^2 \ge \Psi(t).
    \end{multline*}
\end{proof}

\begin{lemma}
    \label{thm:glbiss-orc-cstc}
    Under \Cref{thm:rsc-glbiss}, let
    \begin{equation*}
        \beta_{\min}^o := \min \left( \left| \beta_j^o \right|:\ \beta_j^o \neq 0 \right).
    \end{equation*}
    Then for any $0 < \mu < 1$ and any
    \begin{equation}
        \label{eq:tau-inf-def-glbiss}
        t \ge \tau_{\infty}(\mu) := \frac{1}{\kappa \lambda} \log \frac{1}{\mu} + \frac{2\log s + 4 + d(0) / \kappa}{\lambda \beta_{\min}^o},
    \end{equation}
    we have
    \begin{equation}
        \label{eq:glbiss-orc-cstc-sign}
        \begin{gathered}
            d(t) \le \mu \beta_{\min}^o\\
            \left( \Longrightarrow \mathrm{sign}\left( \beta_S'(t) \right) = \mathrm{sign}\left( \beta_S^o \right),\ \text{if $\beta_j^o \neq 0$ for $j\in S$} \right).
        \end{gathered}
    \end{equation}
    For any $t$, we have
    \begin{equation}
        \label{eq:glbiss-orc-cstc-l2}
        d(t) \le \min \left( \frac{8 \sqrt{s} + 2 d(0) / \kappa}{\lambda t},\ \sqrt{\frac{\Lambda}{\lambda}} \cdot d(0) \right).
    \end{equation}
\end{lemma}

\begin{proof}[Proof of \Cref{thm:glbiss-orc-cstc}]
    \cref{eq:dPsi-glbiss} tells that $\Psi(t)$ is non-increasing. If $\theta'(t) = \theta^o$ for some $t \le \tau_{\infty}(\mu)$, since $\ell(\theta'(t))$ is non-increasing, $\ell(\theta'(t)) \le \ell(\theta^o)$ holds for any $t\ge \tau_\infty(\mu)$. By the definition of $\theta^o$, we have $\theta'(t) = \theta^o$ which implies $d(t) = 0 \le \mu \beta_{\min}^o$, i.e. \cref{eq:glbiss-orc-cstc-sign} holds. If $\theta'(t) \neq \theta^o$ for any $t\le \tau_{\infty}(\mu)$, by \cref{eq:dPsi-glbiss} we have that $\Psi(t)$ is \textit{strictly} decreasing on $[0, \tau_{\infty}(\mu)]$. Besides, $F$ is strictly increasing and continuous on $[(\beta_{\min}^o)^2, +\infty)$. Moreover,
    \begin{align*}
        F\left( d(0)^2 \right) &\ge F\left( \left\| \beta_S^o \right\|_2^2 \right) + \left\| \alpha^o \right\|_2^2 / (2\kappa) \ge \Psi(0),\\
        d(0)^2 &\ge \left\| \beta_S^o \right\|_2^2 \ge s\left( \beta_{\min}^o \right)^2,
    \end{align*}
    If there does not exist some $t\le \tau_{\infty}(\mu)$ satisfying \cref{eq:glbiss-orc-cstc-sign}, then for $0\le t \le \tau_{\infty}(\mu)$,
    \begin{align*}
        &\Psi\left( t \right)\\
        &\begin{cases}
            \ge d\left( t \right)^2 / (2\kappa) \ge \mu^2 \left( \beta_{\min}^o \right)^2 / (2\kappa) >0,& \text{if $\kappa < +\infty$},\\
            > 0,& \text{if $\kappa = +\infty$},
        \end{cases}
    \end{align*}
    which also implies that $F^{-1}(\Psi(t)) > 0$. By \Cref{thm:glbiss-orc-gbi},
    \begin{align*}
        & \lambda \tau_{\infty}(\mu) \le \int_0^{\tau_{\infty}(\mu)} \frac{- \frac{\mathrm{d}}{\mathrm{d}t} \Psi(t)}{F^{-1}\left( \Psi(t) \right)} \mathrm{d}t = \int_{\Psi\left( \tau_{\infty}(\mu) \right)}^{\Psi(0)} \frac{\mathrm{d}x}{F^{-1}(x)}\\
        \le {}& \left( \int_{\mu^2 \left( \beta_{\min}^o \right)^2 / (2\kappa)}^{\left( \beta_{\min}^o \right)^2 / (2\kappa)} + \int_{\left( \beta_{\min}^o \right)^2 / (2\kappa)}^{F\left( \left( \beta_{\min}^o \right)^2 \right)} \right.\\
        & \left. + \int_{F\left( \left( \beta_{\min}^o \right)^2 \right)}^{F\left( s\left( \beta_{\min}^o \right)^2 \right)} + \int_{F\left( s\left( \beta_{\min}^o \right)^2 \right)}^{F\left( d(0)^2 \right)} \right) \frac{\mathrm{d}x}{F^{-1}(x)}\\
        \le {}& \int_{\mu^2 \left( \beta_{\min}^o \right)^2 / (2\kappa)}^{\left( \beta_{\min}^o \right)^2 / (2\kappa)} \frac{\mathrm{d}x}{2\kappa x} + \int_{\left( \beta_{\min}^o \right)^2 / (2\kappa)}^{F\left( \left( \beta_{\min}^o \right)^2 \right)} \frac{1}{\left( \beta_{\min}^o \right)^2} \mathrm{d}x\\
        & + \int_{\left( \beta_{\min}^o \right)^2}^{s\left( \beta_{\min}^o \right)^2} \frac{\mathrm{d}F(x)}{x} + \int_{s\left( \beta_{\min}^o \right)^2}^{d(0)^2} \frac{\mathrm{d}F(x)}{x}\\
        = {}& \frac{1}{2\kappa} \log \frac{1}{\mu^2} + \frac{2}{\beta_{\min}^o} + \int_{\left( \beta_{\min}^o \right)^2}^{s\left( \beta_{\min}^o \right)^2} \left( \frac{1}{2\kappa x} + \frac{2}{\beta_{\min}^o x} \right) \mathrm{d}x\\
        & + \int_{s\left( \beta_{\min}^o \right)^2}^{d(0)^2} \left( \frac{1}{2\kappa x} + \frac{\sqrt{s}}{x\sqrt{x}} \right) \mathrm{d}x\\
        < {}& \frac{1}{2\kappa} \log \frac{1}{\mu^2} + \frac{2}{\beta_{\min}^o} + \frac{1}{2\kappa} \log \frac{d(0)^2}{\left( \beta_{\min}^o \right)^2} + \frac{2\log s}{\beta_{\min}^o} + \frac{2}{\beta_{\min}^o}\\
        \le{}& \frac{1}{\kappa} \log \frac{1}{\mu} + \frac{2\log s + 4 + d(0) / \kappa}{\beta_{\min}^o},
    \end{align*}
    contradicting with the definition of $\tau_{\infty}(\mu)$. Thus \cref{eq:glbiss-orc-cstc-sign} holds for some $0\le \tau\le \tau_{\infty}(\mu)$. If $\kappa = +\infty$, we see that for $t\ge \tau_{\infty}(\mu)$, $\Psi(t) \le \Psi(\tau) = 0$. Then $\theta'(t) = \theta^o$, and \cref{eq:glbiss-orc-cstc-sign} holds. If $\kappa < +\infty$, just note that for $t\ge \tau$,
    \begin{multline*}
        d(t)^2/(2\kappa) \le \Psi(t)\le \Psi(\tau) = d(\tau)^2/(2\kappa)\\
        \Longrightarrow d(t)\le d(\tau) \le \mu \beta_{\min}^o.
    \end{multline*}
    So \cref{eq:glbiss-orc-cstc-sign} holds for $t\ge \tau_{\infty}(\mu)$.

    For any $t>0$, define $C := (4\sqrt{s} + d(0) / \kappa) / (\lambda t)$. Now assume that for any $0\le t'\le t$,
    \begin{equation*}
        \frac{\mathrm{d}}{\mathrm{d}t'}\Psi\left( t' \right) \le - \lambda C^2.
    \end{equation*}
    Note that for $\tilde{F}(x) = x / (2\kappa) + 2\sqrt{sx}\ge F(x)$, by \Cref{thm:glbiss-orc-gbi} we have
    \begin{equation*}
        \frac{\mathrm{d}}{\mathrm{d}t'}\Psi\left( t' \right) \le - \lambda F^{-1}\left( \Psi\left( t' \right) \right)\le - \lambda \tilde{F}^{-1} \left( \Psi\left( t' \right) \right).
    \end{equation*}
    By \cref{eq:dPsi-glbiss} and the fact that
    \begin{align*}
        \tilde{F}\left( d(0)^2 \right) &\ge \tilde{F}\left( \left\| \beta_S^o \right\|_2^2 \right) + \left\| \alpha^o \right\|_2^2 / (2\kappa) \ge \Psi(0),
    \end{align*}
    we have that, if $d(0) > C$, then
    \begin{align*}
        \lambda t &\le \int_0^t \frac{- \frac{\mathrm{d}}{\mathrm{d}t'} \Psi\left( t' \right)}{\max\left( C^2, \tilde{F}^{-1}\left( \Psi\left( t' \right) \right) \right)} \mathrm{d}t'\\
        & = \int_{\Psi(t)}^{\Psi(0)} \frac{\mathrm{d}x}{\max\left( C^2, \tilde{F}^{-1}(x) \right)}\\
        &\le \int_{\tilde{F}(0)}^{\tilde{F}\left( d(0)^2 \right)} \frac{\mathrm{d}x}{\max\left( C^2, \tilde{F}^{-1}(x) \right)}\\
        & = \int_{\tilde{F}(0)}^{\tilde{F}\left( C^2 \right)} \frac{\mathrm{d}x}{C^2} + \int_{C^2}^{d(0)^2} \frac{\mathrm{d}\tilde{F}(x)}{x}\\
        & = \frac{C^2/(2\kappa)+2\sqrt{s}C}{C^2} + \int_{C^2}^{d(0)^2} \left( \frac{1}{2\kappa x} + \frac{\sqrt{s}}{x\sqrt{x}} \right) \mathrm{d}x\\
        & < \frac{4\sqrt{s}}{C} + \frac{1}{2\kappa} \left( 1 + \log \frac{d(0)^2}{C^2} \right)\le \frac{4\sqrt{s} + d(0) / \kappa}{C},
    \end{align*}
    a contradiction with the definition of $C$. If $d(0) \le C$, then similarly
    \begin{multline*}
        \lambda t \le \int_{\tilde{F}(0)}^{\tilde{F}\left( d(0)^2 \right)} \frac{\mathrm{d}x}{\max\left( C^2, \tilde{F}^{-1}(x) \right)} \le \int_{\tilde{F}(0)}^{\tilde{F}\left( d(0)^2 \right)} \frac{\mathrm{d}x}{C^2}\\
        = \frac{d(0)^2 / (2\kappa) + 2\sqrt{s}\cdot d(0)}{C^2} < \frac{4\sqrt{s} + d(0) / \kappa}{C}.
    \end{multline*}
    Also a contradiction. Thus there exists some $t'\le t$ such that
    \begin{equation}
        \label{eq:lambda-C-2-ge}
        \begin{split}
            \lambda C^2 &\ge - \frac{\mathrm{d}\Psi\left( t' \right)}{\mathrm{d}t'}\\
            & = \left( \theta_{S_\alpha}' \left( t' \right) - \theta_{S_\alpha}^o \right)^T \bar{H}^o_{S_{\alpha}, S_{\alpha}} \left( \theta'\left( t' \right) - \theta_{S_\alpha}'\left( t' \right) \right).
        \end{split}
    \end{equation}
    By the Taylor expansion of $\ell(\theta'(t))$ at $\theta^o$, with the fact that $\beta_{S^c}'(t) = \beta_{S^c}^o = 0,\ \nabla_{S_{\alpha}} \ell(\theta^o) = 0$, we have
    \begin{align*}
        & \ell\left( \theta'(t) \right) - \ell\left( \theta^o \right)\\
        ={}& \left( \theta_{S_\alpha}'(t) - \theta_{S_\alpha}^o \right)^T H_{S_{\alpha}, S_{\alpha}}^{\ell}\left( \theta'(t) \right) \left( \theta_{S_\alpha}'(t) - \theta_{S_\alpha}^o \right)\\
        \ge{}& \frac{\lambda}{2} d(t)^2,
    \end{align*}
    where
    \begin{equation*}
        H^{\ell}(\theta) := \int_0^1 (1-\mu) \nabla^2 \ell \left( \theta^o + \mu \left( \theta - \theta^o \right) \right) \mathrm{d} \mu.
    \end{equation*}
    By \cref{eq:lambda-C-2-ge} and the fact that $\ell(\theta'(\cdot))$ is non-increasing, it is easy to derive
    \begin{equation*}
        \lambda C^2 \ge \ell\left( \theta'\left( t' \right) \right) - \ell\left( \theta^o \right) \ge \ell\left( \theta'(t) \right) - \ell\left( \theta^o \right) \ge \frac{\lambda}{2} d(t)^2.
    \end{equation*}
    Besides, by \Cref{thm:rsc-glbiss} and Taylor expansion, it is easy to derive
    \begin{equation*}
        \frac{\Lambda}{2} d(0)^2 \ge \ell\left( \theta'(0) \right) - \ell\left( \theta^o \right) \ge \ell\left( \theta'(t) \right) - \ell\left( \theta^o \right) \ge \frac{\lambda}{2} d(t)^2.
    \end{equation*}
    Thus \cref{eq:glbiss-orc-cstc-l2} holds.
\end{proof}

\subsection{Main Result on GLBISS}

\begin{theorem}
    \label{thm:glbiss-cstc}
    Under \Cref{thm:rsc-glbiss} and \ref{thm:irr-glbiss}, suppose $\kappa\ge 2 d(0)$, and
    \begin{equation*}
        \bar{\tau} = \frac{\eta}{2(C+1)} \cdot \frac{1}{\left\| \nabla \ell\left( \alpha^{\star}, \beta^{\star} \right) \right\|_{\infty}}.
    \end{equation*}
    Then we have all the following properties.

    \textit{No-false-positive}: For all $0\le t\le \bar{\tau}$, the solution path of GLBISS has no false-positive, i.e. $\beta_{S^c}(t) = 0$.

    \textit{Sign consistency}: If
    \begin{equation*}
        \beta_{\min}^{\star} := \min_{j\in S} \left| \beta_j^{\star} \right| \ge \max \left( 2 \left\| \beta_S^o - \beta_S^{\star} \right\|_{\infty},\ \frac{4\log s + 9}{\lambda \bar{\tau}} \right),
    \end{equation*}
    then $\beta(t)$ has the sign consistency at $\bar{\tau}$, i.e. $\mathrm{sign}(\beta(\bar{\tau})) = \mathrm{sign}(\beta^{\star})$.

    \textit{$\ell_2$ consistency}:
    \begin{equation*}
        \left\| \begin{pmatrix} \alpha\left( \bar{\tau} \right) - \alpha^{\star}\\ \beta\left( \bar{\tau} \right) - \beta^{\star} \end{pmatrix} \right\|_2 \le \frac{10\sqrt{s}}{\lambda \bar{\tau}}.
    \end{equation*}
\end{theorem}

To prove such a theorem, we need a lemma stated below.

\begin{lemma}[No-false-positive condition for GLBISS]
    \label{thm:glbiss-nfp}
    For the Oracle Dynamics \cref{eq:glbiss-orc}, if there is $\tau > 0$, such that for $0\le T\le \tau$ the inequality
    \begin{multline}
        \label{eq:glbiss-nfp}
        \left\| \int_0^T \overline{\mathrm{irr}}(t) \left( \begin{pmatrix} 0 \\ \dot{\rho}_S'(t) \end{pmatrix} + \frac{1}{\kappa} \begin{pmatrix} \dot{\alpha}'(t) \\ \dot{\beta}_S'(t) \end{pmatrix} \right) \mathrm{d}t \right.\\
        {} + \left. \left( \int_0^T \overline{\mathrm{irr}}(t) \mathrm{d}t \right) \cdot \nabla_{S_{\alpha}} \ell \left( \theta^{\star} \right) - T \cdot \nabla_{S^c} \ell\left( \theta^{\star} \right) \right\|_{\infty}\\
        < 1
    \end{multline}
    holds, then for $0\le T\le \tau$ the solution path of \cref{eq:glbiss} has no false-positive, i.e. $\beta_{S^c}(T) = 0$.
\end{lemma}

\begin{proof}[Proof of \Cref{thm:glbiss-nfp}]
    Let
    \begin{equation*}
        \bar{\tau} = \inf\left( t\ge 0:\ \left\| \rho_{S^c}(t) \right\|_{\infty} = 1 \right).
    \end{equation*}
    It suffices to show $\bar{\tau} > \tau$. For $0\le t\le \bar{\tau}$, we have $\beta_{S^c}(t) = 0$, which also implies $\rho_S(t) = \rho_S'(t)$ and $\beta_S(t) = \beta_S'(t)$. Hence
    \begin{align*}
        & \begin{pmatrix} 0\\ \dot{\rho}_S'(t) \end{pmatrix} + \frac{1}{\kappa} \begin{pmatrix} \dot{\alpha}'(t)\\ \dot{\beta}_S'(t) \end{pmatrix} = - \nabla_{S_{\alpha}} \left( \theta'(t) \right) \\
        ={}& - \bar{H}_{S_{\alpha}, S_{\alpha}}(t) \left( \theta_{S_\alpha}'(t) - \theta_{S_\alpha}^\star \right) - \nabla_{S_{\alpha}} \ell \left( \theta^{\star} \right)\\
        & \dot{\rho}_{S^c}(t) = - \nabla_{S^c} \ell\left( \theta^{\star} \right)\\
        ={}& - \bar{H}_{S^c, S_{\alpha}}(t) \left( \theta_{S_\alpha}'(t) - \theta_{S_\alpha}^\star \right) - \nabla_{S^c} \ell\left( \theta^{\star} \right).
    \end{align*}
    Combining these two equations we obtain
    \begin{multline*}
        \dot{\rho}_{S^c}(t) = \overline{\mathrm{irr}}(t) \left( \begin{pmatrix} 0\\ \dot{\rho}_S'(t) \end{pmatrix} + \frac{1}{\kappa} \begin{pmatrix} \dot{\alpha}'(t)\\ \dot{\beta}_S'(t) \end{pmatrix} \right)\\
        + \overline{\mathrm{irr}}(t) \cdot \nabla_{S_{\alpha}} \ell\left( \theta^{\star} \right) - \nabla_{S^c} \ell\left( \theta^{\star} \right).
    \end{multline*}
    Integration on both sides leads to
    \begin{multline*}
        \rho_{S^c}(T) = \int_0^T \overline{\mathrm{irr}}(t) \left( \begin{pmatrix} 0 \\ \dot{\rho}_S'(t) \end{pmatrix} + \frac{1}{\kappa} \begin{pmatrix} \dot{\alpha}'(t) \\ \dot{\beta}_S'(t) \end{pmatrix} \right) \mathrm{d}t \\
        {} + \left( \int_0^T \overline{\mathrm{irr}}(t) \mathrm{d}t \right) \cdot \nabla_{S_{\alpha}} \ell \left( \theta^{\star} \right) - T \cdot \nabla_{S^c} \ell\left( \theta^{\star} \right),
    \end{multline*}
    for $0\le T < \bar{\tau}$. Due to the continuity of $\rho_{S^c}(t), \rho_S'(t)$ (and $\beta_S'(t)$, if $\kappa < +\infty$), the equation above also holds for $T = \bar{\tau}$. According to the definition of $\bar{\tau}$, we know \cref{eq:glbiss-nfp} does not hold for $T = \bar{\tau}$. Thus $\bar{\tau} > \tau$, and the desired result follows.
\end{proof}

Now we are ready to prove the main result on GLBISS.

\begin{proof}[Proof of \Cref{thm:glbiss-cstc}]
    By \Cref{thm:irr-glbiss}, we have that
    \begin{equation*}
        \left\| \int_0^{\bar{\tau}} \overline{\mathrm{irr}}(t) \left( \begin{pmatrix} 0 \\ \dot{\rho}_S'(t) \end{pmatrix} + \frac{1}{\kappa} \begin{pmatrix} \dot{\alpha}'(t) \\ \dot{\beta}_S'(t) \end{pmatrix} \right) \mathrm{d}t \right\|_{\infty} < 1 - \frac{\eta}{2},
    \end{equation*}
    and
    \begin{multline*}
        \left\| \left( \int_0^{\bar{\tau}} \overline{\mathrm{irr}}(t) \mathrm{d}t \right) \cdot \nabla_{S_{\alpha}} \ell \left( \theta^{\star} \right) - \bar{\tau} \cdot \nabla_{S^c} \ell\left( \theta^{\star} \right) \right\|_{\infty}\\
        \le \bar{\tau} \left( C \left\| \nabla_{S_{\alpha}} \ell\left( \theta^{\star} \right) \right\|_{\infty} + \left\| \nabla_{S^c} \ell\left( \theta^{\star} \right) \right\|_{\infty} \right)\\
        \le \bar{\tau} \left( C + 1 \right) \left\| \nabla \ell\left( \theta^{\star} \right) \right\|_{\infty} \le \frac{\eta}{2}.
    \end{multline*}
    Thus by \Cref{thm:glbiss-nfp}, the original dynamics \cref{eq:glbiss} has no false-positive for all $0\le t\le \overline{\tau}$.

    Then we prove the sign consistency. We have $\mathrm{sign}(\beta_S^o) = \mathrm{sign}(\beta_S^{\star})$ and $\beta_{\min}^o \ge \beta_{\min}^{\star} / 2$. According to \Cref{thm:glbiss-orc-cstc}, for any $0 < \mu < 1$ and any
    \begin{equation*}
        t \ge \frac{1}{\kappa \lambda} \log \frac{1}{\lambda} + \frac{4\log s + 9}{\lambda \beta_{\min}^{\star}} \left( \ge \tau_{\infty}(\mu) \right),
    \end{equation*}
    there holds $\mathrm{sign}(\beta_S'(t)) = \mathrm{sign}(\beta_S^o)$. By the right continuity of $\beta_S'(t)$, with $\mu \rightarrow 1$, this equation is guaranteed for $t\ge (4\log s + 9) / (\lambda \beta_{\min}^{\star})$ (and hence for $\bar{\tau}$). Then
    \begin{gather*}
        \mathrm{sign} \left( \beta_S\left( \bar{\tau} \right) \right) = \mathrm{sign}\left( \beta_S'\left( \bar{\tau} \right) \right) = \mathrm{sign} \left( \beta_S^o \right) = \mathrm{sign} \left( \beta_S^{\star} \right),\\
        \mathrm{sign}\left( \beta_{S^c} \left( \bar{\tau} \right) \right) = 0 = \mathrm{sign} \left( \beta_{S^c}^{\star} \right).
    \end{gather*}

    Last we prove the $\ell_2$ consistency. By \Cref{thm:glbiss-orc-cstc},
    \begin{equation*}
        \left\| \begin{pmatrix} \alpha\left( \bar{\tau} \right) - \alpha^o\\ \beta\left( \bar{\tau} \right) - \beta^o \end{pmatrix} \right\|_2 = \left\| \theta_{S_\alpha}'\left( \bar{\tau} \right) - \theta_{S_\alpha}^o \right\|_2 = d\left( \bar{\tau} \right) \le \frac{9\sqrt{s}}{\lambda \bar{\tau}}.
    \end{equation*}
    Besides, noting
    \begin{align*}
        \lambda \left\| \theta_{S_\alpha}^o - \theta_{S_\alpha}^\star \right\|_2^2 &\le \left\langle \theta_{S_\alpha}^o - \theta_{S_\alpha}^\star,\ \nabla_{S_{\alpha}} \ell \left( \theta^o \right) - \nabla_{S_{\alpha}} \ell\left( \theta^{\star} \right) \right\rangle\\
        &\le \left\| \theta_{S_\alpha}^o - \theta_{S_\alpha}^{\star} \right\|_2 \cdot \left\| \nabla_{S_{\alpha}} \ell\left( \theta^{\star} \right) \right\|_2,
    \end{align*}
    we have
    \begin{multline*}
        \left\| \theta_{S_\alpha}^o - \theta_{S_\alpha}^\star \right\|_2 \le \frac{1}{\lambda} \left\| \nabla_{S_{\alpha}} \ell\left( \theta^{\star} \right) \right\|_2\\
        \le \frac{\sqrt{s+1}}{\lambda} \left\| \nabla \ell\left( \theta^{\star} \right) \right\|_{\infty} \le \frac{\eta\sqrt{s+1}}{2\lambda(C+1)\bar{\tau}} \le \frac{\sqrt{s}}{\lambda \bar{\tau}}.
    \end{multline*}
\end{proof}

\section{Path Consistency of GLBI}
\label{sec:proof-glbi-cstc}

We define a \textit{potential function}
\begin{align*}
    \Psi_k &:= D^{\rho_{k,S}'} \left( \beta_S^o, \beta_{k,S}' \right) + d_k^2 / (2\kappa)\\
    &= \left\| \beta_S^o \right\|_1 - \left\langle \beta_S^o, \rho_{k,S}' \right\rangle + d_k^2 / (2\kappa),
\end{align*}
where
\begin{align}
    \label{eq:d-def-glbi}
    d_k &:= \left\| \theta_{k,S_{\alpha}}' - \theta_{k,S_{\alpha}}^o \right\|_2 \\
    & = \sqrt{\left\| \alpha_k' - \alpha^o \right\|_2^2 + \left\| \beta_{k,S}' - \beta_S^o \right\|_2^2},
\end{align}

\begin{lemma}[Discrete Generalized Bihari's inequality]
    \label{thm:glbi-orc-gbi}
    Under RSC (\Cref{thm:rsc-glbi}), suppose $\delta$ is small such that
    \begin{equation*}
        \lambda' := \lambda \left( 1 - \kappa \delta \Lambda / 2 \right) > 0.
    \end{equation*}
    For all $k\ge 0$ we have
    \begin{equation*}
        \Psi_{k+1} - \Psi_k \le - \delta \cdot \lambda' F^{-1}\left( \Psi_k \right),
    \end{equation*}
    where $\beta_{\min}^o := \min(|\beta_j^o|:\ \beta_j^o \neq 0)$ and
    \begin{align*}
        F(x) &:= \frac{x}{2\kappa} +
        \begin{cases}
            0,& 0\le x < (\beta_{\mathrm{min}}^o)^2,\\
            2x/\beta_{\mathrm{min}}^o,& (\beta_{\mathrm{min}}^o)^2 \le x < s (\beta_{\mathrm{min}}^o)^2,\\
            2\sqrt{sx},& x \ge s(\beta_{\mathrm{min}}^o)^2,
        \end{cases}\\
        F^{-1}(x) &:= \inf(y:\ F(y)\ge x)\ (y\ge 0).
    \end{align*}
\end{lemma}

\begin{proof}[Proof of \Cref{thm:glbi-orc-gbi}]
    Similar to the continuous case, we have $F(d_k^2) \ge \Psi_k$, so it suffices to show
    \begin{equation*}
        \Psi_{k+1} - \Psi_k \le - \delta \lambda' d_k^2.
    \end{equation*}
    Note that $\beta_{k+1,S}'(\rho_{k+1,S}' - \rho_{k,S}') \ge 0$. We have
    \begin{align*}
        & \Psi_{k+1} - \Psi_k\\
        ={} & - \left\langle \begin{pmatrix} \alpha^o \\ \beta_S^o \end{pmatrix},\ \begin{pmatrix} 0\\ \rho_{k+1,S}' - \rho_{k,S}' \end{pmatrix} \right\rangle\\
        & + \frac{1}{2\kappa} \left( \left\| \begin{pmatrix} \alpha_{k+1}' - \alpha^o\\ \beta_{k+1,S}' - \beta_S^o \end{pmatrix} \right\|_2^2 - \left\| \begin{pmatrix} \alpha_k' - \alpha^o\\ \beta_{k,S}' - \beta_S^o \end{pmatrix} \right\|_2^2 \right)\\
        ={} & \left( \theta_{k,S_\alpha}' - \theta_{S_\alpha}^o \right)^T \begin{pmatrix} 0\\ \rho_{k+1,S}' - \rho_{k,S}' \end{pmatrix}\\
        & - \theta_{k,S_\alpha}'^T \begin{pmatrix} 0\\ \rho_{k+1,S}' - \rho_{k,S}' \end{pmatrix} + \frac{1}{2\kappa} \left\| \theta_{k+1,S_\alpha}' - \theta_{k,S_\alpha}' \right\|_2^2\\
        & + \frac{1}{\kappa} \left( \theta_{k+1,S_\alpha}' - \theta_{k,S_\alpha}' \right)^T \left( \theta_{k,S_\alpha}' - \theta_{S_\alpha}^o \right)\\
        ={} & \left( \theta_{k,S_\alpha}' - \theta_{S_\alpha}^o \right)^T \cdot \left( \begin{pmatrix} 0\\ \rho_{k+1,S}' - \rho_{k,S}' \end{pmatrix} \right.\\
        & \left. \quad + \frac{1}{\kappa} \left( \theta_{k+1,S_\alpha}' - \theta_{k,S_\alpha}' \right) \right)\\
        & + \frac{1}{2\kappa} \left\| \theta_{k+1,S_\alpha}' - \theta_{k,S_\alpha}' \right\|_2^2 - \theta_{k,S_\alpha}'^T \begin{pmatrix} 0\\ \rho_{k+1,S}' - \rho_{k,S}' \end{pmatrix}\\
        \le{} & - \delta \left( \theta_{k,S_\alpha}' - \theta_{S_\alpha}^o \right)^T \left( \nabla_{S_{\alpha}} \ell\left( \theta_k' \right) - \nabla_{S_{\alpha}} \ell\left( \theta^o \right) \right)\\
        & + \frac{1}{2\kappa} \left\| \theta_{k+1,S_\alpha}' - \theta_{k,S_\alpha}' \right\|_2^2\\
        & + \left( \theta_{k+1,S_\alpha}' - \theta_{k,S_\alpha}' \right)^T \begin{pmatrix} 0\\ \rho_{k+1,S}' - \rho_{k,S}' \end{pmatrix}\\
        \le{} & - \delta \left( \theta_{k,S_\alpha}' - \theta_{S_\alpha}^o \right)^T \bar{H}_{S_{\alpha}, S_{\alpha}}^o \left( \theta_k' \right) \left( \theta_{k,S_\alpha}' - \theta_{S_\alpha}^o \right)\\
        & + \frac{\kappa}{2} \left\| \begin{pmatrix} 0\\ \rho_{k+1,S}' - \rho_{k,S}' \end{pmatrix} + \frac{1}{\kappa} \left( \theta_{k+1,S_\alpha}' - \theta_{k,S_\alpha}' \right) \right\|_2^2\\
        ={} & - \delta \left( \theta_{k,S_\alpha}' - \theta_{S_\alpha}^o \right)^T \bar{H}_{S_{\alpha}, S_{\alpha}}^o \left( \theta_k' \right) \left( \theta_{k,S_\alpha}' - \theta_{S_\alpha}^o \right) \\
        & + \frac{\kappa \delta^2}{2} \left\| \nabla_{S_{\alpha}} \ell\left( \theta_k' \right) - \nabla_{S_{\alpha}} \ell\left( \theta^o \right) \right\|_2^2\\
        ={} & - \delta \left( \theta_{k,S_\alpha}' - \theta_{S_\alpha}^o \right)^T\\
        & \quad \cdot \left( \bar{H}_{S_{\alpha}, S_{\alpha}}^o \left( \theta_k' \right) - \frac{\kappa \delta}{2} \bar{H}_{S_{\alpha}, S_{\alpha}}^o \left( \theta_k' \right)^2 \right) \left( \theta_{k,S_\alpha}' - \theta_{S_\alpha}^o \right)\\
        & \le - \delta \lambda' d_k^2.
    \end{align*}
\end{proof}

\begin{lemma}[Consistency of the Oracle Iteration of GLBI]
    \label{thm:glbi-orc-cstc}
    Under \Cref{thm:rsc-glbi}, suppose $\delta$ is small such that
    \begin{equation}
        \label{eq:lambda-prime}
        \lambda' := \lambda \left( 1 - \kappa \delta \Lambda / 2 \right) > 0.
    \end{equation}
    Let
    \begin{equation*}
        \beta_{\min}^o := \min \left( \left| \beta_j^o \right|:\ \beta_j^o \neq 0 \right)
    \end{equation*}
    and $d_k$ defined as \cref{eq:d-def-glbi}. Then for any $0 < \mu < 1$ and any $k$ such that
    \begin{gather}
        \label{eq:tau-inf-def-glbi}
        k \delta \ge \tau_{\infty}'(\mu) := \frac{1}{\kappa \lambda'} \log \frac{1}{\mu} + \frac{2\log s + 4 + d_0 / \kappa}{\lambda' \beta_{\min}^o} + 4\delta,
    \end{gather}
    we have
    \begin{equation}
        \label{eq:glbi-orc-cstc-sign}
        \begin{gathered}
            d_k \le \mu \beta_{\min}^o\\
            \left( \Longrightarrow \mathrm{sign}\left( \beta_{k,S}' \right) = \mathrm{sign}\left( \beta_S^o \right),\ \text{if $\beta_j^o \neq 0$ for $j\in S$} \right).
        \end{gathered}
    \end{equation}
    For any $k$, we have
    \begin{equation}
        \label{eq:glbi-orc-cstc-l2}
        d_k \le \min \left( \frac{8 \sqrt{s} + 2 d_0 / \kappa}{\lambda' k \delta},\ \sqrt{\frac{\Lambda}{\lambda}} \cdot d_0 \right).
    \end{equation}
\end{lemma}

\begin{proof}[Proof of \Cref{thm:glbi-orc-cstc}]
    The proof is almost a discrete version of the continuous case. The only non-trivial thing is described as follows. First, suppose there does not exist $k \le \tau_{\infty}'(\mu) / \delta$ satisfying \cref{eq:glbi-orc-cstc-sign}, then for any $k \le \tau_{\infty}'(\mu)$, we have $\Psi_k > \mu^2 (\beta_{\min}^o)^2 / (2\kappa)$. Letting $k_0 = 0$, then $\Psi_{k_0} = \Psi_0 \le F(d_0^2)$. Suppose that
    \begin{multline*}
        F\left( d_0^2 \right) \ge \Psi_{k_0}, \ldots, \Psi_{k_1-1} > F\left( s \left( \beta_{\min}^o \right)^2 \right)\\
        \ge \Psi_{k_1}, \ldots, \Psi_{k_2-1} > F\left( \left( \beta_{\min}^o \right)^2 \right)\\
        \ge \Psi_{k_2}, \ldots, \Psi_{k_3-1} > \left( \beta_{\min}^o \right)^2 / (2\kappa)\\
        \ge \Psi_{k_3}, \ldots, \Psi_{k_4-1} > \mu^2 \left( \beta_{\min}^o \right)^2 / (2\kappa)\ge \Psi_{k_4}, \ldots
    \end{multline*}
    Then $k_4 \delta > \tau_{\infty}'(\mu)$. Besides, by \Cref{thm:glbi-orc-gbi},
    \begin{equation*}
        \delta \le \frac{\Psi_k - \Psi_{k+1}}{\lambda' F^{-1}(\Psi_k)}\ (0\le k\delta\le \tau_{\infty}'(\mu)).
    \end{equation*}
    Thus $\lambda' (k_4-4) \delta$ is not greater than
    \begin{align*}
        & \left( \sum_{k=k_3}^{k_4-2} + \sum_{k=k_2}^{k_3-2} + \sum_{k=k_1}^{k_2-2} + \sum_{k=k_0}^{k_1-2} \right) \frac{\Psi_k - \Psi_{k+1}}{F^{-1}(\Psi_k)}\\
        \le{}& \sum_{k=k_3}^{k_4-2} \frac{\Psi_k - \Psi_{k+1}}{2\kappa \Psi_k} + \sum_{k=k_2}^{k_3-2} \frac{\Psi_k - \Psi_{k+1}}{\left( \beta_{\min}^o \right)^2}\\
        & + \sum_{k=k_1}^{k_2-2} \frac{F(\Delta_k) - F(\Delta_{k+1})}{\Delta_k} + \sum_{k=k_0}^{k_1-2}\frac{F(\Delta_k) - F(\Delta_{k+1})}{\Delta_k}\\
        & \left( \Delta_k := F^{-1}(\Psi_k) \right)\\
        ={} & \sum_{k=k_3}^{k_4-2} \frac{\Psi_k - \Psi_{k+1}}{2\kappa \Psi_k} + \sum_{k=k_2}^{k_3-2} \frac{\Psi_k - \Psi_{k+1}}{\left( \beta_{\min}^o \right)^2}\\
        & + \sum_{k=k_1}^{k_2-2} \left( \frac{\Delta_k - \Delta_{k+1}}{2\kappa \Delta_k} + \frac{2(\Delta_k - \Delta_{k+1})}{\beta_{\min}^o \Delta_k} \right)\\
        & + \sum_{k=k_0}^{k_1-2} \left( \frac{\Delta_k - \Delta_{k+1}}{2\kappa \Delta_k} + \frac{2\sqrt{s}\left( \sqrt{\Delta_k} - \sqrt{\Delta_{k+1}} \right)}{\Delta_k} \right).
    \end{align*}
    By $(u-v)/u \le \log (u/v)$ and $(\sqrt{u} - \sqrt{v})/u \le 1/\sqrt{v} - 1/\sqrt{u}$ for $u \ge v > 0$, the quantity above is not greater than
    \begin{align*}
        & \frac{\log \left( \Psi_{k_3} / \Psi_{k_4-1} \right)}{2\kappa} + \frac{\Psi_{k_2} - \Psi_{k_3-1}}{\left( \beta_{\min}^o \right)^2} + \frac{\log \left( \Delta_{k_0} / \Delta_{k_2-1} \right)}{2\kappa}\\
        & + \frac{2 \log \left( \Delta_{k_1} / \Delta_{k_2-1} \right)}{\beta_{\min}^o} + 2\sqrt{s} \left( \frac{1}{\sqrt{\Delta_{k_1-1}}} - \frac{1}{\sqrt{\Delta_{k_0}}} \right)\\
        <{} & \frac{\log \left( 1 / \mu^2 \right)}{2\kappa} + \frac{2 \beta_{\min}^o}{\left( \beta_{\min}^o \right)^2} + \frac{\log \left( d_0^2 / \left( \beta_{\min}^o \right)^2 \right)}{2\kappa}\\
        & + \frac{2\log s}{\beta_{\min}^o} + \frac{2\sqrt{s}}{\sqrt{s \left( \beta_{\min}^o \right)^2}}.
    \end{align*}
    Therefore we get
    \begin{multline*}
        \lambda' \left( \tau_{\infty}'(\mu) - 4\delta \right) < \lambda' \left( k_4-4 \right)\delta\\
        < \frac{1}{\kappa} \log \frac{1}{\mu} + \frac{2\log s + 4 + d_0 / \kappa}{\beta_{\min}^o},
    \end{multline*}
    a contradiction with the definition of $\tau_{\infty}'(\mu)$. So there exists some $k \le \tau_{\infty}'(\mu) / \delta$ satisfying \cref{eq:glbi-orc-cstc-sign}. Then continue to imitate the proof in the continous version, we obtain \cref{eq:glbi-orc-cstc-sign} for all $k\ge \tau_{\infty}'(\mu) / \delta$. The proof of \cref{eq:glbi-orc-cstc-l2} follows the same spirit.
\end{proof}

\begin{proof}[Proof of \Cref{thm:glbi-cstc}]
    It is merely discrete version of the proof of \Cref{thm:glbiss-cstc}. In the proof, \Cref{thm:glbi-orc-gbi} and \ref{thm:glbi-orc-cstc} are applied, instead of \Cref{thm:glbiss-orc-gbi} and \ref{thm:glbiss-orc-cstc}.
\end{proof}

\section{Proof of Proposition \ref{thm:rsc-irr-logistic} on RSC and IRR for Typical Sparse Logistic Regression Models}
\label{sec:rsc-irr-logistic}

In this section we will prove the continuous form of \Cref{thm:rsc-irr-logistic} (then with slight modification we obtain the discrete form \Cref{thm:rsc-irr-logistic}). More specifically, suppose $x^{(i)}$'s are i.i.d. drawn from some $X \sim N(0, \Sigma)$, where $\Sigma_{j,j}\le 1\ (1\le j\le p)$. We are going to prove that there exist $C_0, C_1, C_2 > 0$ such that the continuous version of RSC and IRR (\Cref{thm:rsc-glbiss} and \ref{thm:irr-glbiss}) hold with probability not less than $1 - C_0 / p$, as long as $\kappa$ is sufficiently large and \cref{eq:rsc-irr-logistic-cond} holds. For simplicity, we will take $\kappa \rightarrow +\infty$ in our proof.

For $\theta = (\alpha, \beta^T)^T$, let
\begin{gather*}
    L(\theta) = L(\alpha, \beta) := \mathbb{E} \left[ \ell (\alpha, \beta) \right],\\
    \Theta_S := \left\{ \theta:\ \beta_{S^c} = 0 \right\},\\
\end{gather*}
Note that
\begin{align*}
    \nabla L(\theta) &= \mathbb{E} \left[ \nabla \ell(\alpha, \beta) \right],\\
    &= \mathbb{E} \left[ \frac{1}{1 + \exp\left( \alpha + \beta^T X \right)}\cdot \left( 1, X^T \right)^T y \right],\\
    \nabla^2 L(\theta) &= \mathbb{E} \left[ \nabla^2 \ell(\alpha, \beta) \right]\\
    &= \mathbb{E} \left[ \left( 1, X^T \right)^T \eta \left( \alpha, \beta; X \right) \left( 1, X^T \right) \right],
\end{align*}
where
\begin{equation*}
    \eta\left( \alpha, \beta; x \right) := \frac{\exp\left( \alpha + \beta^T x \right)}{\left( 1 + \exp\left( \alpha + \beta^T x \right) \right)^2}.
\end{equation*}
Since
\begin{equation*}
    \mathbb{E} \left[ \left( 1, X_S^T \right)^T \left( 1, X_S^T \right) \right] = \begin{pmatrix} 1 & 0\\ 0 & \Sigma_{S,S} \end{pmatrix} \succ 0,
\end{equation*}
there must exist some $c > 0$ such that
\begin{equation*}
    \mathbb{E} \left[ 1_{\|X_S\|_2 \le c} \left( 1, X_S^T \right)^T \left( 1, X_S^T \right) \right] \succ 0,
\end{equation*}
Then for any $\theta = (\alpha, \beta^T)^T\in \Theta_S$,
\begin{align*}
    & \nabla_{S_{\alpha}, S_{\alpha}}^2 L(\theta)\\
    \succeq{}& \mathbb{E} \left[ 1_{\|X_S\|_2 \le c} \cdot \left( 1, X_S^T \right)^T \eta \left( \alpha, \beta_S; X_S \right) \left( 1, X_S^T \right) \right]\\
    \succeq{}& \eta\left( |\alpha|, \|\beta_S\|_2; c \right) \cdot \mathbb{E} \left[ 1_{\|X_S\|_2 \le c} \left( 1, X_S^T \right)^T \left( 1, X_S^T \right) \right]
\end{align*}
which is positive-definite. Hence $L(\theta)$ is strictly convex on $\Theta_S$, with the only global minimum obtained at $\theta^{\star}$ (since $\mathbb{E}[\nabla L(\theta^{\star})] = 0$). Now we know
\begin{equation*}
    L\left( \theta^{\star} \right) < L(0) \left( = \log 2 \right) < 1 < \sup_{\theta\in \Theta_S} L(\theta) = +\infty.
\end{equation*}
Let
\begin{equation*}
    D := \left\{ \theta \in \Theta_S:\ L(\theta) \le 1 \right\}.
\end{equation*}
Intuitively, $D$ is the set of all estimators $\theta$'s which are not much worse than the ``trivial estimator'' $0$, and later we will show $\theta^o,\ \theta^{\star},\ \theta'(t)\ (t\ge 0)$ all drop in $D$ with high probability. Obviously $D$ is a compact subset of $\Theta_S$, and there exist $\lambda, \Lambda > 0$ such that
\begin{equation*}
    2 \lambda I \preceq \nabla_{S_{\alpha}, S_{\alpha}}^2 L(\theta) \preceq \frac{\Lambda}{2} I\ \left( \theta \in D \right).
\end{equation*}
Besides, for $\theta\in D$, by the Taylor expansion of $L(\theta)$ at $\theta^\star$, with the fact that $\beta_{S^c} = \beta_{S^c}^\star = 0,\ \nabla L(\theta^\star) = 0$, we have
\begin{align*}
    1 - 0 &\ge L(\theta) - L\left( \theta^\star \right)\\
    & = \left( \theta_{S_\alpha} - \theta_{S_\alpha}^\star \right)^T H_{S_{\alpha}, S_{\alpha}}^L (\theta) \left( \theta_{S_\alpha} - \theta_{S_\alpha}^\star \right)\\
    & \ge \frac{\lambda}{2} \left\| \theta_{S_\alpha} - \theta_{S_\alpha}^\star \right\|_2^2 = \frac{\lambda}{2} \left\| \theta - \theta^\star \right\|_2^2,
\end{align*}
where
\begin{equation*}
    H^L(\theta) = \int_0^1 (1-\mu) \nabla^2 L \left( \theta^\star + \mu \left( \theta - \theta^\star \right) \right) \mathrm{d} \mu.
\end{equation*}
So
\begin{equation}
    \label{eq:theta-minus-star-upper-l2}
    \left\| \theta - \theta^\star \right\|_2 \le \sqrt{2/\lambda}\ (\theta\in D)
\end{equation}
and
\begin{equation}
    \label{eq:theta-upper-l2}
    \left\| \theta \right\|_2 \le \left\| 0 - \theta^\star \right\|_2 + \left\| \theta - \theta^\star \right\|_2 \le \sqrt{8/\lambda}\ (\theta\in D).
\end{equation}
Now we need some lemmas.

\begin{lemma}
    \label{thm:lemma-HI}
    Let $x^{(1)}, \ldots, x^{(n)}$ are i.i.d. drawn from a random variable $X$ on $\left( \Omega, \mathscr{F}, \mathbb{P} \right)$. $h:\ \Omega \rightarrow \mathbb{R}$, and for $E_h := \mathbb{E}\left[ h(X) \right]$ there exists some $\sigma > 0$ and $0 < t_0 \le +\infty$ such that
    \begin{equation*}
        \mathbb{E} \left[ \exp \left( t \left( h(X) - E_h \right) \right) \right] \le \exp \left( \frac{\sigma^2 t^2}{2} \right)\ \left( |t| < t_0 \right)
    \end{equation*}
    Then for any $\epsilon > 0$,
    \begin{multline}
        \label{eq:lemma-HI}
        \mathbb{P} \left( \left| \frac{1}{n} \sum_{i=1}^n h\left( x^{(i)} \right) - E_h \right| \ge \epsilon \right)\\
        \le 2 \exp \left( - \frac{n \epsilon^2}{2\left( \sigma^2 + \epsilon / t_0 \right)} \right).
    \end{multline}
    Consequently, if
    \begin{equation*}
        n\ge \frac{2 \left( \sigma^2 + \epsilon / t_0 \right)}{\epsilon^2} \log \frac{2}{\delta},
    \end{equation*}
    then the left hand side of \cref{eq:lemma-HI} is not greater than $\delta$.
\end{lemma}

\begin{proof}[Proof of \Cref{thm:lemma-HI}]
    Note that for arbitrary $0 < t < t_0$,
    \begin{align*}
        & \mathbb{P} \left( \frac{1}{n} \sum_{i=1}^n h\left( x^{(i)} \right) - E_h \ge \epsilon \right)\\
        ={}& \mathbb{P} \left( \exp \left( t \sum_{i=1}^n \left( h \left( x^{(i)} \right) - E_h \right) \right) \ge \exp \left( n t \epsilon \right) \right)\\
        \le{}& \exp\left( - n t \epsilon \right) \left( \mathbb{E} \left[ \exp \left( t \left( h(X) - E_h \right) \right) \right] \right)^n\\
        \le{}& \exp\left( \frac{n \sigma^2 t^2}{2} - n t \epsilon \right).
    \end{align*}
    and
    \begin{equation*}
        \min_{0 < t\le t_0} \left( \frac{n \sigma^2 t^2}{2} - nt\epsilon \right)
        \begin{cases}
            = - (n\epsilon^2)/(2\sigma^2),& \epsilon / \sigma^2 < t_0,\\
            \le - (n\epsilon)/(2/t_0),& \epsilon / \sigma^2 \ge t_0.
        \end{cases}
    \end{equation*}
    Thus
    \begin{multline*}
        \mathbb{P} \left( \frac{1}{n} \sum_{i=1}^n h\left( x^{(i)} \right) - E_h \ge \epsilon \right)\\
        \le \max \left( \exp\left( - \frac{n\epsilon^2}{2\sigma^2} \right),\ \exp\left( - \frac{n\epsilon}{2/t_0} \right) \right)\\
        \le 2 \exp \left( - \frac{n \epsilon^2}{2\left( \sigma^2 + \epsilon / t_0 \right)} \right).
    \end{multline*}
    In a similar way, we obtain the inequality on the other side.
\end{proof}

\begin{lemma}
    \label{thm:lemma-ULLN}
    Let $x^{(1)}, \ldots, x^{(n)}$ are i.i.d. drawn from a random variable $X$ on $\left( \Omega, \mathscr{F}, \mathbb{P} \right)$, and
    \begin{equation*}
        \Xi = \left\{ \xi \in \mathbb{R}^d:\ \left\| \xi - \xi_0 \right\|_2 \le C \right\} \subseteq \mathbb{R}^d
    \end{equation*}
    where $C > 0$ and $\xi_0 \in \mathbb{R}^d$. $h:\ \Xi\times \Omega \rightarrow \mathbb{R}$ satisfies
    \begin{equation}
        \label{eq:h-lipschitz}
        \left| h\left( \xi; X \right) - h\left( \xi'; X \right) \right| \le M(X) \left\| \xi - \xi' \right\|_2
    \end{equation}
    for any $\xi, \xi' \in \Xi,\ x\in \Omega$, where
    \begin{equation*}
        V_M := \mathbb{E} \left[ M(X)^2 \right] < +\infty,
    \end{equation*}
    and for $E_h(\xi) := \mathbb{E}\left[ h(\xi; X) \right]$ there exists some $\sigma > 0$ and $0 < t_0 \le +\infty$ such that
    \begin{multline*}
        \mathbb{E} \left[ \exp\left( t \left( h(\xi; X) - E_h (\xi) \right) \right) \right]\\
        \le \exp \left( \frac{\sigma^2 t^2}{2} \right)\ \left( \xi \in \Xi,\ |t| < t_0 \right).
    \end{multline*}
    Then for any $\epsilon > 0$ and $0 < \delta \le 1$,
    \begin{multline}
        \label{eq:lemma-ULLN}
        \mathbb{P} \left( \sup_{\xi\in \Xi} \left| \frac{1}{n} \sum_{i=1}^n h\left( \xi; x^{(i)} \right) - E_h(\xi) \right| \ge \epsilon \right) \le \frac{\delta}{2} + {}\\
        2 \exp\left( d \log \left( 1 + \frac{16 C \sqrt{V_M}}{\epsilon \delta} \sqrt{d} \right) - \frac{n\epsilon^2}{8 \left( \sigma^2 + \epsilon / t_0 \right)} \right).
    \end{multline}
    Consequently, if
    \begin{equation*}
        n \ge \frac{8 \left( \sigma^2 + \epsilon / t_0 \right)}{\epsilon^2} \left( \log \frac{4}{\delta} + d \log \left( 1 + \frac{16 C \sqrt{V_M}}{\epsilon \delta} \sqrt{d} \right) \right),
    \end{equation*}
    then the left hand side of \cref{eq:lemma-ULLN} is not greater than $\delta$.
\end{lemma}

\begin{remark}
    If there exists some $M(X)$ with finite second order moment, such that $\| \nabla_\xi h(\xi; X) \|_2 \le M(X)$ almost surely for any $\xi\in \Xi$, then \cref{eq:h-lipschitz} holds according to the Mean Value Theorem.
\end{remark}

\begin{proof}[Proof of \Cref{thm:lemma-ULLN}]
    Let
    \begin{equation*}
        \zeta := \frac{\epsilon}{4 \sqrt{V_M} (1 + 1/\sqrt{\delta}) \sqrt{d}}.
    \end{equation*}
    and consider
    \begin{equation*}
        \Xi_{\mathrm{grid}} = \left\{ \left( \zeta g_1, \ldots, \zeta g_s \right)\in \Xi:\ g_1, \ldots, g_s\in \mathbb{Z} \right\}.
    \end{equation*}
    For any $\xi\in \Xi$, there is some $\xi'\in \Xi_{\mathrm{grid}}$ such that
    \begin{equation*}
        \left\| \xi' - \xi \right\|_2 \le \zeta \sqrt{d}.
    \end{equation*}
    Thus
    \begin{align*}
        & \left| \frac{1}{n} \sum_{i=1}^n h\left( \xi; x^{(i)} \right) - E_h (\xi) \right|\\
        \le{}& \left| \frac{1}{n} \sum_{i=1}^n h\left( \xi; x^{(i)} \right) - \frac{1}{n} \sum_{i=1}^n h\left( \xi'; x^{(i)} \right) \right|\\
        & + \left| \frac{1}{n} \sum_{i=1}^n h\left( \xi'; x^{(i)} \right) - E_h \left( \xi' \right) \right| + \left| E_h \left( \xi' \right) - E_h (\xi) \right|\\
        \le{}& \zeta \sqrt{d} \left( \frac{1}{n} \sum_{i=1}^n M\left( x^{(i)} \right) + \mathbb{E}[M(X)] \right)\\
        & + \left| \frac{1}{n} \sum_{i=1}^n h\left( \xi'; x^{(i)} \right) - E_h \left( \xi' \right) \right|.
    \end{align*}
    Therefore by \Cref{thm:lemma-HI}
    \begin{align*}
        & \mathbb{P} \left( \sup_{\xi\in \Xi} \left| \frac{1}{n} \sum_{i=1}^n h\left( \xi; x^{(i)} \right) - E_h (\xi) \right| \ge \epsilon \right)\\
        \le{}& \mathbb{P} \left( \frac{1}{n} \sum_{i=1}^n M\left( x^{(i)} \right) + \mathbb{E} \left[ M(X) \right] \ge \frac{\epsilon}{2} \left( \zeta \sqrt{d} \right)^{-1} \right)\\
        & + \mathbb{P} \left( \max_{\xi' \in \Xi_{\mathrm{grid}}} \left| \frac{1}{n} \sum_{i=1}^n h\left( \xi'; x^{(i)} \right) - E_h \left( \xi' \right) \right| \ge \frac{\epsilon}{2} \right)\\
        \le{}& \mathbb{P} \left( \frac{1}{n} \sum_{i=1}^n \left( M\left( x^{(i)} \right) - \mathbb{E} \left[ M(X) \right] \right) \ge \sqrt{\frac{2V_M}{\delta}} \right)\\
        & + \left| \Xi_{\mathrm{grid}} \right| \cdot 2 \exp\left( -\frac{n\epsilon^2}{8 \left( \sigma^2 + \epsilon / t_0 \right)} \right)\\
        \le{}& \frac{1}{n^2 (2V_M / \delta)} \cdot n \mathbb{E} \left[ \left( M(X) - \mathbb{E} \left[ M(X) \right] \right)^2 \right]\\
        & + 2 \left( 1 + 2 \cdot \frac{C}{\zeta} \right)^d \exp\left( -\frac{n\epsilon^2}{8 \left( \sigma^2 + \epsilon / t_0 \right)} \right)
    \end{align*}
    and now the right hand side is not greater than the right hand side of \cref{eq:lemma-ULLN}.
\end{proof}

\begin{lemma}
    \label{thm:lemma-mgf-control}
    Let $U,V$ are random variables such that $|U|\le V$ almost surely, and there exist $\sigma > 0,\ 0 < t_0 \le +\infty$ and $c \ge 1$ such that
    \begin{equation*}
        \mathbb{E}\left[ \exp(tV) \right] \le c \cdot \exp \left( \frac{\sigma^2 t^2}{2} \right)\ (0 \le t < t_0).
    \end{equation*}
    Then
    \begin{multline*}
        \mathbb{E}\left[ \exp\left( t(U-\mathbb{E}[U]) \right) \right]\\
        \le \exp\left( \mathrm{e}c^2 \left( \sigma^2 + 1 / t_0^2 \right) t^2 \right)\ (|t| < t_0).
    \end{multline*}
\end{lemma}

\begin{proof}
    For $U' = U - \mathbb{E} [U]$ we have $|U'| \le |V| + \mathbb{E}[V]$ almost surely. Take an arbitrary $t'\in (0, t_0)$. For $t' \le |t| < t_0$,
    \begin{align*}
        & \mathbb{E}\left[ \exp\left( t U' \right) \right] \le \mathbb{E}\left[ \exp\left( |t|(V+\mathbb{E}[V]) \right) \right]\\
        ={}& \mathbb{E}\left[ \exp\left( |t|V \right) \right] \cdot \exp\left( |t| \mathbb{E}[V] \right) \le \left( \mathbb{E}\left[ \exp\left( |t|V \right) \right] \right)^2\\
        \le{}& c^2 \exp\left( \sigma^2 t^2 \right) \le \exp\left( \left( \sigma^2 + \frac{c^2}{t'^2} \right) t^2 \right)\\
        \le{}& \exp\left( \frac{c^2}{t'^2} \left( 1 + \sigma^2 t'^2 \right) t^2 \right) \le \exp\left( \frac{c^2}{t'^2} \exp\left( \sigma^2 t'^2 \right) t^2 \right).
    \end{align*}
    Now define $\psi(x) = \mathrm{e}^x - x - 1$. It is easy to verify that $0\le \psi(-x) \le \psi(x)\ (x\ge 0)$ and $\psi(tx)\le t^2 \psi(x)\ (0\le t\le 1,\ x\ge 0)$. So
    \begin{align*}
        & \mathbb{E}\left[ \exp\left( t U' \right) \right] = 1 + \mathbb{E}\left[ U' \right] + \mathbb{E} \left[ \psi\left( tU' \right) \right]\\
        \le{}& 1 + \mathbb{E} \left[ \psi\left( \left| tU' \right| \right) \right] = 1 + \mathbb{E} \left[ \psi\left( \left| \frac{t}{t'} \right| \left| t'U' \right| \right) \right]\\
        \le{}& 1 + \frac{t^2}{t'^2} \mathbb{E}\left[ \psi\left( \left| t'U' \right| \right) \right] \le 1 + \frac{t^2}{t'^2} \mathbb{E}\left[ \exp\left( \left| t'U' \right| \right) \right] \\
        \le{}& 1 + \frac{t^2}{t'^2} \cdot c^2 \exp\left( \sigma^2 t'^2 \right) \le \exp\left( \frac{c^2}{t'^2} \exp\left( \sigma^2 t'^2 \right) t^2 \right).
    \end{align*}
    Hence for $|t| < t_0$,
    \begin{equation*}
        \mathbb{E}\left[ \exp\left( tU' \right) \right] \le \exp\left( \frac{c^2}{t'^2} \exp\left( \sigma^2 t'^2 \right) t^2 \right).
    \end{equation*}
    Letting $t'^2 = 1 / (\sigma^2 + 1 / t_0^2) < t_0^2$, the right hand side is not greater than
    \begin{multline*}
        \exp\left( c^2 \left( \sigma^2 + 1 / t_0^2 \right) \exp\left( \frac{\sigma^2}{\sigma^2 + 1 / t_0^2} \right) t^2 \right)\\
        \le \exp\left( \mathrm{e} c^2 \left( \sigma^2 + 1 / t_0^2 \right) t^2 \right).
    \end{multline*}
\end{proof}

We then prove $\theta^\star, \theta^o, \theta'(t)\ (t\ge 0)$ all drop in $D$ with high probability. Obviously $\theta^{\star} \in D$. Let
\begin{equation*}
    \hat{D} = \left\{ \theta \in \Theta_S:\ \ell(\theta)\le \ell(0) = \log 2 \right\},
\end{equation*}
then $\theta^o, \theta'(t)\in \hat{D}$ since
\begin{equation*}
    \ell\left( \theta^o \right) = \min_{\theta\in \Theta_S} \ell(\theta) \le \ell\left( \theta'(t) \right) \le \ell(0) = \log 2.
\end{equation*}
It suffices to show $\hat{D} \subseteq D$ with high probability. If $\hat{D} \nsubseteq D$, we can find some $\theta' \in \hat{D} - D$, and then can find some $\theta''\in \hat{D}\cap \partial D\ (\in D)$ on the line segment between $0$ and $\theta'$. So
\begin{multline*}
    \sup_{\theta\in D} \left| \frac{1}{n} \sum_{i=1}^n h\left( \theta_{S_\alpha}; x_S^{(i)}, y^{(i)} \right) - \mathbb{E} \left[ h\left( \theta_{S_\alpha}; X_S, y \right) \right] \right|\\
    \ge L\left( \theta'' \right) - \ell\left( \theta'' \right) \ge 1 - \log 2,
\end{multline*}
where
\begin{equation*}
    h\left( \theta_{S_\alpha}; X_S, y \right) = \log \left( 1 + \exp\left( - \left( \alpha + \beta_S^T X_S \right) y \right) \right).
\end{equation*}
It is easy to see
\begin{align*}
    & \left\| \nabla_{\theta_{S_\alpha}} h\left( \theta_{S_\alpha}; X_S, y \right) \right\|_2 \le \sqrt{1 + \left\| X_S \right\|_2^2}
\end{align*}
holds for any $\theta\in D$, with the right hand side having second order moment $\lesssim s$. Besides,
\begin{equation*}
    \left| h\left( \theta_{S_\alpha}; X_S, y \right) \right| \le V := \log 2 + |\alpha| + |W|,
\end{equation*}
with
\begin{equation*}
    W := \beta_S^T X_S \sim N(0,\ \beta_S^T \Sigma_{S,S} \beta_S)
\end{equation*}
and
\begin{equation*}
    \beta_S^T \Sigma_{S,S} \beta_S \le \|\Sigma_{S,S}\|_{\mathrm{F}} \cdot \|\beta_S\|_2^2 \le 8s/\lambda.
\end{equation*}
So for $0\le t < 1$ we have
\begin{align*}
    & \mathbb{E} \left[ \exp\left( t V \right) \right]\\
    \le{}& \exp((\log 2 + |\alpha|) t) \cdot \mathbb{E} \left[ \exp\left( t W \right) + \exp\left( - t W \right) \right]\\
    \le{}& 2 \exp\left( \left( \log 2 + \sqrt{8/\lambda} \right)t \right) \exp\left( (4s/\lambda) t^2 \right)\\
    \le{}& 4 \exp\left( \sqrt{8/\lambda} \right) \cdot \exp\left( (4s/\lambda) t^2 \right),
\end{align*}
and by \Cref{thm:lemma-mgf-control}, for $|t|<1$,
\begin{multline*}
    \mathbb{E} \left[ \exp\left( t \left( h(\theta_{S_\alpha}; X_S, y) - \mathbb{E} \left[ h(\theta_{S_\alpha}; X_S, y) \right] \right) \right) \right]\\
    \le \exp\left( 16 \mathrm{e}\cdot \exp\left( \sqrt{32/\lambda} \right) \left( (8s/\lambda) + 1 \right) t^2 \right)\\
    \sim \exp \left( s t^2 \right).
\end{multline*}
And then by \Cref{thm:lemma-ULLN}, with probability $\lesssim 1/p$,
\begin{multline*}
    \sup_{\theta\in D} \left| \frac{1}{n} \sum_{i=1}^n h\left( \theta_S; x_S^{(i)}, y^{(i)} \right) - \mathbb{E} \left[ h\left( \theta_S; X_S, y \right) \right] \right|\\
    \ge 1 - \log 2
\end{multline*}
as long as $n \gtrsim s^2\log p$. To conclude, $\theta^\star, \theta^o, \theta'(t)\in D$ with high probability. From now we assume $\theta^{\star}, \theta^o, \theta'(t)\in D$. We then prove the left hand side of RSC \cref{eq:rsc-glbiss} holds with high probability. Since $\nabla_{S_\alpha, S_\alpha}^2 L(\theta)\succeq 2\lambda I$, it suffices to show that with high probability,
\begin{equation*}
    \left\| \nabla_{S_\alpha, S_\alpha}^2 \ell(\theta) - \nabla_{S_\alpha, S_\alpha}^2 L(\theta) \right\|_2 \le \lambda,
\end{equation*}
which is equivalent to
\begin{equation*}
    \sup_{\|\varphi\|_2 \le 1} \left| \frac{1}{n} \sum_{i=1}^n h \left( \theta_{S_\alpha}, \varphi; X_S \right) - \mathbb{E} \left[ h\left( \theta_{S_\alpha}, \varphi; x_S^{(i)} \right) \right] \right| \le \lambda,
\end{equation*}
where
\begin{equation*}
    h\left( \theta_{S_\alpha}, \varphi; X_S \right) = \eta\left( \alpha, \beta_S; X_S \right) \left( \varphi_\alpha + \varphi^T X_S \right)^2.
\end{equation*}
It is easy to find some $M(X)$ with second order moment $\lesssim s^{3/2}$, such that
\begin{equation*}
    \left\| \nabla_{(\theta_{S_\alpha}, \varphi)} h\left( \theta_{S_\alpha}, \varphi; X_S \right) \right\|_2 \le M(X)
\end{equation*}
holds for any $(\theta, \varphi)\in D \times \{\varphi\in \mathbb{R}^{s+1}:\ \|\varphi\|_2 \le 1\}$, using the fact that $\|(\theta_{S_\alpha}, \varphi)\|_2$ is bounded by a constant. Besides,
\begin{equation*}
    \left| h\left( \theta_{S_\alpha}, \varphi; X_S \right) \right| \le \left( 1 + \left\| X_S \right\|_2 \right)^2 \le V := 2 + 2 \left\| X_S \right\|_2^2,
\end{equation*}
Note that
\begin{align*}
    &\mathbb{E}\left[ V^k \right] \le \frac{(2(s+1))^k}{s+1} \mathbb{E} \left[ 1 + \sum_{j\in S} \mathbb{E} \left[ X_j^{2k} \right] \right]\\
    \le{}& (2(s+1))^k \left( 2k-1 \right)!! = (s+1)^k \frac{(2k)!}{k!}\le (2s)^k \frac{(2k)!}{k!}.
\end{align*}
So for $0\le t < 1 / (16s)$,
\begin{align*}
    &\mathbb{E}\left[ \exp\left( tV \right) \right] = \sum_{k=0}^{\infty} \frac{1}{k!} \mathbb{E} \left[ (tV)^k \right] \le \sum_{k=0}^{\infty} {2k \choose k} (2st)^k\\
    ={}& \frac{1}{\sqrt{1 - 8st}} \le \exp\left( 4st + 32s^2 t^2 \right) \le 2 \exp\left( 32 s^2 t^2 \right),
\end{align*}
and by \Cref{thm:lemma-mgf-control}, for $|t| < 1/(16s)$,
\begin{multline*}
    \mathbb{E} \left[ \exp\left( t \left( h\left( \theta_{S_\alpha}, \varphi; X_S \right) - \mathbb{E} \left[ h\left( \theta_{S_\alpha}, \varphi; X_S \right) \right] \right) \right) \right]\\
    \le \exp\left( 1280 \mathrm{e} \cdot s^2 t^2 \right) \lesssim \exp\left( s^2 t^2 \right).
\end{multline*}
And then by \Cref{thm:lemma-ULLN}, with probability $\lesssim 1/p$,
\begin{equation*}
    \sup_{\|\varphi\|_2 \le 1} \left| \frac{1}{n} \sum_{i=1}^n h \left( \theta_{S_\alpha}, \varphi; X_S \right) - \mathbb{E} \left[ h\left( \theta_{S_\alpha}, \varphi; x_S^{(i)} \right) \right] \right| \ge \lambda
\end{equation*}
as long as $n \gtrsim s^3 \log p$. To conclude, the left hand side of RSC holds with high probability. In the same way, we can deal with the right hand side. From now we assume that RSC holds.

Similarly we can prove
\begin{multline*}
    \mathbb{P} \left( \sup_{\theta\in D} \left\| \nabla_{\cdot, S_\alpha}^2 \ell(\theta) - \nabla_{\cdot, S_\alpha}^2 L(\theta) \right\|_{\infty} \ge \epsilon \right)\\
    \lesssim \exp\left( s \log p - \frac{n\epsilon^2}{s} \right)
\end{multline*}
and then obviously
\begin{multline}
    \label{eq:H-mean-upper-uniform}
    \mathbb{P} \left( \sup_{\theta\in D} \left\| \bar{H}_{\cdot, S_\alpha} (\theta) - \mathbb{E} \left[ \bar{H}_{\cdot, S_\alpha} (\theta) \right] \right\|_{\infty} \ge \epsilon \right)\\
    \lesssim \exp\left( s\log p - \frac{n\epsilon^2}{s} \right)
\end{multline}
which will be useful later.

We then prove IRR \cref{eq:irr-glbiss} holds with high probability.

\begin{lemma}
    \label{thm:lemma-irr-const}
    For any $\theta\in \Theta_S$, we have
    \begin{equation*}
        \nabla_{S^c, S_{\alpha}}^2 L(\theta) = \mathrm{irr}_0 \cdot \nabla_{S_{\alpha}, S_{\alpha}}^2 L(\theta),
    \end{equation*}
    where
    \begin{equation*}
        \mathrm{irr}_0 := \left( 0_{p-s},\ \Sigma_{S^c, S} \Sigma_{S,S}^{-1} \right) \in \mathbb{R}^{(p-s)\times (s+1)}.
    \end{equation*}
\end{lemma}

\begin{proof}
    Note that for $U := X_{S^c} - \Sigma_{S^c, S} \Sigma_{S,S}^{-1} X_S$,
    \begin{multline*}
        \begin{pmatrix} X_S\\ U \end{pmatrix} = \begin{pmatrix} I_s & 0 \\ - \Sigma_{S^c, S} \Sigma_{S,S}^{-1} & I_{p-s} \end{pmatrix} \begin{pmatrix} X_S \\ X_{S^c} \end{pmatrix}\\
        \sim N\left( 0,\ \begin{pmatrix} \Sigma_{S,S} & 0\\ 0 & \Sigma_{S^c, S^c} - \Sigma_{S^c, S} \Sigma_{S,S}^{-1} \Sigma_{S, S^c} \end{pmatrix} \right),
    \end{multline*}
    which implies that $X_S$ is independent of $U$. So
    \begin{align*}
        & \nabla_{S^c, S_{\alpha}}^2 L(\theta)\\
        ={}& \mathbb{E} \left[ \eta\left( \alpha, \beta_S; X_S \right) X_{S^c} \left( 1,\ X_S^T \right) \right]\\
        ={}& \mathbb{E} \left[ \eta\left( \alpha, \beta_S; X_S \right) U \left( 1,\ X_S^T \right) \right]\\
        & + \mathbb{E} \left[ \eta\left( \alpha, \beta_S; X_S \right) \Sigma_{S^c, S} \Sigma_{S,S}^{-1} X_S \left( 1,\ X_S^T \right) \right]\\
        ={}& \mathbb{E} \left[ U \right] \cdot \mathbb{E} \left[ \eta \left( \alpha, \beta_S; X_S \right) \left( 1,\ X_S^T \right) \right]\\
        & + \Sigma_{S^c, S} \Sigma_{S,S}^{-1} \cdot \mathbb{E} \left[ \eta\left( \alpha, \beta_S; X_S \right) X_S \left(1, X_S^T \right) \right]\\
        ={}& \mathrm{irr}_0 \cdot \nabla_{S_{\alpha}, S_{\alpha}}^2 L(\theta).
    \end{align*}
\end{proof}

We then prove \cref{eq:irr-glbiss-b} holds with high probability. By \Cref{thm:lemma-irr-const},
\begin{align*}
    & \left\| \overline{\mathrm{irr}}(t) - \mathrm{irr}_0 \right\|_{\infty}\\
    \le{}& \left\| \left( \bar{H}_{S^c, S_{\alpha}} \left( \theta'(t) \right) \right. \right. \\
    & \quad \left. \left. - \mathbb{E} \left[ \bar{H}_{S^c, S_{\alpha}}\left( \theta'(t) \right) \right] \right) \bar{H}_{S_{\alpha}, S_{\alpha}}\left( \theta'(t) \right)^{-1} \right\|_{\infty}\\
    & + \left\| \mathbb{E} \left[ \bar{H}_{S^c, S_{\alpha}}\left( \theta'(t) \right) \right] \right.\\
    & \quad \left. \cdot \left( \bar{H}_{S_{\alpha}, S_{\alpha}}\left( \theta'(t) \right)^{-1} - \mathbb{E} \left[ \bar{H}_{S_{\alpha}, S_{\alpha}}\left( \theta'(t) \right) \right]^{-1} \right) \right\|_{\infty}\\
    \le{}& \left\| \bar{H}_{S^c, S_{\alpha}}\left( \theta'(t) \right) - \mathbb{E} \left[ \bar{H}_{S^c, S_{\alpha}}\left( \theta'(t) \right) \right] \right\|_{\infty} \\
    & \quad \cdot \left\| \bar{H}_{S_{\alpha}, S_{\alpha}}\left( \theta'(t) \right)^{-1} \right\|_{\infty}\\
    & + \left\| \mathrm{irr}_0 \cdot \mathbb{E} \left[ \bar{H}_{S_{\alpha}, S_{\alpha}}\left( \theta'(t) \right) \right] \right.\\
    & \quad \left. \cdot \left( \bar{H}_{S_{\alpha}, S_{\alpha}}\left( \theta'(t) \right)^{-1} - \mathbb{E} \left[ \bar{H}_{S_{\alpha}, S_{\alpha}}\left( \theta'(t) \right) \right]^{-1} \right) \right\|_{\infty}.
\end{align*}
The first term of the right hand side is not greater than
\begin{align*}
    & \left\| \bar{H}_{S^c, S_{\alpha}}(t) - \mathbb{E} \left[ \bar{H}_{S^c, S_{\alpha}}(t) \right] \right\|_{\infty} \sqrt{2s} \left\| \bar{H}_{S_{\alpha}, S_{\alpha}}(t)^{-1} \right\|_2\\
    \le{}& \left\| \bar{H}_{\cdot, S_{\alpha}}(t) - \mathbb{E} \left[ \bar{H}_{\cdot, S_{\alpha}}(t) \right] \right\|_{\infty} \frac{\sqrt{2s}}{\lambda},
\end{align*}
and the right hand side is not greater than
\begin{align*}
    &\left\| \mathrm{irr}_0 \right\|_{\infty} \cdot \left\| \mathbb{E} \left[ \bar{H}_{S_{\alpha}, S_{\alpha}}(t) \right] \bar{H}_{S_{\alpha}, S_{\alpha}}(t)^{-1} - I \right\|_{\infty}\\
    \le{}& (1 - \eta) \left\| \left( \mathbb{E} \left[ \bar{H}_{S_{\alpha}, S_{\alpha}}(t) \right] - \bar{H}_{S_{\alpha}, S_{\alpha}}(t) \right) \bar{H}_{S_{\alpha}, S_{\alpha}}(t)^{-1} \right\|_{\infty}\\
    \le{}& \left\| \bar{H}_{S^c, S_{\alpha}}(t) - \mathbb{E} \left[ \bar{H}_{S^c, S_{\alpha}}(t) \right] \right\|_{\infty} \left\| \bar{H}_{S_{\alpha}, S_{\alpha}}(t)^{-1} \right\|_{\infty}\\
    \le{}& \left\| \bar{H}_{\cdot, S_{\alpha}}(t) - \mathbb{E} \left[ \bar{H}_{\cdot, S_{\alpha}}(t) \right] \right\|_{\infty} \frac{\sqrt{2s}}{\lambda}.
\end{align*}
Hence
\begin{equation}
    \label{eq:irr-mean-upper-uniform}
    \left\| \overline{\mathrm{irr}}(t) - \mathrm{irr}_0 \right\|_{\infty} \le \frac{3\sqrt{s}}{\lambda} \left\| \bar{H}_{\cdot, S_{\alpha}}(t) - \mathbb{E} \left[ \bar{H}_{\cdot, S_{\alpha}}(t) \right] \right\|_{\infty}.
\end{equation}
By \Cref{eq:H-mean-upper-uniform}, the right hand side is not greater than $C = 1$ with high probability as long as $n\gtrsim s^3 \log p$. To conclude, \cref{eq:irr-glbiss-b} holds with high probability. From now we assume \cref{eq:irr-glbiss-b} holds.

Finally we prove \cref{eq:irr-glbiss-a} holds with high prbability. Since we have assumed that $\kappa \rightarrow +\infty$, the left hand side of \cref{eq:irr-glbiss-a} is not greater than
\begin{multline*}
    \left\| \int_0^T \left( \overline{\mathrm{irr}}(t) - \mathrm{irr}_0 \right) \begin{pmatrix} 0\\ \dot{\rho}_S'(t) \end{pmatrix} \mathrm{d}t \right\|_{\infty}\\
    + \left\| \mathrm{irr}_0 \right\|_{\infty} \cdot \left\| \begin{pmatrix} 0\\ \rho_S'(t) \end{pmatrix} \right\|_{\infty},
\end{multline*}
and the second term is not greater than $1 - \eta$. Thus it suffices to show
\begin{equation}
    \label{eq:R-int-upper-infty}
    \left\| \int_0^T R(t) \mathrm{d}t \right\|_{\infty} \le \frac{\eta}{2}
\end{equation}
for any $T\ge 0$, where
\begin{equation}
    \begin{split}
        R(t) &:= \left( \overline{\mathrm{irr}}(t) - \mathrm{irr}_0 \right) \begin{pmatrix} 0\\ \dot{\rho}_S'(t) \end{pmatrix}\\
        & = \left( \overline{\mathrm{irr}}(t) - \mathrm{irr}_0 \right) \left( - \nabla_{S_\alpha} \ell \left( \theta'(t) \right) \right) = R_1(t) + R_2(t),
    \end{split}
\end{equation}
and
\begin{align*}
    R_1(t) &= - \left( \overline{\mathrm{irr}}(t) - \mathrm{irr}_0 \right) \bar{H}_{S_\alpha, S_\alpha} \left( \theta'(t) \right) \left( \theta_{S_\alpha}'(t) - \theta_{S_\alpha}^\star \right),\\
    R_2(t) &= - \left( \overline{\mathrm{irr}}(t) - \mathrm{irr}_0 \right) \nabla_{S_\alpha} \ell\left( \theta^\star \right).
\end{align*}
According to the lower bound of $\beta_{\min}^{\star}$ condition and \cref{eq:glbiss-orc-cstc-sign}, we have that for $t\ge T_0 := (\eta/48) \sqrt{n/\log p}$, $d(t) = 0$, and the solution path is constant. Thus it suffices to show \cref{eq:R-int-upper-infty} for $0\le T\le T_0$. By \Cref{thm:lemma-HI}, with high probability
\begin{equation*}
    \left\| \nabla \ell\left( \theta^{\star} \right) \right\|_{\infty} \le 8 \sqrt{\frac{\log p}{n}}.
\end{equation*}
By \cref{eq:glbiss-orc-cstc-l2} and \cref{eq:theta-upper-l2}, we have that with high probability
\begin{align*}
    &\left\| R_1(t) \right\|_{\infty} \le \left\| \overline{\mathrm{irr}}(t) - \mathrm{irr}_0 \right\|_{\infty} \cdot \Lambda \left\| \theta_{S_\alpha}'(t) - \theta_{S_\alpha}^\star \right\|_2\\
    \le{}&
    \begin{cases}
        (\eta\sqrt{\lambda})/(20\Lambda),& n\gtrsim s^3 \log p,\\
        (\eta\lambda) / (12\Lambda \sqrt{s} \log n) \cdot \sqrt{s}/t,& n / (\log n)^2 \gtrsim s^4 \log p.
    \end{cases}
\end{align*}
So for $0\le T\le T_0$,
\begin{align*}
    \left\| \int_0^T R_1(t) \mathrm{d}t \right\|_{\infty} &\le \int_0^1 \left\| R_1(t) \right\|_{\infty} \mathrm{d}t + \int_1^{T_0} \left\| R_1(t) \right\|_{\infty}\\
    & \le \frac{\eta\sqrt{\lambda}}{20\Lambda} \cdot \Lambda \sqrt{\frac{8}{\lambda}} + \frac{\eta\lambda}{12\Lambda \log n} \cdot \log T\\
    & \le \frac{\eta}{6} + \frac{\eta}{6} = \frac{\eta}{3}
\end{align*}
as long as $n/(\log n)^2 \gtrsim s^4 \log p$. Besides,
\begin{align*}
    & \left\| \int_0^T R_2(t) \mathrm{d}t \right\|_{\infty} \le T_0 \cdot 8 \sqrt{\frac{\log p}{n}} = \frac{\eta}{6}
\end{align*}
with high probability. To conclude, \cref{eq:irr-glbiss-a} holds with high probability.

\section{Ising Model --- Special Case of General Discrete Markov Random Fields}

Recall \cref{eq:model-ising} for Ising model, with $(\alpha^{\star}, \beta^{\star})$ replaced by $(h^{\star}, J^{\star})$. In some research areas, people are accustomed to studying $x' = (x + 1_p) / 2$ whose distribution is given by
\begin{multline}
    \label{eq:model-ising10}
    \mathbb{P}\left( x' = \left( x_1', \ldots, x_p' \right)^T \right)\\
    \propto \exp\left( \sum_{j=1}^p \alpha_{j;1}^{\star} x_j' + \sum_{1\le j<j'\le p} \beta_{j,j';1,1}^{\star} x_j' x_{j'}' \right),\\
    x_1', \ldots, x_p' \in \left\{ 1,0 \right\},
\end{multline}
where
\begin{multline*}
    \beta_{j,j';1,1}^{\star} = 2J_{j,j'}^{\star},\ \alpha_{j;1}^{\star} = h_j^{\star} - \sum_{j'\neq j} J_{j,j'}^{\star}\\
    \Longleftrightarrow J_{j,j'}^{\star} = \frac{\beta_{j,j';1,1}^{\star}}{2},\ h_j^{\star} = \alpha_{j;1}^{\star} + \sum_{j'\neq j} \frac{\beta_{j,j'; 1,1}^{\star}}{2}.
\end{multline*}
One can still consider GLBI, with $\theta = (\alpha, \beta)$ where $\alpha = (\alpha_{j;1})_{1\le j\le p},\ \beta = (\beta_{j,j';1,1})_{j<j'}$. Besides, \cref{eq:model-ising10} can actually be viewed as a special case of general discrete Markov random fields with reduced parameters. See \Cref{sec:gdmrf}.

\section{Application: GLBI for Learning General Discrete Markov Random Fields}
\label{sec:gdmrf}

Similar with the discussion in \citet{ravikumar_high-dimensional_2010}, we may consider a random vector $x = \left( x_1, \ldots, x_p \right)^T$ whose distribution is given by
\begin{multline}
    \label{eq:model-gdmrf}
    \mathbb{P}\left( x = \left( x_1,\ldots,x_p \right)^T \right)\\
    \propto \exp\left( \sum_{\substack{1\le j\le p\\ 1\le l\le q}} \alpha_{r;l}^{\star} 1_{x_j = l} + \sum_{\substack{1\le j<j'\le p\\ 1\le l,l'\le q}} \beta_{j,j'; l,l'}^{\star} 1_{x_j=l} 1_{x_j'=l'} \right)\\
    = \exp\left( \sum_{j=1}^p \alpha_{j;x_j}^{\star} + \sum_{1\le j<j'\le p} \beta_{j,j'; x_j,x_{j'}}^{\star} \right),\\
    x_1, \ldots, x_p \in \mathcal{X} = \left\{ 1,2,\ldots,q \right\},
\end{multline}
where
\begin{gather*}
    \alpha_j^{\star} = ( \alpha_{j;l}^{\star} ) \in \mathbb{R}^q\ (1\le j\le p),\\
    \beta_{j,j'}^{\star} = ( \beta_{j,j'; l,l'}^{\star} ) \in\mathbb{R}^{q\times q}\ (1\le j<j'\le p).\protect\footnote{For convenience we let $\beta_{j,j}^{\star} = 0_{q\times q}\ (1\le j\le p)$ and $\beta_{j',j}^{\star} = \beta_{j,j'}^{\star T}\ (1\le j<j' \le p)$.}
\end{gather*}
Due to the redundancy of $1_{x_j = q}$ which equals to $1 - \sum_{l<q} 1_{x_j = l}$, we can further assume that $\alpha_{j; q}^{\star} = 0$ and $\beta_{j,j'; l,q}^{\star} = 0$ for actual needs in some cases, but in other cases we generally do not make such assumption in order to keep the symmetry of parameters.

The Potts model, which has applications in computational chemistry, is a special case of \cref{eq:model-gdmrf}. Actually, the distribution of $q$-state (standard) Potts model is given by
\begin{multline*}
    \mathbb{P}\left( x = \left( x_1, \ldots, x_p \right)^T \right)\\
    \propto \exp\left( \sum_{\substack{1\le j\le p\\ 1\le l\le q}} \alpha_{j;l}^{\star} 1_{x_j=l} + \sum_{\substack{\left( j,j' \right)\in E\\ l=l'}} \beta_{j,j'; l,l'}^{\star} 1_{x_j=l} 1_{x_{j'}=l'} \right),
\end{multline*}
where $E$ indicates the edges of the true graph.

The Ising model \cref{eq:model-ising10} is also a special case of \cref{eq:model-gdmrf}. Let $q = 2$ and reduce the parameters, the distribution of $x$ has a simple form as follows
\begin{multline*}
    \mathbb{P} \left( x = (x_1, \ldots, x_p)^T \right)\\
    \propto \exp \left( \sum_{j=1}^p \alpha_{j;1}^{\star} 1_{x_j = 1} + \sum_{1\le j<j' \le p} \beta_{j,j'; 1,1}^{\star} 1_{x_j=1} 1_{x_{j'}=1} \right).
\end{multline*}
Thus the distribution of $x' := 2\cdot 1_p - x \in \left\{ 1,0 \right\}^p$ is given by
\begin{multline*}
    \mathbb{P}\left( x' = \left( x_1', \ldots, x_p' \right)^T \right)\\
    \propto \exp\left( \sum_{j=1}^p \alpha_{j;1}^{\star} x_j' + \sum_{1\le j<j'\le p} \beta_{j,j';1,1}^{\star} x_j' x_{j'}' \right),\\
    x_1', \ldots, x_p' \in \left\{ 1,0 \right\},
\end{multline*}
which is just \cref{eq:model-ising10}.

For \cref{eq:model-gdmrf}, we assume the graph is sparse, i.e. most $q\times q$ blocks of $\beta$ are $0$. Let
\begin{equation*}
    \alpha^{\star} =
    \begin{pmatrix}
        \alpha_1\\
        \vdots\\
        \alpha_p
    \end{pmatrix}
    \in \mathbb{R}^{qp},\
    \beta^{\star} =
    \begin{pmatrix}
        \beta_{1,1}^{*} & \cdots & \beta_{1,p}^{\star}\\
        \beta_{2,1}^{*} & \cdots & \beta_{2,p}^{\star}\\
        & \ddots & \\
        \beta_{p,1}^{*} & \cdots & \beta_{p,p}^{\star}
    \end{pmatrix}
    \in \mathbb{R}^{qp\times qp}.
\end{equation*}
The goal is to find the graph structure, i.e. to determine which blocks are totally $0$, given the sample matrix $X = ( x^{(1)}, \ldots, x^{(n)} )^T\in \mathbb{R}^{n\times p}$. Here we do not reduce the parameters, and it is invalid and unnecessary to make estimations on the true parameters.

Note that for $1\le l\le q$,
\begin{multline*}
    \mathbb{P}_{\alpha, \beta} \left( x_j = l | x_{-j} \right) =\\
    \frac{\exp\left( \alpha_{j;l} + \sum_{j'\neq j,\ 1\le l'\le q} \beta_{j,j';l,l'} 1_{x_{j'} = l'} \right)}{\sum_{m=1}^q \exp\left( \alpha_{j;m} + \sum_{j'\neq j,\ 1\le l'\le q} \beta_{j,j';m,l'} 1_{x_{j'} = l'} \right)}.
\end{multline*}
Consider the negative composite conditional log-likelihood
\begin{multline*}
    \ell\left( \alpha, \beta \right) = - \frac{1}{n} \sum_{\substack{1\le j\le p\\ 1\le i\le n}} \log \mathbb{P}_{\alpha, \beta} \left( x_j = x_j^{(i)} | x_{-j} = x_{-j}^{(i)} \right)\\
    = \frac{1}{n} \sum_{\substack{1\le j\le p \\ 1\le i\le n}} \left( \sum_{l=1}^q \alpha_{j;l} 1_{x_j^{(i)} = l} + {} \right.\\
    \left. \sum_{\substack{j'\neq j\\ 1\le l,l'\le q}} \beta_{j,j';l,l'} 1_{x_j^{(i)} = l} 1_{x_{j'}^{(i)} = l'} - {} \right.\\
    \left. \log \left( \sum_{m=1}^q \exp \left( \alpha_{j;m} + \sum_{\substack{j'\neq j\\ 1\le l'\le q}} \beta_{j,j';m,l'} 1_{x_{j'}^{(i)} = l'} \right) \right) \right).
\end{multline*}
A typical regularization approach is to solve the following optimization problem
\begin{gather*}
    \left( \hat{\alpha}, \hat{\beta} \right) = \arg\min_{\alpha, \beta} \left( \ell\left( \alpha, \beta \right) + \lambda P\left( \beta \right) \right),\\
    P\left( \beta \right) = \sum_{1\le j<j'\le p} \left\| \beta_{j,j'} \right\|_{\mathrm{F}},
\end{gather*}
where $\left\| \cdot \right\|_{\mathrm{F}}$ indicates the Frobenius norm. However, the GLBISS has the form
\begin{align*}
    \dot{\alpha}(t) / \kappa &= - \nabla_{\alpha} \ell\left( \alpha(t), \beta(t) \right),\\
    \dot{\rho}(t) + \dot{\beta}(t) / \kappa &= - \nabla_{\beta} \ell\left( \alpha(t), \beta(t) \right),\\
    \rho(t)&\in \partial P \left( \beta(t) \right),
\end{align*}
where $\rho(0) = \beta(0) = 0$. Let $z(t) = \rho(t) + \beta(t)/\kappa$, and we can view $\kappa \rho(t) + \beta(t) - \kappa z(t) = 0$ as the first order optimality condition of the following optimization problem,
\begin{align*}
    &\beta(t)\\
    \in{}& \arg\min_{\beta} \left( \kappa P\left( \beta \right) + \frac{1}{2} \left\| \beta - \kappa z(t) \right\|_\mathrm{F}^2 \right)\\
    ={}& \arg\min_{\beta} \sum_{1\le j<j'\le p} \left( \kappa \left\| \beta_{j,j'} \right\|_{\mathrm{F}} + \frac{1}{2} \left\| \beta_{j,j'} - \kappa z_{j,j'}(t) \right\|_{\mathrm{F}}^2 \right).
\end{align*}
Hence
\begin{align*}
    & \mathrm{vec}\left( \beta_{j,j'}(t) \right)\\
    ={}& \arg\min_{\mathrm{vec}\left( \beta_{j,j'} \right)} \left( \kappa \left\| \mathrm{vec}\left( \beta_{j,j'} \right) \right\|_2 \right.\\
    & \left. + \left\| \mathrm{vec}\left( \beta_{j,j'} \right) - \kappa\cdot \mathrm{vec}\left( z_{j,j'}(t) \right) \right\|_2^2 / 2 \right)\\
    ={}& \mathrm{prox}_{\kappa \left\|\cdot \right\|_2} \left( \kappa\cdot \mathrm{vec}\left( z_{j,j'}(t) \right) \right)\\
    ={}& \left( 1 - \frac{\kappa}{\left\| \kappa \cdot \mathrm{vec}\left( z_{j,j'}(t) \right) \right\|_2} \right)_{+} \cdot \kappa\cdot \mathrm{vec}\left( z_{j,j'}(t) \right)\\
    \Longrightarrow {}& \beta_{j,j'}(t) = \kappa \left( 1 - \frac{1}{\left\| z_{j,j'}(t) \right\|_{\mathrm{F}}} \right)_{+} z_{j,j'}(t),
\end{align*}
and the GLBISS has an alternative form
\begin{align*}
    \dot{\alpha}(t) &= - \kappa \nabla_{\alpha} \ell\left( \beta(t) \right),\\
    \dot{z}(t) &= - \nabla_{\beta} \ell\left( \beta(t) \right),\\
    \beta_{j,j'}(t) &= \kappa \left( 1 - \left\| z_{j,j'}(t) \right\|_{\mathrm{F}}^{-1} \right)_{+} z_{j,j'}(t)\ (1\le j<j'\le p),
\end{align*}
where $z(0) = \beta(0) = 0$. One can similarly derive the form of the corresponding GLBI.

As for the exact form of $\nabla \ell(\alpha, \beta)$, it is not hard to find that
\begin{align*}
    & \nabla_{\alpha_{j; l}} \ell\left( \alpha, \beta \right)\\
    ={}& \frac{1}{n} \sum_{i=1}^n \left( \frac{\exp\left( M_{j;l}^{(i)} \right)} {\sum_{m=1}^q \exp\left( M_{j;m}^{(i)} \right)} - \hat{X}_{j;l}^{(i)} \right),\\
    & \nabla_{\beta_{j,j'; l,l'}} \ell\left( \alpha, \beta \right)\\
    ={}& \frac{1}{n} \sum_{i=1}^n \left( \frac{\exp\left( M_{j;l}^{(i)} \right)} {\sum_{m=1}^q \exp\left( M_{j;m}^{(i)} \right)} - \hat{X}_{j;l}^{(i)} \right) \hat{X}_{j';l'}^{(i)}\\
    & + \frac{1}{n} \sum_{i=1}^n \left( \frac{\exp\left( M_{j';l'}^{(i)} \right)} {\sum_{m=1}^q \exp\left( M_{j';m}^{(i)} \right)} - \hat{X}_{j';l'}^{(i)} \right) \hat{X}_{j;l}^{(i)},
\end{align*}
where
\begin{equation*}
    \hat{X} = \left( 1_{x_j^{(i)}=l} \right)\in \mathbb{R}^{n\times qp},\ M = 1_n \alpha^T + \hat{X} \beta \in \mathbb{R}^{n\times qp}.
\end{equation*}

\section{More Experimental Results}
\label{sec:exp-ext}

Here we list \Cref{tab:simu-logistic-p80-s20-m2}, \ref{tab:simu-logistic-p200-s50-m1}, \ref{tab:simu-logistic-p200-s50-m2} for \Cref{sec:simu-logistic}, and \Cref{tab:simu-ising-p100} for \Cref{sec:simu-ising}. Besides, we have \Cref{fig:real-nips-ext} for \Cref{sec:real-nips}.

For the setting $p = 80,\ s = 20,\ M = 1,\ r = 0.25$ of \Cref{sec:simu-logistic}, we can plot \Cref{fig:simu-logistic-cv}, the CV curves of prediction error. We observe that the CV curve of prediction error by GLBI often drops more rapidly as $k$ increases, than that by \texttt{glmnet} as $\lambda$ decreases, indicating bias reduction in a possible increase of variance. A proper early stopping for GLBI path often provides us a relatively good estimator with small prediction error.

\begin{figure}
    \centering
    \includegraphics[width = 0.8\linewidth]{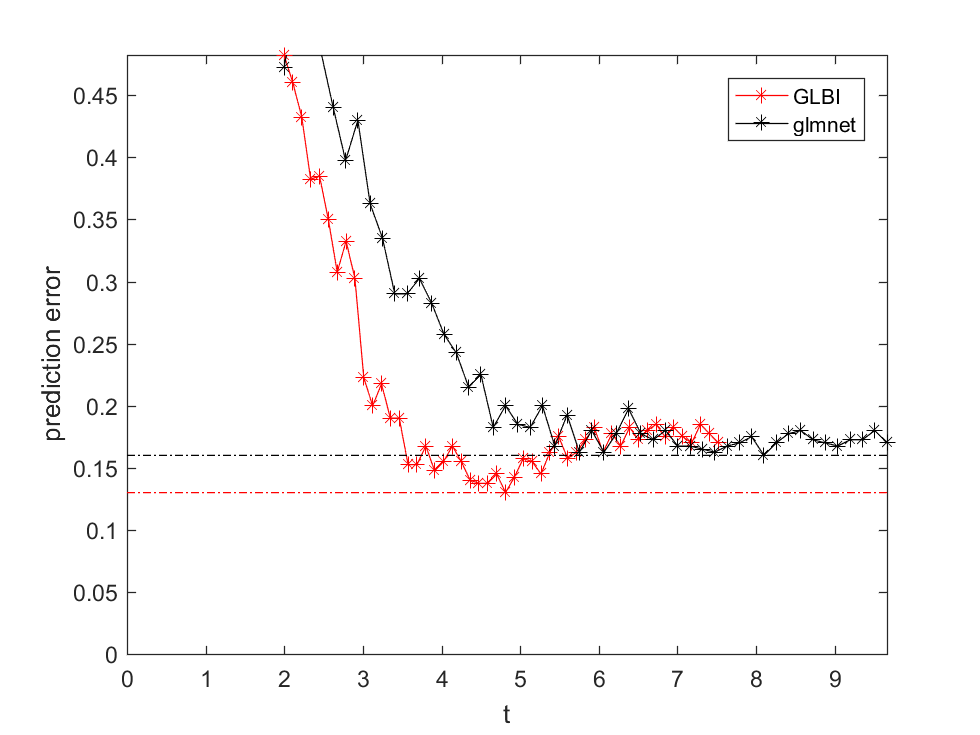}
    \caption{CV curves of prediction error. The red curve represents the CV estimates of prediction error by GLBI, and the dashed red line represents the smallest value among them. The black curve represents the CV estimates of prediction error by \texttt{glmnet}, and the dashed black line represents the smallest value among them. The CV curve of prediction error by GLBI often drops more rapidly as $k$ increases, than that by \texttt{glmnet} as $\lambda$ decreases, indicating bias reduction in a possible increase of variance. A proper early stopping for GLBI path often provides us a relatively good estimator with small prediction error.}
    \label{fig:simu-logistic-cv}
\end{figure}

\begin{table}[!t]
    \caption{Comparisons between GLBI and \texttt{glmnet}, for logistic models with $p = 80,\ s = 20,\ M = 2$. For each algorithm, we run $\mathrm{rep} = 20$ independent experiments.}
    \label{tab:simu-logistic-p80-s20-m2}
    \centering
    \begin{tabular}{cccccc}
        \toprule
        & & \multicolumn{2}{c}{AUC} & \multicolumn{2}{c}{prediction error}\\
        \cmidrule(lr){3-4} \cmidrule(lr){5-6}
        & & GLBI & \texttt{glmnet} & GLBI & \texttt{glmnet}\\
        \cmidrule(lr){1-6}
        $0.25$ & $400$ & $\boldsymbol{.9925}$ & $.9917$ & $\boldsymbol{.0871}$ & $.0935$\\
               &       & $(.0070)$ & $(.0074)$ & $(.0137)$ & $(.0126)$\\
               & $800$ & $.9994$ & $\boldsymbol{.9995}$ & $\boldsymbol{.0663}$ & $.0701$\\
               &       & $(.0018)$ & $(.0015)$ & $(.0089)$ & $(.0071)$\\
        \cmidrule(lr){1-6}
        $0.5$  & $400$ & $\boldsymbol{.9810}$ & $.9803$ & $\boldsymbol{.0914}$ & $.0961$\\
               &       & $(.0129)$ & $(.0119)$ & $(.0216)$ & $(.0202)$\\
               & $800$ & $.9932$ & $\boldsymbol{.9941}$ & $\boldsymbol{.0671}$ & $.0706$\\
               &       & $(.0073)$ & $(.0068)$ & $(.0108)$ & $(.0111)$\\
        \bottomrule
    \end{tabular}
\end{table}

\begin{table}[!t]
    \caption{Comparisons between GLBI and \texttt{glmnet}, for logistic models with $p = 200,\ s = 50,\ M = 1$. For each algorithm, we run $\mathrm{rep} = 20$ independent experiments.}
    \label{tab:simu-logistic-p200-s50-m1}
    \centering
    \begin{tabular}{cccccc}
        \toprule
        & & \multicolumn{2}{c}{AUC} & \multicolumn{2}{c}{prediction error}\\
        \cmidrule(lr){3-4} \cmidrule(lr){5-6}
        $r$ & $n$ & GLBI & \texttt{glmnet} & GLBI & \texttt{glmnet}\\
        \cmidrule(lr){1-6}
        $0.25$ & $600$  & $\boldsymbol{.9702}$ & $.9700$ & $\boldsymbol{.1248}$ & $.1358$\\
               &        & $(.0113)$ & $(.0108)$ & $(.0147)$ & $(.0151)$\\
               & $1000$ & $\boldsymbol{.9948}$ & $.9937$ & $\boldsymbol{.1001}$ & $.1124$\\
               &        & $(.0031)$ & $(.0035)$ & $(.0078)$ & $(.0104)$\\
        \cmidrule(lr){1-6}
        $0.5$  & $600$  & $\boldsymbol{.9391}$ & $.9374$ & $\boldsymbol{.1391}$ & $.1401$\\
               &        & $(.0158)$ & $(.0158)$ & $(.0161)$ & $(.0150)$\\
               & $1000$ & $\boldsymbol{.9749}$ & $.9729$ & $\boldsymbol{.1053}$ & $.1112$\\
               &        & $(.0083)$ & $(.0090)$ & $(.0125)$ & $(.0128)$\\
        \bottomrule
    \end{tabular}
\end{table}

\begin{table}[!t]
    \caption{Comparisons between GLBI and \texttt{glmnet}, for logistic models with $p = 200,\ s = 50,\ M = 2$. For each algorithm, we run $\mathrm{rep} = 20$ independent experiments.}
    \label{tab:simu-logistic-p200-s50-m2}
    \centering
    \begin{tabular}{cccccc}
        \toprule
        & & \multicolumn{2}{c}{AUC} & \multicolumn{2}{c}{prediction error}\\
        \cmidrule(lr){3-4} \cmidrule(lr){5-6}
        & & GLBI & \texttt{glmnet} & GLBI & \texttt{glmnet}\\
        \cmidrule(lr){1-6}
        $0.25$ & $600$  & $\boldsymbol{.9771}$ & $.9757$ & $\boldsymbol{.1083}$ & $.1203$\\
               &        & $(.0103)$ & $(.0087)$ & $(.0180)$ & $(.0154)$\\
               & $1000$ & $\boldsymbol{.9962}$ & $.9951$ & $\boldsymbol{.0737}$ & $.0901$\\
               &        & $(.0025)$ & $(.0024)$ & $(.0068)$ & $(.0072)$\\
        \cmidrule(lr){1-6}
        $0.5$  & $600$  & $.9445$ & $\boldsymbol{.9449}$ & $.1235$ & $\boldsymbol{.1207}$\\
               &        & $(.0206)$ & $(.0162)$ & $(.0128)$ & $(.0122)$\\
               & $1000$ & $\boldsymbol{.9800}$ & $.9780$ & $\boldsymbol{.0819}$ & $.0888$\\
               &        & $(.0065)$ & $(.0070)$ & $(.0093)$ & $(.0083)$\\
        \bottomrule
    \end{tabular}
\end{table}

\begin{table}
    \caption{Comparisons of GLBI1 (GLBI + composite), GLBI2 (GLBI + MPF), and \texttt{glmnet}, for Ising models with $p = 100$. For each algorithm, we run $\mathrm{rep} = 20$ independent experiments.}
    \label{tab:simu-ising-p100}
    \centering
    \begin{tabular}{ccccc}
        \toprule
        & & \multicolumn{3}{c}{AUC}\\
        \cmidrule(lr){3-5}
        $T$ & $n$ & GLBI1 & GLBI2 & \texttt{glmnet}\\
        \cmidrule(lr){1-5}
        $1.25$ & $1000$ & $.9794$ & $\boldsymbol{.9888}$ & $.9807$\\
               &        & $(.0163)$ & $(.0120)$ & $(.0160)$\\
        \cmidrule(lr){2-5}
               & $1500$ & $.9890$ & $\boldsymbol{.9950}$ & $.9911$\\
               &        & $(.0078)$ & $(.0088)$ & $(.0051)$\\
        \cmidrule(lr){1-5}
        $1.5$  & $1000$ & $.9941$ & $\boldsymbol{.9984}$ & $.9955$\\
               &        & $(.0054)$ & $(.0015)$ & $(.0045)$\\
        \cmidrule(lr){2-5}
               & $1500$ & $.9974$ & $\boldsymbol{.9993}$ & $.9983$\\
               &        & $(.0028)$ & $(.0011)$ & $(.0019)$\\
        \bottomrule
    \end{tabular}
    \begin{tabular}{ccccc}
        \toprule
        & & \multicolumn{3}{c}{2nd order MDC}\\
        \cmidrule(lr){3-5}
        $T$ & $n$ & GLBI1 & GLBI2 & \texttt{glmnet}\\
        \cmidrule(lr){1-5}
        $1.25$ & $1000$ & $\boldsymbol{.9872}$ & $.9865$ & $.9868$\\
               &        & $(.0046)$ & $(.0055)$ & $(.0050)$\\
        \cmidrule(lr){2-5}
               & $1500$ & $\boldsymbol{.9912}$ & $.9911$ & $.9904$\\
               &        & $(.0014)$ & $(.0014)$ & $(.0017)$\\
        \cmidrule(lr){1-5}
        $1.5$  & $1000$ & $\boldsymbol{.9820}$ & $.9817$ & $.9814$\\
               &        & $(.0057)$ & $(.0061)$ & $(.0059)$\\
        \cmidrule(lr){2-5}
               & $1500$ & $.9874$ & $\boldsymbol{.9876}$ & $.9868$\\
               &        & $(.0022)$ & $(.0021)$ & $(.0023)$\\
        \bottomrule
    \end{tabular}
\end{table}

\begin{figure}
    \centering
    \includegraphics[width = 0.325\linewidth]{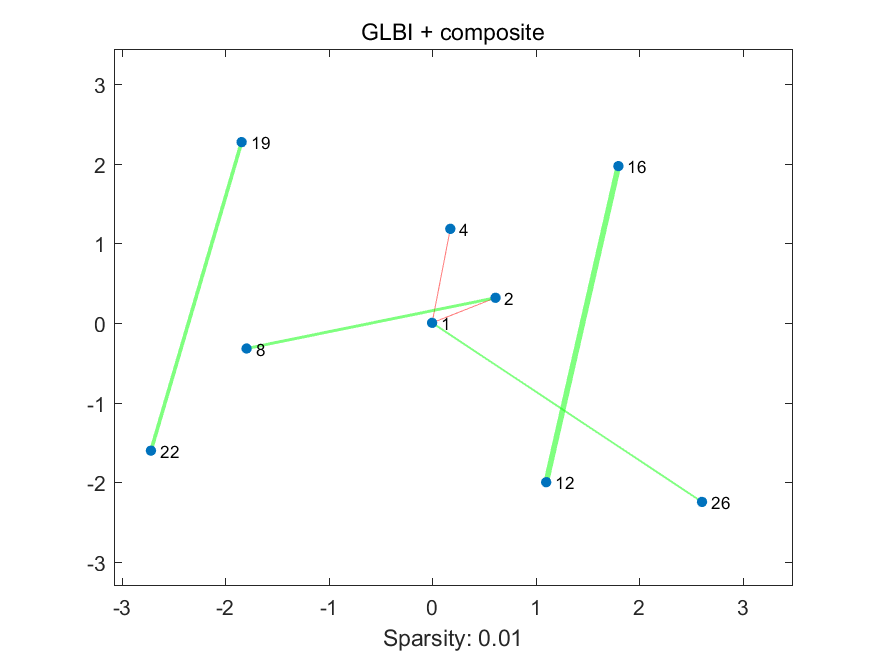}
    \includegraphics[width = 0.325\linewidth]{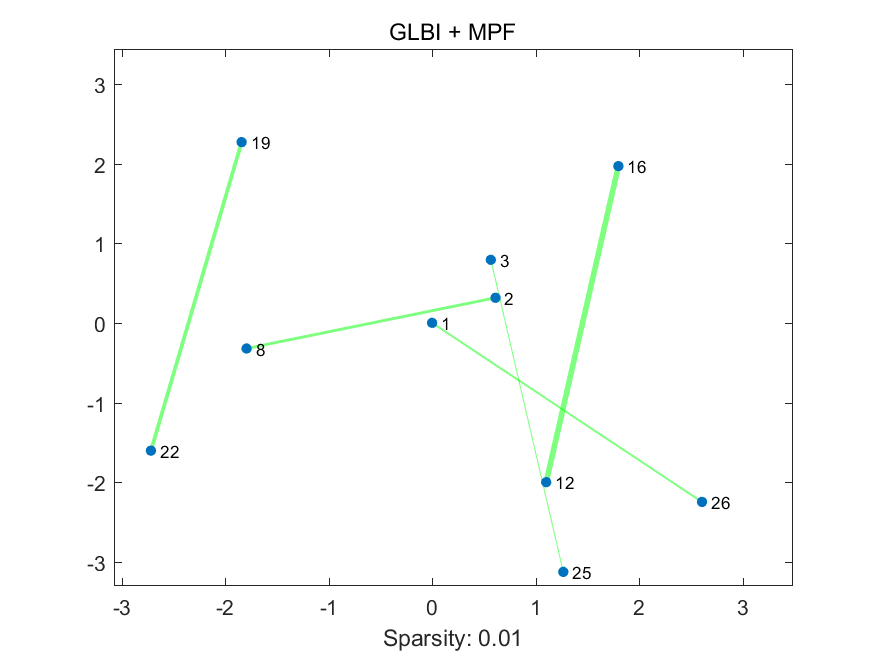}
    \includegraphics[width = 0.325\linewidth]{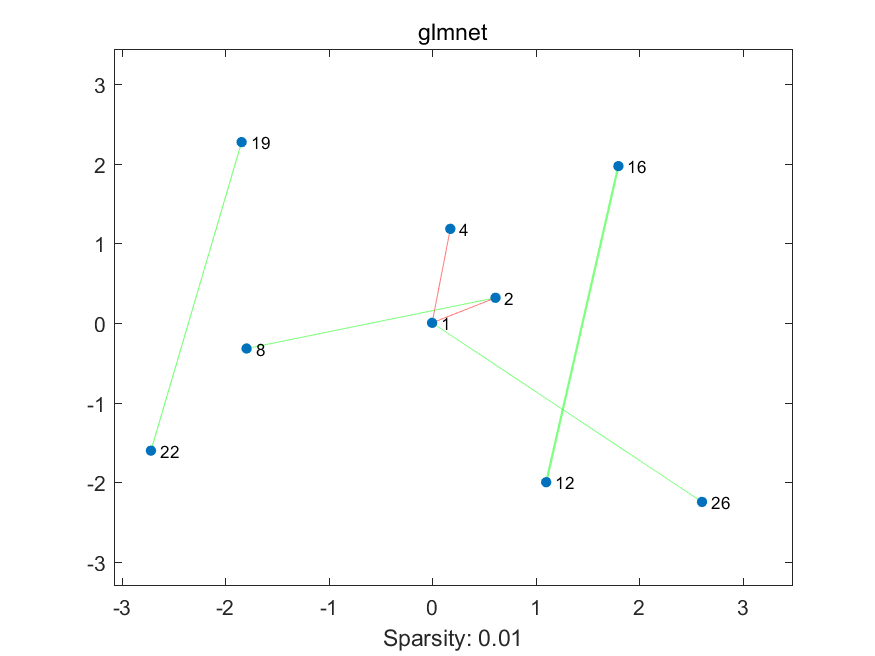}\\
    \includegraphics[width = 0.325\linewidth]{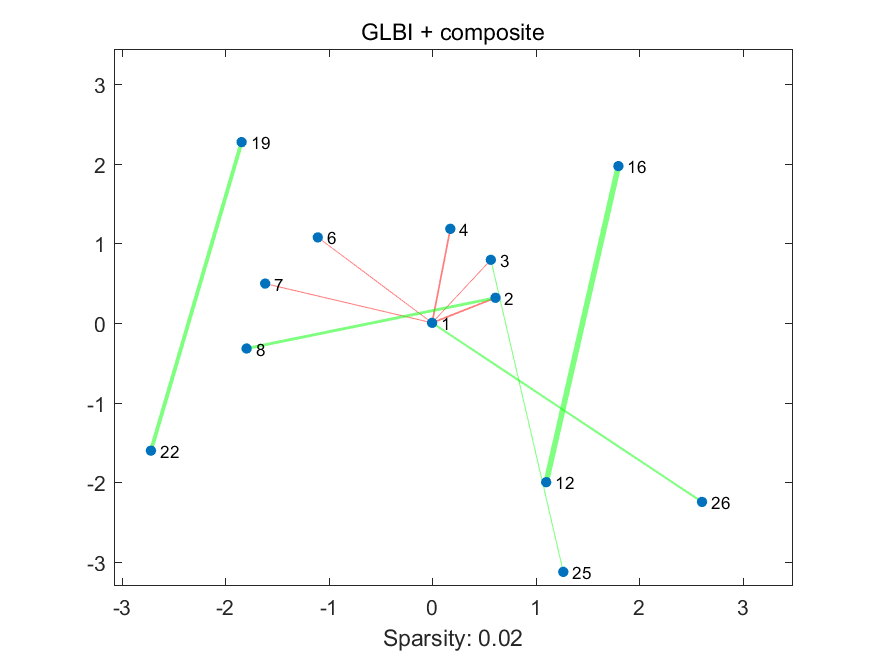}
    \includegraphics[width = 0.325\linewidth]{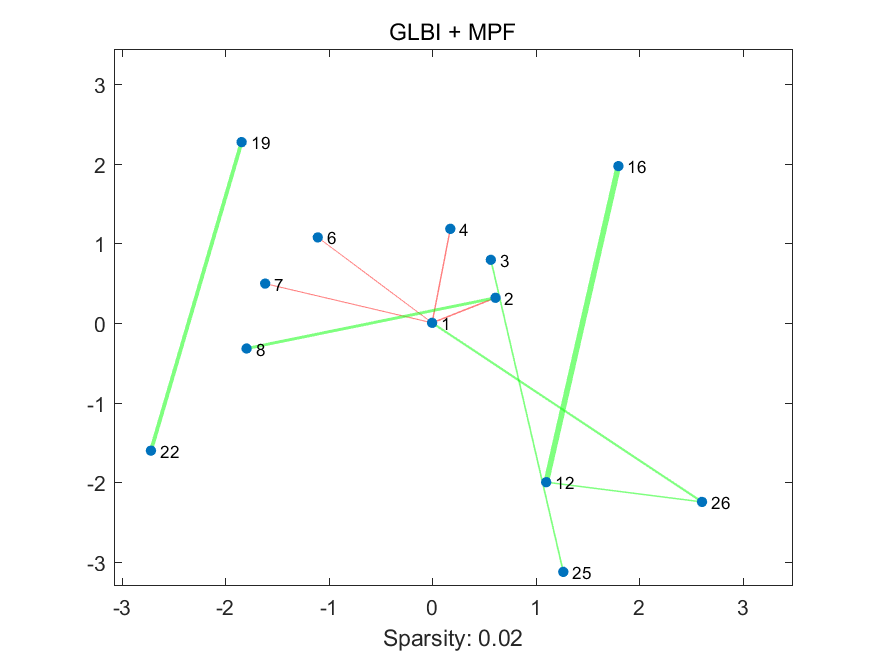}
    \includegraphics[width = 0.325\linewidth]{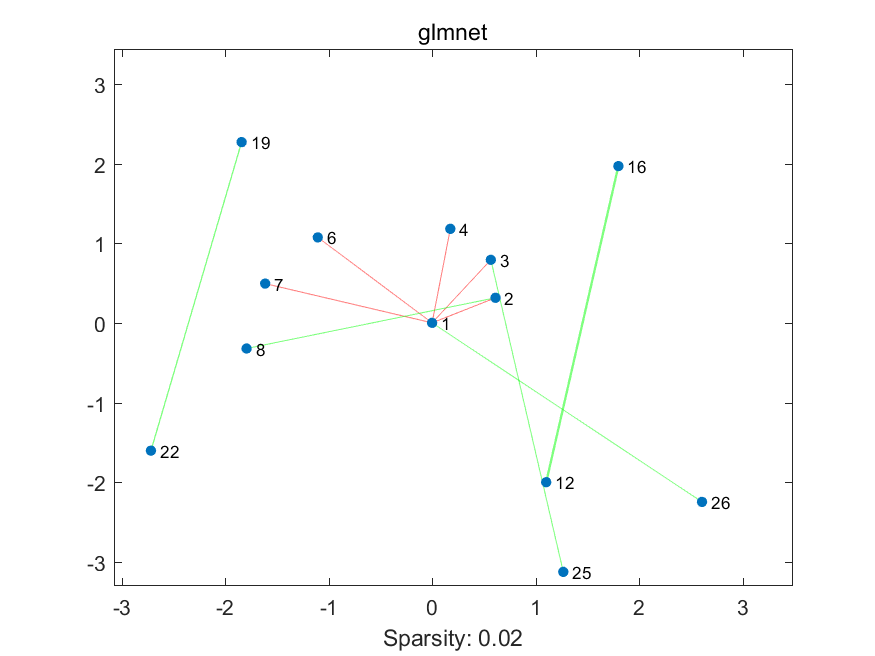}\\
    \includegraphics[width = 0.325\linewidth]{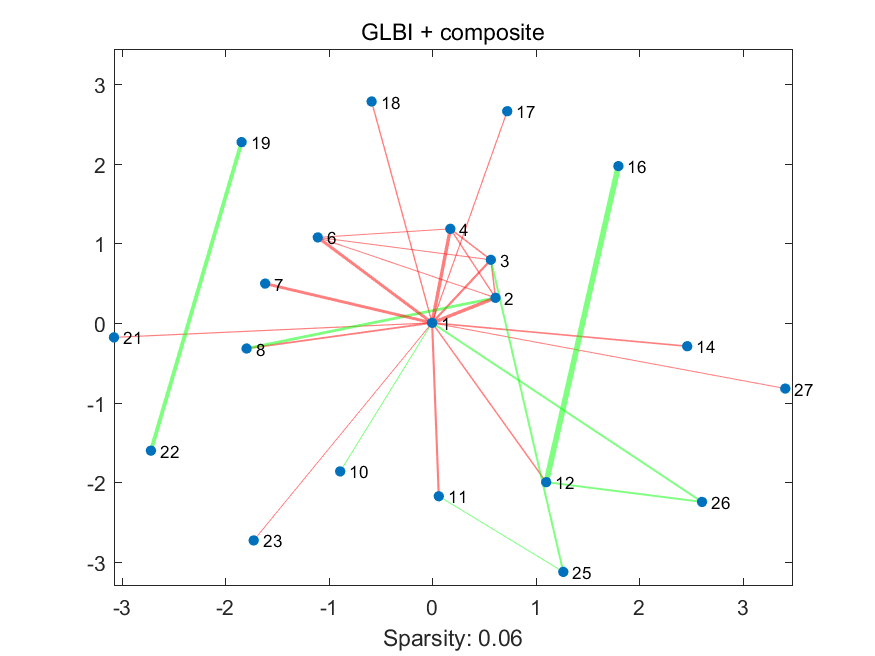}
    \includegraphics[width = 0.325\linewidth]{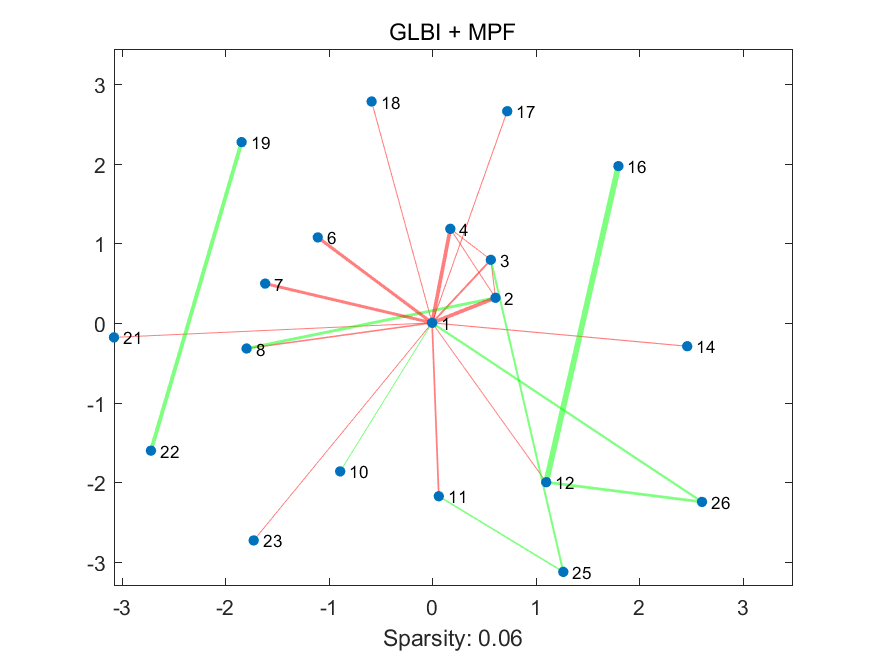}
    \includegraphics[width = 0.325\linewidth]{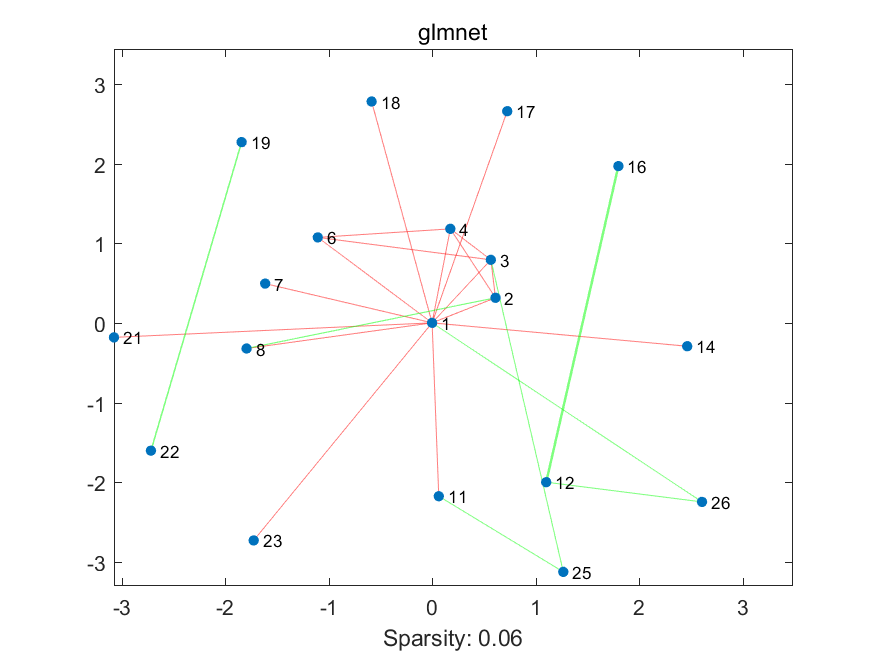}\\
    \includegraphics[width = 0.325\linewidth]{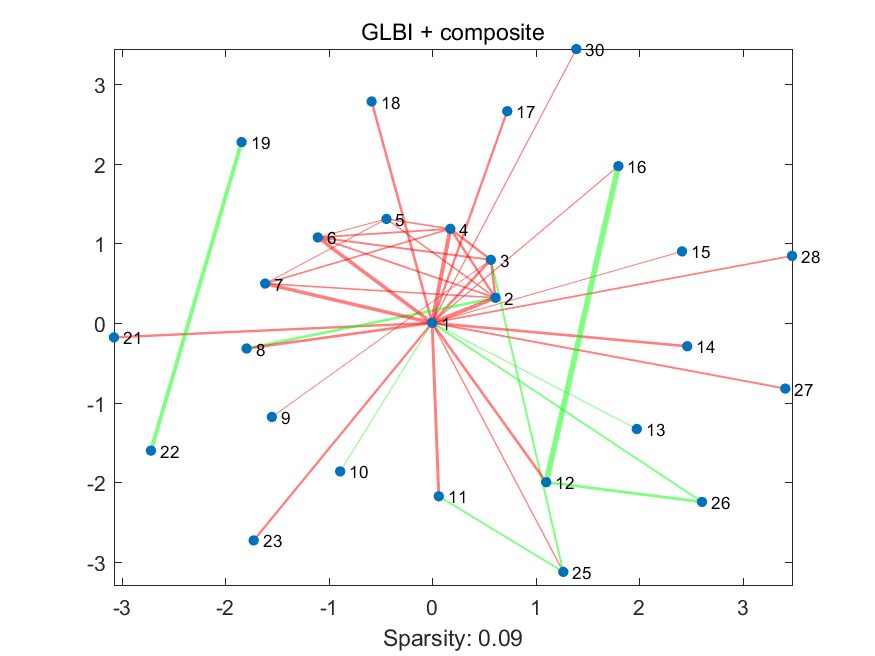}
    \includegraphics[width = 0.325\linewidth]{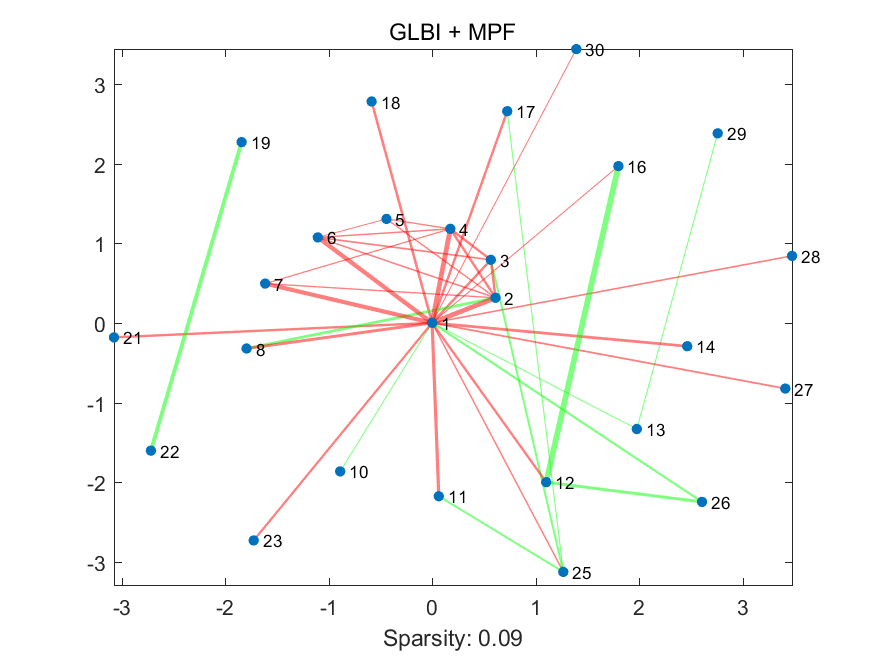}
    \includegraphics[width = 0.325\linewidth]{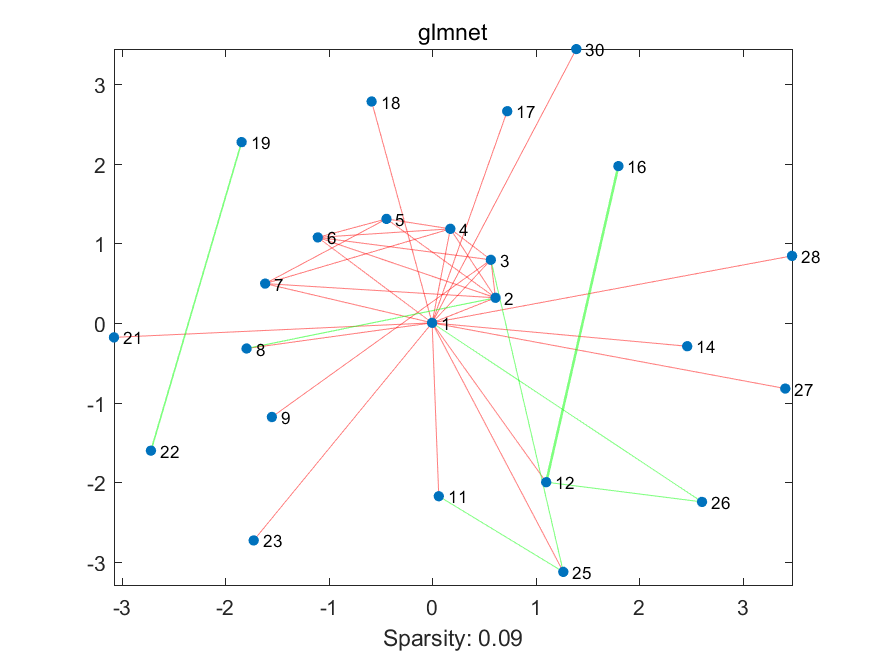}\\
    \includegraphics[width = 0.325\linewidth]{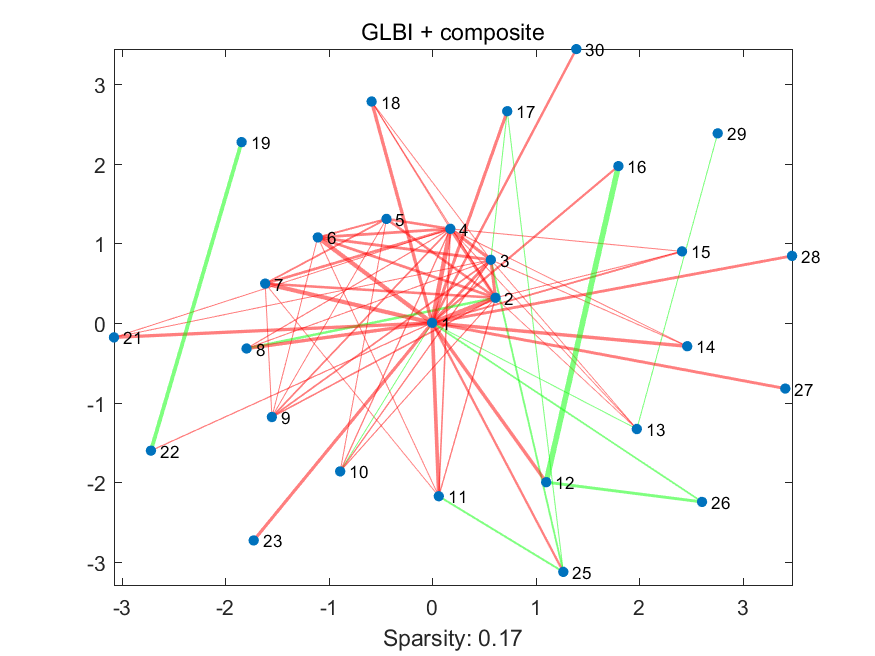}
    \includegraphics[width = 0.325\linewidth]{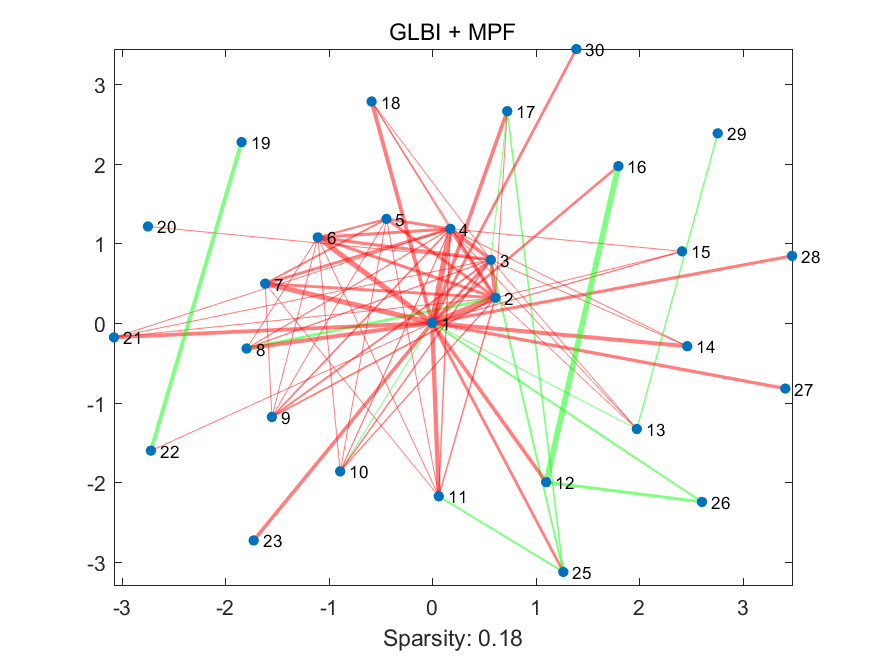}
    \includegraphics[width = 0.325\linewidth]{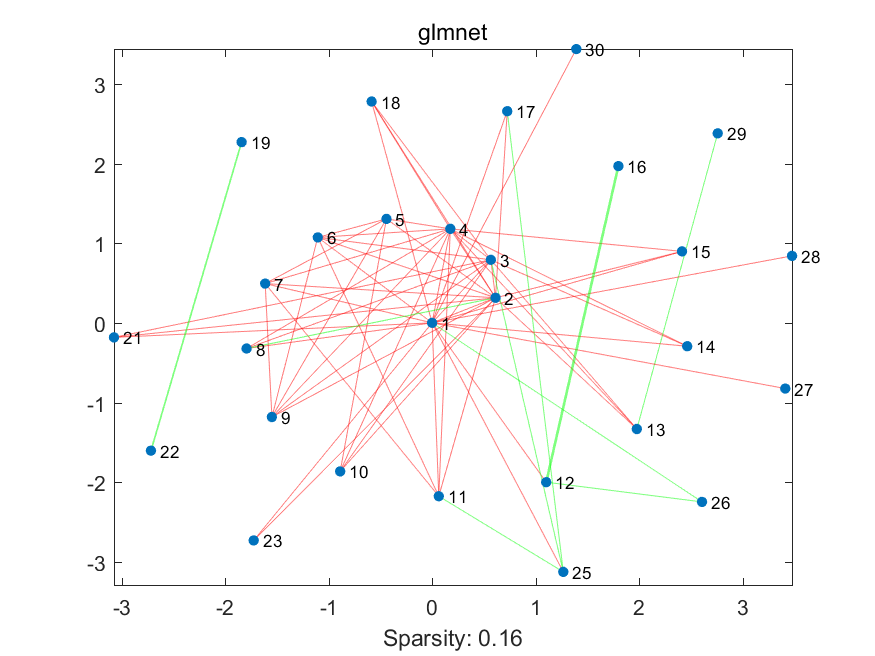}\\
    \includegraphics[width = 0.325\linewidth]{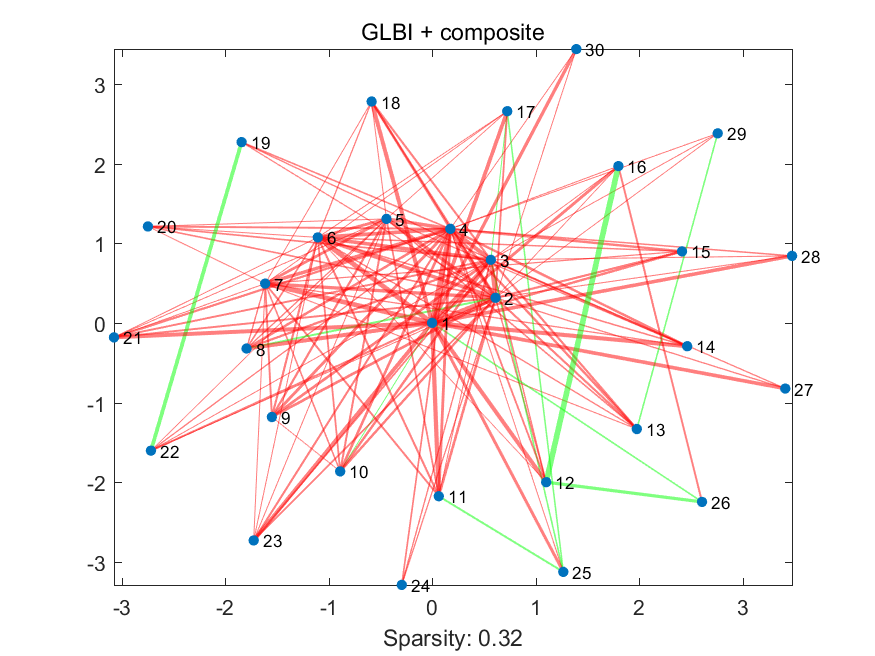}
    \includegraphics[width = 0.325\linewidth]{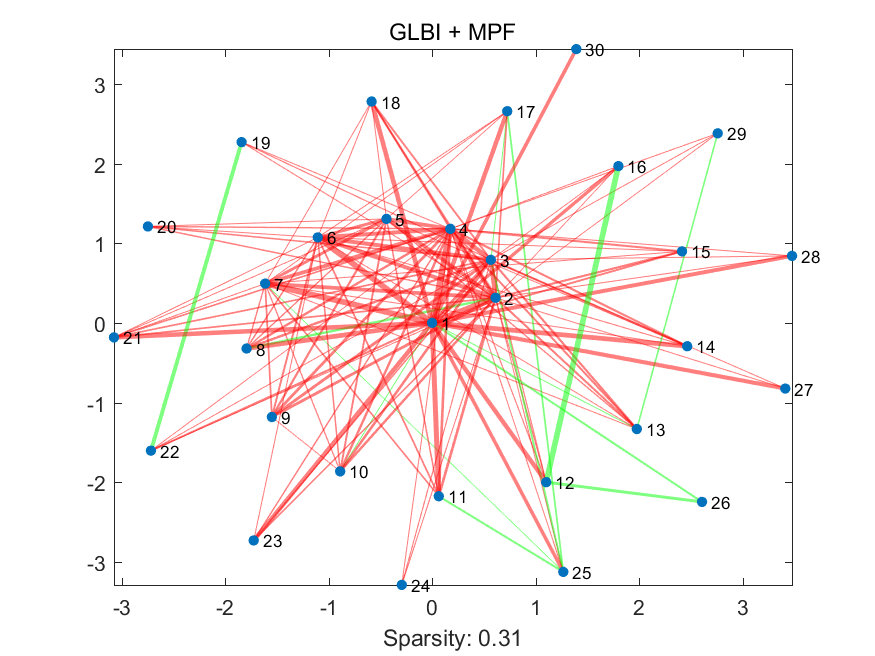}
    \includegraphics[width = 0.325\linewidth]{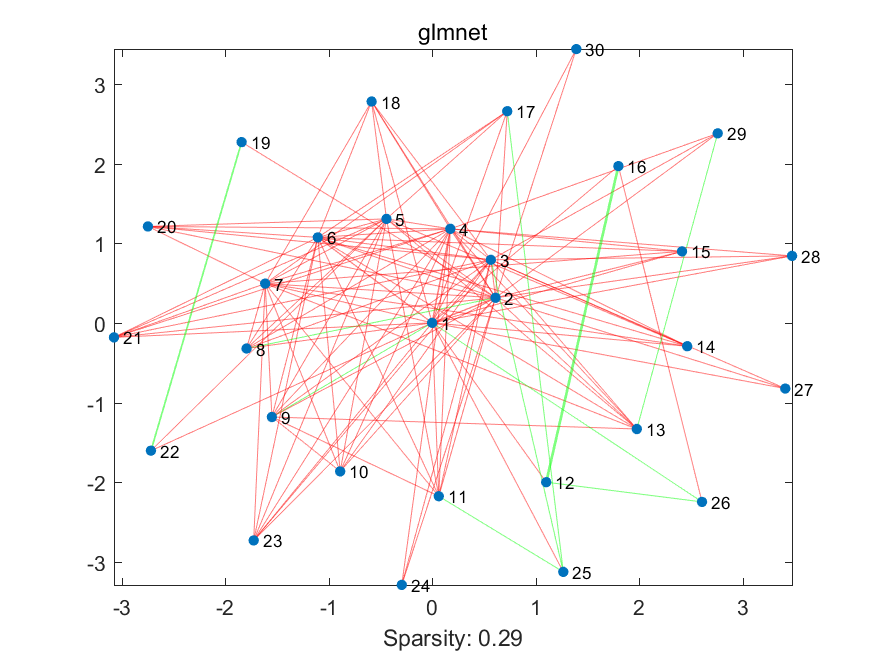}\\
    \includegraphics[width = 0.325\linewidth]{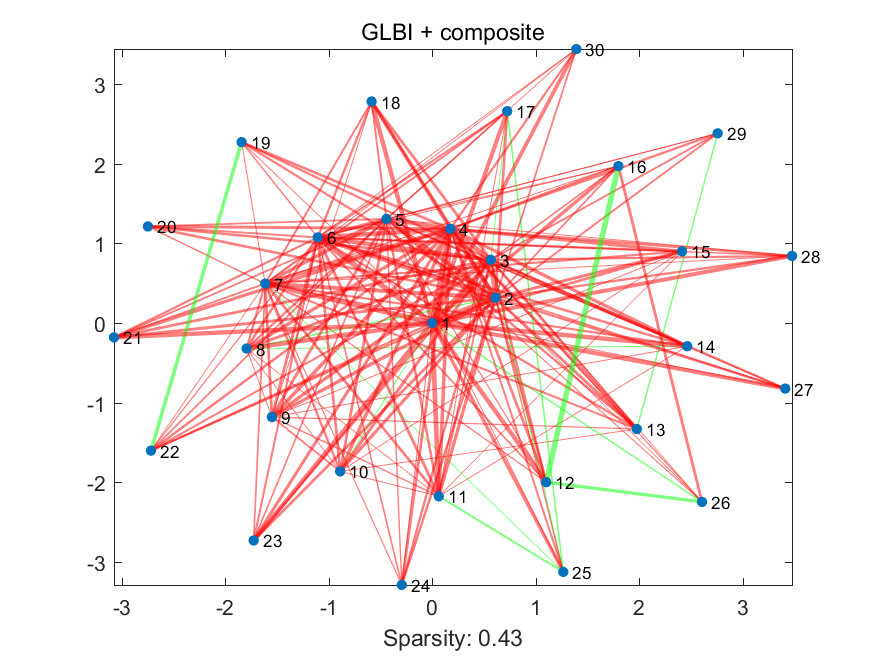}
    \includegraphics[width = 0.325\linewidth]{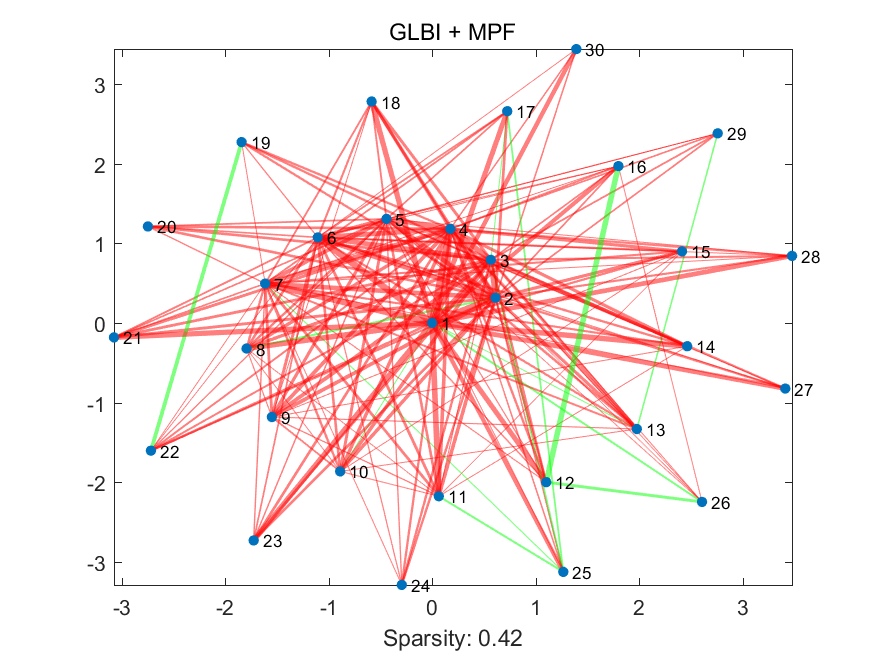}
    \includegraphics[width = 0.325\linewidth]{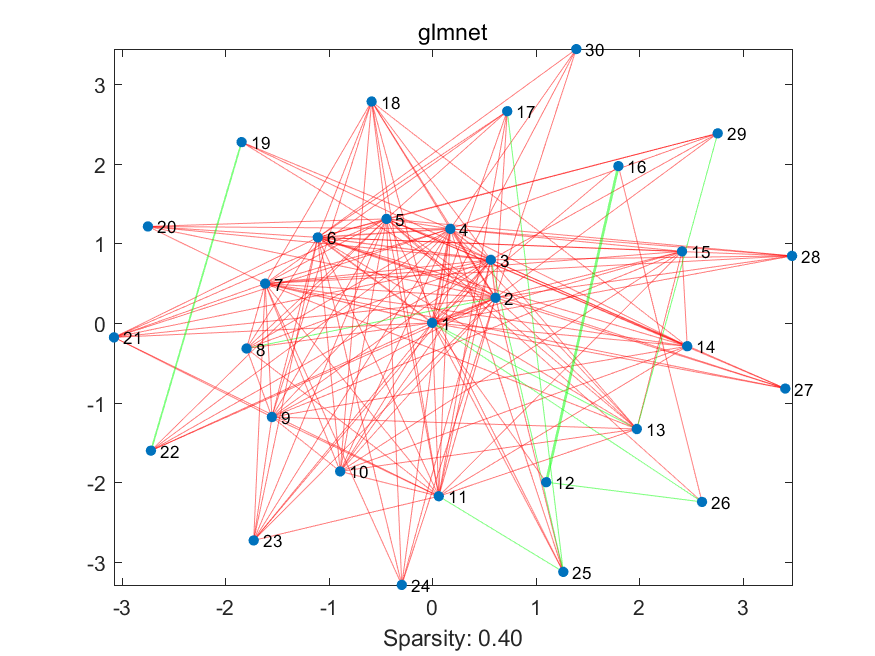}
    \caption{Left: a learned graph corresponding to an estimator picked from the path of GLBI1. Middle: from GLBI2. Right: from \texttt{glmnet}. The term \emph{sparsity level} means the ratio of the number of learned edges over the number of edges of a complete graph $K_p$. Green edges indicating positive conditional dependence of coauthorship, while red edges indicating the negative coauthorship -- the probability of coauthoring a paper significantly decreases. Strong dependence relationships are indicated by edges with large widths.}
    \label{fig:real-nips-ext}
\end{figure}

\end{document}